\documentclass{article}

\usepackage{microtype}
\usepackage{graphicx}

\usepackage{subcaption}
\usepackage{booktabs} 
\usepackage{multirow}

\usepackage{xcolor}
\usepackage{colortbl}
\definecolor{lightgrey}{RGB}{242, 242, 242}
\definecolor{average}{RGB}{240, 204, 126}
\definecolor{last}{RGB}{225, 191, 192}
\definecolor{transfer}{RGB}{191, 191, 225}
\definecolor{diag}{RGB}{151, 205, 112}

\usepackage{adjustbox}
\newcolumntype{R}[2]{%
    >{\adjustbox{angle=#1,lap=\width-(#2)}\bgroup}%
    l%
    <{\egroup}%
}

\usepackage{makecell}
\usepackage{array}

\usepackage{hyperref}
\usepackage{multirow}

\usepackage[preprint]{icml2026}

\usepackage{amsmath}
\usepackage{amssymb}
\usepackage{mathtools}
\usepackage{amsthm}

\usepackage[capitalize,noabbrev]{cleveref}

\theoremstyle{plain}

\newtheorem{proposition}{Proposition}
\newtheorem{lemma}{Lemma}
\newtheorem{corollary}{Corollary}

\theoremstyle{definition}

\newtheorem{assumption}{Assumption}

\theoremstyle{remark}
\newtheorem{remark}{Remark}

\makeatletter
\newtheorem*{rep@theorem}{\rep@title}
\newcommand{\newreptheorem}[2]{%
	\newenvironment{rep#1}[1]{%
		\def\rep@title{#2 ##1}%
		\begin{rep@theorem}}%
		{\end{rep@theorem}}}
\makeatother

\newreptheorem{lemma}{Lemma}
\newreptheorem{proposition}{Proposition}
\newreptheorem{theorem}{Theorem}

\usepackage[textsize=tiny]{todonotes}

\usepackage{xspace}

\icmltitlerunning{LR-RGDA \& HopDC}

\begin{document}

\twocolumn[
\icmltitle{Scalable Analytic Classifiers with Associative Drift Compensation for Class-Incremental Learning of Vision Transformers}



\icmlsetsymbol{equal}{*}

\begin{icmlauthorlist}
	\icmlauthor{Xuan Rao}{bnu}
	\icmlauthor{Mingming Ha}{kuaishou}
	\icmlauthor{Bo Zhao}{bnu}
	\icmlauthor{Derong Liu}{sustech}
	\icmlauthor{Cesare Alippi}{polimi,usi}
\end{icmlauthorlist}

\icmlaffiliation{bnu}{School of Systems Science, Beijing Normal University, Beijing 100875, China}
\icmlaffiliation{kuaishou}{Kuaishou Technology, Beijing 100085, China}
\icmlaffiliation{sustech}{School of Artificial Intelligence, Anhui University, Hefei 230601, China}
\icmlaffiliation{polimi}{Dipartimento di Elettronica e Informazione, Politecnico di Milano, Milano 20133, Italy}
\icmlaffiliation{usi}{Faculty of Informatics, Universita' Della Svizzera Italiana, Lugano, Switzerland}

\icmlcorrespondingauthor{Bo Zhao}{zhaobo@bnu.edu.cn}

\icmlkeywords{Gaussian discriminant analysis, Class incremental learning, Vision transformers}

\vskip 0.3in
]



\printAffiliationsAndNotice{}  

\begin{abstract}
	Class-incremental learning (CIL) with Vision Transformers (ViTs) faces a major computational bottleneck during the classifier reconstruction phase, where most existing methods rely on costly iterative stochastic gradient descent (SGD). We observe that analytic Regularized Gaussian Discriminant Analysis (RGDA) provides a Bayes-optimal alternative with accuracy comparable to SGD-based classifiers; however, its quadratic inference complexity limits its use in large-scale CIL scenarios. To overcome this, we propose Low-Rank Factorized RGDA (LR-RGDA), a scalable classifier that combines RGDA’s expressivity with the efficiency of linear classifiers. By exploiting the low-rank structure of the covariance via the Woodbury matrix identity, LR-RGDA decomposes the discriminant function into a global affine term refined by a low-rank quadratic perturbation, reducing the inference complexity from $\mathcal{O}(Cd^2)$ to $\mathcal{O}(d^2+Crd^2)$, where $C$ is the class number, $d$ the feature dimension, and $r \ll d$ the subspace rank. To mitigate representation drift caused by backbone updates, we further introduce Hopfield-based Distribution Compensator (HopDC), a training-free mechanism that uses modern continuous Hopfield Networks to recalibrate historical class statistics through associative memory dynamics on unlabeled anchors, accompanied by a theoretical bound on the estimation error. Extensive experiments on diverse CIL benchmarks demonstrate that our framework achieves state-of-the-art performance, providing a scalable solution for large-scale class-incremental learning with ViTs. Code: \href[]{https://github.com/raoxuan98-hash/lr_rgda_hopdc}{LR-RGDA and HopDC}.
\end{abstract}
\section{Introduction}
\label{sec:introduction}
Recent years have witnessed growing interest in leveraging pre-trained models (PTMs), such as vision transformers (ViTs), for class-incremental learning (CIL) \cite{zheng2023preventing, 10970405, 10882940}. A popular pipeline decouples CIL into two partly decoupled stages: (1) backbone optimization; (2) class distribution approximation and classifier reconstruction \cite{zhang2023slca, 10.1007/978-3-031-73209-6_18, Li_2024_WACV, NEURIPS2024_0f06be00}. While the former aspect has mitigated catastrophic forgetting through strategies including knowledge distillation (KD) \cite{gao2025knowledge, zhao2024safe}, null-space projection (NSP) \cite{NEURIPS2024_0f06be00, liang2024inflora}, model ensembles \cite{marczak2024magmax}, parameter-efficient fine-tuning (PEFT) \cite{Li_2024_WACV, he2025cllora}, etc, the latter remains an open, yet critical, issue.

Stage 2 is challenged by two closely related issues: (1) the efficiency–accuracy trade-off in classifier reconstruction; (2) The distribution misalignment caused by representation drift \cite{rao2025compensating, dpcr}. Existing approaches mainly rely on iterative stochastic gradient descent (SGD) to refine classifiers using pseudo-features sampled from Gaussian approximations \cite{zhang2023slca, NEURIPS2024_0f06be00, 10.1007/978-3-031-73209-6_18, 10970405}. However, these methods introduce substantial computational overhead beyond backbone optimization. In addition, synthesizing and buffering large batches of pseudo-features places a heavy demand on GPU memory, while the iterative optimization itself offers no convergence guarantees.
\begin{figure}[t]
	\centering
	\includegraphics[width=1.0\linewidth]{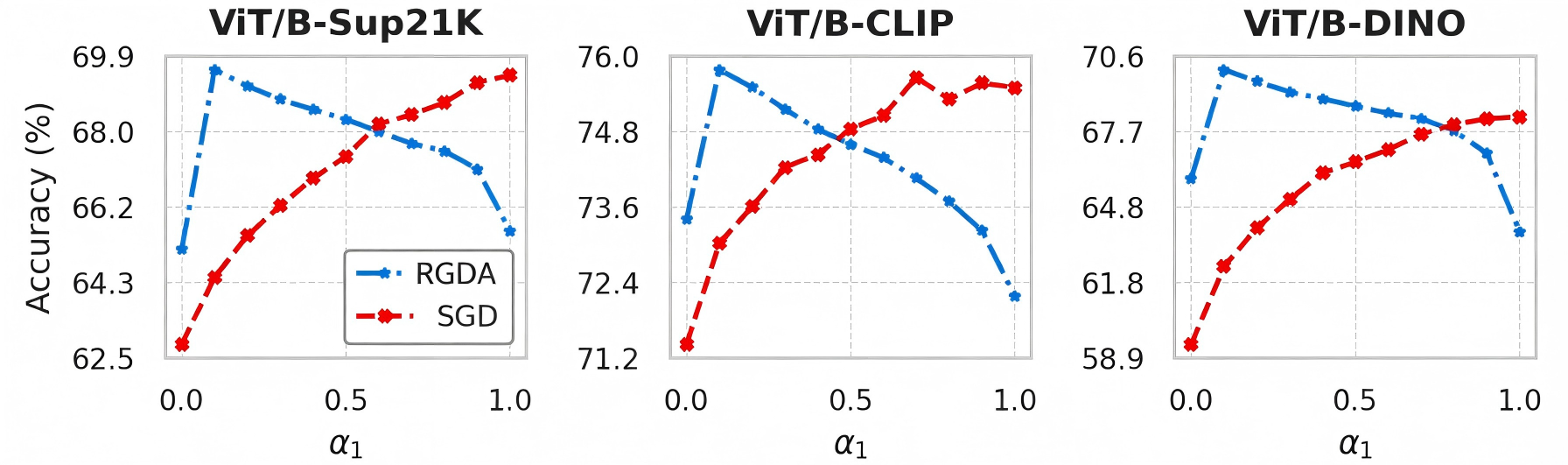}
	\caption{RGDA vs. SGD-based classifiers performance on a 1001-class cross-domain dataset. By regularizing covariances as $\Sigma_c^{\rm reg} = \alpha_1 \Sigma_c + (1-\alpha_1)\Sigma_{\rm avg}$, RGDA interpolates between LDA ($\alpha_1=0$) and QDA ($\alpha_1=1$).}
	\label{fig:motivationexps}
\end{figure}

To address these limitations, we revisit Gaussian discriminant analysis (GDA) \cite{izenman2013linear, tharwat2016linear}, which shows to be provably Bayes-optimal under the gaussianity assumption. As shown in Figure 1, both Linear Discriminant Analysis (LDA) \cite{izenman2013linear}, which assumes a shared covariance, and Quadratic Discriminant Analysis (QDA) \cite{tharwat2016linear}, which uses class-specific covariances exclusively, underperform compared to SGD-based classifiers. Remarkably, Regularized GDA (RGDA), which interpolates between shared and class-specific covariances, achieves accuracy comparable to, or even surpassing, SGD-based classifiers, without requiring iterative optimization. However, RGDA remains computationally prohibitive for large-scale employments, as its inference complexity scales quadratically as $\mathcal{O}(Cd^2)$, where $C$ is the class number and $d$ the feature dimension.

\begin{figure}[t]
	\centering
	\includegraphics[width=1.00\linewidth]{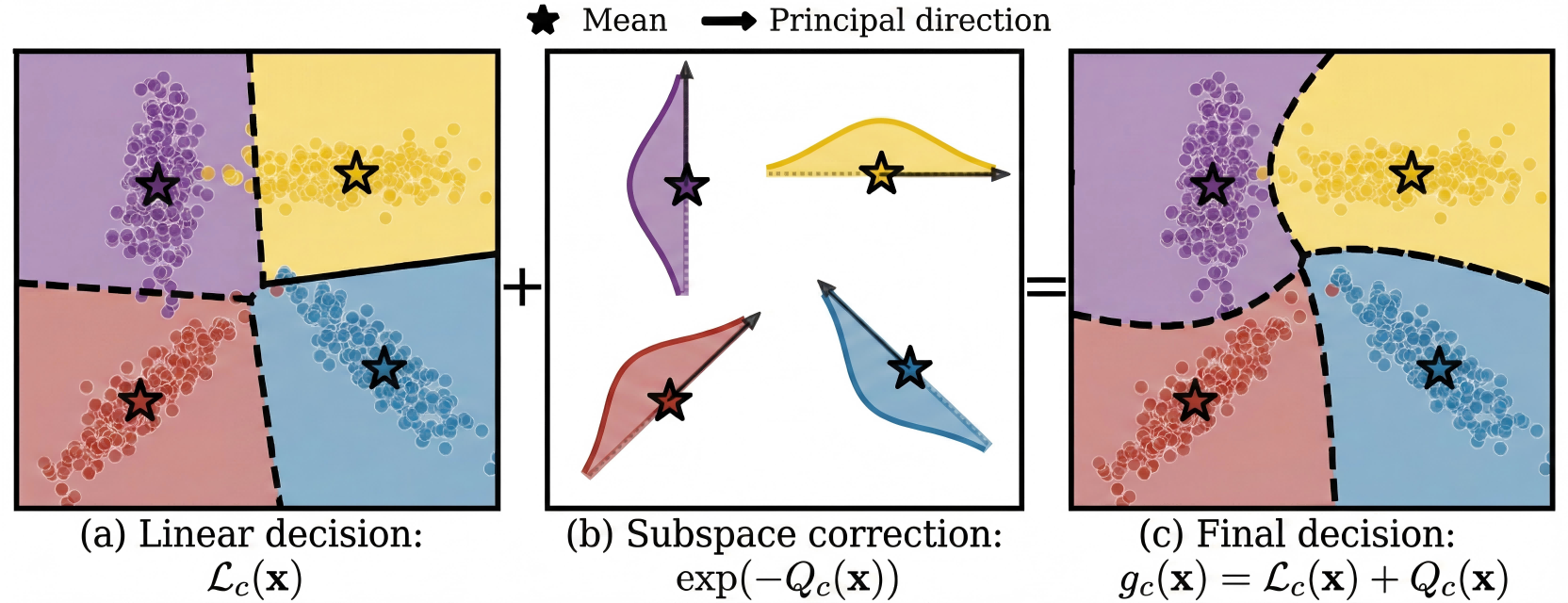}
	\caption{Conceptual visualization of LR-RGDA’s discriminant function. (a) The global affine term $\mathcal{L}_c(\mathbf{x})$ provides linear decision boundaries using the average covariance. (b) Class-specific low-rank corrections $\mathcal{Q}_c(\mathbf{x})$ (visualized via $\exp(\mathcal{Q}_c(\mathbf{x}))$ for clarity) introduce nonlinear adjustments. (c) The final discriminant $g_c(\mathbf{x}) = \mathcal{L}_c(\mathbf{x}) + \mathcal{Q}_c(\mathbf{x})$ yields composite, nonlinear boundaries.}
	\label{fig:lrrgda}
\end{figure}

To resolve this dilemma, we propose low-rank factorized RGDA (LR-RGDA). Motivated by the observation that class covariance matrices often exhibit low-rank structures (see Figure \ref{fig:vit-b-p16-clip95percentrankboxplot}), we establish the covariance as a global full-rank term plus a low-rank class-specific perturbation. Hereafter, by applying the Woodbury matrix identity \cite{hager1989updating}, we reduce the construction complexity of RGDA from $\mathcal{O}(Cd^3)$ to $\mathcal{O}(d^3 + Crd^2)$. Crucially, LR-RGDA decomposes the discriminant function into a global linear function refined by a lightweight class-specific quadratic form in a rank-$r$ subspace where $r \ll d$. As shown in Figure \ref{fig:lrrgda}, the low-rank subspace correction  augments the initial linear boundaries into the nonlinear ones of the full LR-RGDA classifier. This structure allows LR-RGDA to retain RGDA's expressiveness while reducing per-sample inference complexity from $\mathcal{O}(Cd^2)$ to $\mathcal{O}(d^2 + Crd)$, making it computationally comparable to linear classifiers when $r \ll d$.

The other challenge arises from the representation drifts induced by backbone updates \cite{rao2025compensating, dpcr, gomez2024exemplar}. As the backbone is optimized for new tasks in Stage-1, the stored statistics of old classes become misaligned with its current distribution, which can lead to biased decision boundaries. 

To tackle this aspect, we propose the Hopfield-based Distribution Compensator (HopDC), a training-free calibration mechanism. Prior works typically attempt to model drift using linear or weakly-nonlinear operators~\cite{rao2025compensating, gomez2024exemplar}; however, these approaches necessitate additional optimization phases to learn transformation parameters and lack theoretical guarantees regarding the upper bound of drift estimation errors. In contrast, HopDC formulates drift estimation as an associative memory problem via Modern Continuous Hopfield Networks (MCHNs)~\cite{DBLP:conf/iclr/RamsauerSLSWGHA21}. By leveraging attention-based energy minimization on a small set of unlabeled anchor samples, HopDC aligns historical class statistics with the updated feature space in a completely training-free manner, while providing a rigorous theoretical upper bound on the estimation error. Our contributions are summarized as follows. 
\begin{enumerate}
	\item We propose LR-RGDA, a novel scalable analytic classifier that resolves the efficiency-accuracy dilemma in classifier reconstruction. By exploiting the Woodbury matrix identity to decompose the discriminant function, we theoretically prove that LR-RGDA reduces the inference complexity from $\mathcal{O}(Cd^2)$ to $\mathcal{O}(d^2 + Crd)$ while maintaining the Bayes-optimality of RGDA.
	
	\item We propose HopDC, a novel training-free distribution compensator to address representation drift in sequential optimization. By leveraging MCHNs for statistical alignment, we mathematically prove an estimation error upper bound constrained by the semantic drift's Lipschitz continuity and anchor density, which provides a formal guarantee for its reliability.
	
	\item Comprehensive experiments demonstrate that our entire framework achieves new state-of-the-art performance across multiple pre-trained ViTs and backbone adaptation strategies (e.g., KD, NSP, LoRA).
\end{enumerate}

\section{Related Works}
\textbf{Backbone optimization in CIL.} Existing ViT-based CIL approaches generally fall into four categories based on their strategies for handling representation stability and plasticity.

The first category optimizes task-specific adapters and decomposes the inference into a hierarchical process of task-identity prediction followed by label prediction \cite{10970405, Li_2024_WACV}. However, this paradigm suffers from error propagation due to the inaccurate task identification, high computational overhead from repeated forward passes, and linearly scaling storage demands.

The second approach maintains a shared backbone (or a shared adapter) using techniques like lower learning rates, distillation \cite{heo2019comprehensive}, model ensemble \cite{marczak2024magmax}, or gradient projection \cite{wang2021training} to mitigate catastrophic forgetting. Recent works have also explored PEFT-based adaptation to enhance stability \cite{zhao2024safe, liang2024inflora, he2025cllora}. For example, slow learner with classifier alignment (SLCA) \cite{zhang2023slca, zhang2024slcaunleashpowersequential} adapts ViT backbones with lower learning rates, continual model averaging (CoMA) improves SLCA by averaging current and past models to enhance performance \cite{10.1007/978-3-031-73209-6_18}. However, the progressive optimization strategies remain inherently vulnerable to cumulative representation drift.

The third category combines shared adapters with instance-level feature adaptation via visual prompting. Approaches like L2P \cite{wang2022learning}, DualPrompt \cite{wang2022dualprompt}, and CODA-Prompt \cite{smith2023coda} leverage learnable query mechanisms for dynamic prompt selection, while recent works explore null-space tuning to further minimize interference \cite{NEURIPS2024_0f06be00}.

The fourth category leverages the frozen PTM representations. The FSA employs incremental LDA for exemplar-free adaptation \cite{Panos_2023_ICCV, zhou2025revisiting}, while methods like RanPAC \cite{NEURIPS2023_2793dc35} and LayUP \cite{ahrens2024read} enhance representation expressivity via high-dimensional random projections and multi-layer feature concatenation, respectively. Recently, TSVD bridges theory and practice by utilizing singular value decomposition for rigorous feature analysis \cite{peng2025tsvd}.

\textbf{Distribution drifts.} Several methods explicitly tackles the distribution shifts \cite{yu2020semantic, gomez2024exemplar} for neural networks without pre-training. For instance, AddGauss mitigates the task-recency bias via nonlinear covariance adaptation \cite{NEURIPS2024_73ba81c7}, while DPCR quantifies semantic drifts using task-wise geometric projections \cite{dpcr}. Recently, \citet{rao2025compensating} investigated the distribution drifts in ViT-based CIL and found the feature space evolution is weakly nonlinear. They proposed linear and weak-nonlinear operators to approximate the evolution and compensate for these drifts, providing valuable insights for CIL with PTMs. Our HopDC takes a step further by modeling this evolution via the associative memory dynamics of MCHNs \cite{DBLP:conf/iclr/RamsauerSLSWGHA21}.

\section{Method}
\textbf{Problem formulation.}
We consider a CIL setting involving a sequence of datasets $\mathcal{D}_1, \dots, \mathcal{D}_T$, where $\mathcal{D}_t = \{(\mathbf{x}_{(t,n)}, y_{(t,n)})\}_{n=1}^{n_t}$. For simplicity, we assume the label spaces $\mathcal{Y}_t$ are disjoint during CIL, i.e., $\mathcal{Y}_t \cap \mathcal{Y}_{t'} = \emptyset$ for $t \neq t'$. Let $\mathcal{C}_t = \bigcup_{t'=1}^t \mathcal{Y}_{t'}$ be the set of all observed classes up to task $t$. Let $\mathcal{F}_{\boldsymbol{\theta}}$ be a frozen or fine-tuned pre-trained backbone, and $\mathcal{C}_{\boldsymbol{\phi}}$ be a linear classifier mapping features to $\mathbb{R}^{|\mathcal{C}_t|}$ with $\boldsymbol{\theta}$ and $\boldsymbol{\phi}$ trainable parameters. At task $t$, the model is optimized via the task-wise cross-entropy loss
$\mathcal{L}_{\text{CE}}(\boldsymbol{\varphi}; \mathcal{D}_t) = -\frac{1}{n_t} \sum_{n=1}^{n_t} \log p_{\boldsymbol{\varphi}}(\mathbf{x}_{(t,n)}; \mathcal{Y}_t)_{y_{(t,n)}}$,
where $p_{\boldsymbol{\varphi}}(\cdot; \mathcal{S})$ denotes the Softmax output over the label subset $\mathcal{S}$. After learning task $t$, we approximate the feature distribution of each class $c \in \mathcal{Y}_t$ as a Gaussian $\mathcal{N}(\boldsymbol{\mu}_c, \boldsymbol{\Sigma}_c)$, with its moments estimated by
\small
\begin{equation}
		\boldsymbol{\mu}_c = \frac{1}{n_c} \sum_{i=1}^{n_c} \mathbf{f}_c^{(i)}, \quad \boldsymbol{\Sigma}_c = \frac{1}{n_c} \sum_{i=1}^{n_c} \mathbf{f}_c^{(i)} (\mathbf{f}_c^{(i)})^\top - \boldsymbol{\mu}_c \boldsymbol{\mu}_c^\top
\end{equation}
\normalsize
where $\mathbf{f}_c^{(i)} = \mathcal{F}_{\boldsymbol{\theta}}(\mathbf{x}^{(i)})$. Let $\mathcal{H}_t = \{ \mathcal{N}(\boldsymbol{\mu}_c, \boldsymbol{\Sigma}_c) \mid c \in \mathcal{C}_t \}$ be the collection of historical statistics.
\subsection{Regularized Gaussian discriminant analysis}
Building upon the Gaussian statistics $\mathcal{H}_t$, our goal is to construct a classifier that approximates the Bayes-optimal decision rule. As evidenced by experiments in Figure \ref{fig:motivationexps}, we propose to leverage the RGDA by employing a shrinkage strategy that interpolates among the class-specific covariance $\boldsymbol{\Sigma}_c$, the global average $\boldsymbol{\Sigma}_{\mathrm{avg}} = \frac{1}{C}\sum_{c}\boldsymbol{\Sigma}_{c}$, and the identity matrix $\mathbf{I}_d$ as
\begin{align}
	\boldsymbol{\Sigma}_c^{\mathrm{reg}} = \alpha_1 \boldsymbol{\Sigma}_c + \alpha_2 \boldsymbol{\Sigma}_{\mathrm{avg}} + \alpha_3 \mathbf{I}_d,
	\label{eq:regularized_covariance}
\end{align}
where $\alpha_1, \alpha_2, \alpha_3 \geq 0$ are regularization hyperparameters. Let $\boldsymbol{\pi} = [\pi_1, \ldots, \pi_C]^\top$ be the class prior distribution.
By Bayes theorem, the RGDA discriminant function $g_c^{\mathrm{RGDA}}(\mathbf{x}) \propto \log p(c|\mathbf{x})$ can be derived as
\begin{align}
	g_c^{\mathrm{RGDA}}(\mathbf{x}) \propto -\frac{1}{2}  d_{\boldsymbol{\Sigma}_c^{\mathrm{reg}}}^2(\mathbf{x}, \boldsymbol{\mu}_c) -  \frac{1}{2} \log \det(\boldsymbol{\Sigma}_c^{\mathrm{reg}}),
	\label{eq:rgda_discriminant}
\end{align}
where $d_{\boldsymbol{\Sigma}}^2(\mathbf{x}, \boldsymbol{\mu}) = (\mathbf{x} - \boldsymbol{\mu})^\top \boldsymbol{\Sigma}^{-1} (\mathbf{x} - \boldsymbol{\mu})$ denotes the squared Mahalanobis distance (i.e., the quadratic form).
\begin{lemma}[Bayes Optimality of RGDA]
\label{eq:bayes_opt}
Under the assumption that Gaussian class-conditional features, maximizing the RGDA discriminant function corresponds to the Bayes-optimal decision rule (See Appendix \ref{app:proof_bayes} for proof).
\end{lemma}

\begin{remark}
Despite its simplicity, the gaussianity assumption is widely adopted in ViT-based CIL \cite{zhang2023slca,NEURIPS2024_0f06be00,rao2025compensating}. Empirical studies have shown that the features extracted by pre-trained ViTs are highly separable and approximately Gaussian within each class \cite{10970405, Li_2024_WACV}. 
\end{remark}

\textbf{Relationship with QDA and LDA.} Specifically, setting $\alpha_1=1$ and $\alpha_{2/3}=0$ recovers RGDA to QDA. Conversely, setting $\alpha_{1/3}=0$ and $\alpha_2=1$ simplifies RGDA to LDA, whose discriminant function satisfies  
\begin{align}
	g_c^{\mathrm{LDA}}(\mathbf{x}) &\propto -\frac{1}{2} d^2_{\Sigma_{\rm avg}}\left( \mathbf{x}, \boldsymbol{\mu}_{c} \right)\\
	&\propto \boldsymbol{\mu}_c^\top \boldsymbol{\Sigma}_{\mathrm{avg}}^{-1} \mathbf{x} - \frac{1}{2} \boldsymbol{\mu}_c^\top \boldsymbol{\Sigma}_{\mathrm{avg}}^{-1} \boldsymbol{\mu}_c.
\end{align}
It yields the linear discriminant function
\begin{equation}
	g_c^{\mathrm{LDA}}(\mathbf{x}) = \mathbf{w}_c^\top \mathbf{x} + b_c,
	\label{eq:lda_linear}
\end{equation}
where $\mathbf{w}_c = \boldsymbol{\Sigma}_{\mathrm{avg}}^{-1} \boldsymbol{\mu}_c$ and $b_c = -\frac{1}{2} \boldsymbol{\mu}_c^\top \boldsymbol{\Sigma}_{\mathrm{avg}}^{-1} \boldsymbol{\mu}_c$. This derivation explicitly demonstrates that LDA is mathematically equivalent to an affine classifier, which offers high inference efficiency of $\mathcal{O}(Cd)$, albeit at the cost of ignoring class-specific fine-grained information.

\textbf{The efficiency-accuracy dilemma.} Although RGDA harmonizes the expressivity of QDA with the stability of LDA, its reliance on class-specific regularized covariances $\boldsymbol{\Sigma}_c^{\mathrm{reg}}$ imposes prohibitive computational costs. Unlike linear models, this dependence leads to severe scalability bottlenecks. Specifically, constructing the classifier requires inverting matrices for all classes with $\mathcal{O}(Cd^3)$ complexity, far exceeding the $\mathcal{O}(d^3)$ of LDA. Furthermore, inference entails a quadratic operation scaling with $\mathcal{O}(Cd^2)$, which is a substantial increase over the $\mathcal{O}(Cd)$ of linear classifiers. In addition, maintaining full-precision matrices demands $\mathcal{O}(Cd^2)$ memory (e.g., over 4.5GB for $C=1000, d=768$ in double precision).

\subsection{Low-Rank Factorized RGDA}


To resolve the efficiency-accuracy dilemma of RGDA, we propose the LR-RGDA, a scalable variant that assumes the class-specific information resides in a low-dimensional subspace, as empirically validated by Figure \ref{fig:vit-b-p16-clip95percentrankboxplot}.

\textbf{Covariance matrix decomposition.}
We model the covariance as a global full-rank base refined by a class-specific low-rank perturbation. Let the base matrix $\mathbf{B} = \alpha_2 \boldsymbol{\Sigma}_{\mathrm{avg}} + \alpha_3 \mathbf{I}_d$ capture the common components shared across all classes. For the class-specific term $\alpha_1 \boldsymbol{\Sigma}_c$, we apply truncated singular value decomposition (SVD) to obtain its rank-$r$ approximation $\boldsymbol{\Sigma}_c \approx \mathbf{U}_c \mathbf{S}_c \mathbf{U}_c^\top$, where $\mathbf{U}_c \in \mathbb{R}^{d \times r}$ contains the top-$r$ eigenvectors and $\mathbf{S}_c \in \mathbb{R}^{r \times r}$ is the diagonal matrix of eigenvalues. Consequently, the regularized covariance matrix $\boldsymbol{\Sigma}_c^{\mathrm{reg}}$ can be reformulated as a rank-$r$ update to the base matrix $\mathbf{B}$:
\begin{equation}
	\boldsymbol{\Sigma}_c^{\mathrm{reg}} \approx \mathbf{B} + \widetilde{\mathbf{U}}_c \widetilde{\mathbf{U}}_c^\top,
	\label{eq:low_rank_reg_cov}
\end{equation}
where $\widetilde{\mathbf{U}}_c = \alpha_1^{1/2} \mathbf{U}_c \mathbf{S}_c^{1/2}$ absorbs the scaling factor.

\textbf{Efficient inversion via Woodbury matrix identity.}
One of RGDA's main computational bottlenecks is computing $(\boldsymbol{\Sigma}_c^{\mathrm{reg}})^{-1}$. By leveraging the low-rank structure in Eq.~(\ref{eq:low_rank_reg_cov}), we apply the Woodbury matrix identity \cite{hager1989updating} to circumvent the expensive $d \times d$ inversion.
\begin{lemma}[Woodbury Matrix Identity]
	\label{eq:woodbury_identity}
    For an invertible $n \times n$ matrix $\mathbf{A}$, $n \times k$ matrix $\mathbf{U}$, $n \times k$ matrix $\mathbf{V}$, and $k \times k$ matrix $\mathbf{C}$, the following identity holds
	\small
	 \begin{equation}
		(\mathbf{A} + \mathbf{U} \mathbf{C} \mathbf{V}^\top)^{-1} = \mathbf{A}^{-1} - \mathbf{A}^{-1} \mathbf{U} (\mathbf{C}^{-1} + \mathbf{V}^\top \mathbf{A}^{-1} \mathbf{U})^{-1} \mathbf{V}^\top \mathbf{A}^{-1}.
		\label{eq:woodbury_identity_enhanced}
	\end{equation}
	\normalsize
\end{lemma}
Applying Lemma \ref{eq:woodbury_identity}  to $\boldsymbol{\Sigma}_c^{\mathrm{reg}} = \mathbf{B} + \widetilde{\mathbf{U}}_c \widetilde{\mathbf{U}}_c^\top$, we obtain
\begin{equation}
	(\boldsymbol{\Sigma}_{c}^{\mathrm{reg}})^{-1} = \mathbf{B}^{-1} - \mathbf{B}^{-1} \widetilde{\mathbf{U}}_c \mathbf{M}_c^{-1} \widetilde{\mathbf{U}}_c^\top \mathbf{B}^{-1},
	\label{eq:efficient_inverse_enhanced}
\end{equation}
where $\mathbf{M}_c = \mathbf{I}_r + \widetilde{\mathbf{U}}_c^\top \mathbf{B}^{-1} \widetilde{\mathbf{U}}_c$ is a small $r \times r$ matrix. Since $\mathbf{B}^{-1}$ is shared across all classes, it needs to be computed only once. The inversion of class-specific $\mathbf{M}_c$ requires only $\mathcal{O}(r^3)$ operations per class. Therefore, the total construction complexity of RGDA is reduced from $\mathcal{O}(Cd^3)$ to the $\mathcal{O}(d^3 + Crd^2)$ complexity of LR-RGDA.

\textbf{Discriminant function decomposition.}
Beyond efficient construction, the Woodbury identity enables an elegant decomposition of the discriminant function.
\begin{proposition}[Decomposition of LR-RGDA's discriminant function] \label{prop:decomposed}
	    The discriminant function of LR-RGDA, denoted by $g_c^{\mathrm{LR-RGDA}}(\mathbf{x})$, can be decomposed into a global affine term refined by a class-specific quadratic correction in an $r$-dimensional subspace:
	\begin{equation}
		g_c^{\mathrm{LR-RGDA}}(\mathbf{x}) = \mathcal{L}_c(\mathbf{x}) + \mathcal{Q}_c(\mathbf{x}).
		\label{eq:decomposed_discriminant}
	\end{equation}
	Here, $\mathcal{L}_c(\mathbf{x}) = \mathbf{w}_c^\top \mathbf{x} + b_c$ represents the affine discriminant function determined with parameters
	\begin{align}
		\mathbf{w}_c &= \mathbf{B}^{-1}\boldsymbol{\mu}_c, \\
		b_c &= -\tfrac{1}{2}\boldsymbol{\mu}_c^\top \mathbf{B}^{-1}\boldsymbol{\mu}_c
		-\tfrac{1}{2}\log\det(\boldsymbol{\Sigma}_c^{\mathrm{reg}})
		+ \log\pi_c.
	\end{align}
	The term $\mathcal{Q}_c(\mathbf{x}) = \tfrac{1}{2}\mathbf{u}_c^\top \mathbf{M}_c^{-1}\mathbf{u}_c$ represents the class-specific quadratic correction term, where
	\begin{equation}
		\mathbf{u}_c = \mathbf{P}_{c}^{\mathrm{proj}}(\mathbf{x}-\boldsymbol{\mu}_c) 
		 \in \mathbb{R}^r,
		\label{eq:low_dim_projection}
	\end{equation}
	is the projection of the centered feature $(\mathbf{x} - \boldsymbol{\mu}_c)$ onto the $r$-dimensional principal subspace of $\Sigma_{c}$, via the projection matrix $\mathbf{P}_{c}^{\mathrm{proj}} = \widetilde{\mathbf{U}}_c^\top \mathbf{B}^{-1} \in \mathbb{R}^{r \times d}$ (See Appendix \ref{app:proof_decomposed} for detailed derivation)
\end{proposition}
\begin{remark}
Eq.~\eqref{eq:decomposed_discriminant} shifts the quadratic complexity from the original high-dimensional space $\mathbb{R}^d$ to a smaller subspace $\mathbb{R}^r$. This transformation reduces the total inference complexity from $\mathcal{O}(Cd^2)$ to $\mathcal{O}(d^2 + Cdr)$ as $r \ll d$. Consequently, it effectively reconciles the expressivity of RGDA with the inference speed of linear classifiers. We provide a detailed complexity breakdown and comparison in Appendix~\ref{sec:complexity_comp}.
	
\end{remark}

\subsection{Hopfield-based Distribution Compensator}
A critical challenge is the representation drift induced by backbone updates, i.e., $\mathcal{F}_{\boldsymbol{\theta}_{t-1}} \to \mathcal{F}_{\boldsymbol{\theta}_t}$. This semantic shift renders the Gaussian statistics $\mathcal{N}(\boldsymbol{\mu}_c, \boldsymbol{\Sigma}_c)$ for $c \in \mathcal{C}_{t}$, which are originally accumulated under $\mathcal{F}_{\boldsymbol{\theta}_{1:t-1}}$, misaligned with $\mathcal{F}_{\boldsymbol{\theta}_{t}}$. Consequently, reconstructing classifiers with these obsolete priors leads to biased decision boundaries.

To mitigate it, we propose HopDC, a training-free calibration mechanism that aligns historical statistics to the current feature space. Our core insight is to formulate drift estimation as an associative memory problem: we posit that old class prototypes shift consistently with semantically related unlabelled anchors, allowing for precise recalibration without historical exemplars.

\textbf{Drift estimation via anchor associativity.}
Let $\mathcal{A} = \{\mathbf{a}_i\}_{i=1}^N$ be a fixed set of unlabeled anchor images. We define the feature matrices $\mathbf{F}^{\text{old}}$ and $\mathbf{F}^{\text{new}} \in \mathbb{R}^{N \times d}$ by row-wise stacking the embeddings of all anchors computed by the previous and current backbones, respectively. The semantic drift matrix $\mathbf{D} \in \mathbb{R}^{N \times d}$ is derived by
\begin{equation}
	\mathbf{D} = \mathbf{F}^{\text{new}} - \mathbf{F}^{\text{old}}, \quad \text{with } \mathbf{F}_{i,:} = \mathcal{F}_{\boldsymbol{\theta}}(\mathbf{a}_i),
\end{equation}
where the $i$th row $\boldsymbol{\delta}_i$ of $\mathbf{D}$ represents the drift vector of anchor $\mathbf{a}_i$. We can view these anchor points as samples from an underlying continuous drift function $\delta: \mathbb{R}^d \to \mathbb{R}^d$ such that $\delta(\mathbf{k}_i) = \boldsymbol{\delta}_i$.

\textbf{Hopfield energy minimization for drift retrieval.}
Instead of learning a parametric drift mapping, we utilize the dense associative memory properties of MCHNs \cite{DBLP:conf/icml/MillidgeS0LB22, DBLP:conf/iclr/RamsauerSLSWGHA21} to interpolate the drift for old classes. We treat the normalized historical anchor features as \textit{stored patterns} (Keys) and the anchor drift vectors as \textit{state values} (Values)
\begin{equation}
	\mathbf{K} = l_2\text{-Normalize}(\mathbf{F}^{\text{old}}) \in \mathbb{R}^{N \times d}.
\end{equation}
For each old class $c$, we sample $M$ pseudo-features $\mathbf{z}_k \sim \mathcal{N}(\boldsymbol{\mu}_c, \boldsymbol{\Sigma}_c)$ from the stored statistics and normalize them to form the \textit{state queries} $\mathbf{Q}_c \in \mathbb{R}^{M \times d}$. 
The retrieval process follows the energy minimization dynamics of MCHNs, which is implemented via a dot-product attention mechanism. To filter out irrelevant anchor noise, we employ a top-$k$ sparse attention mechanism. Accordingly, the estimated drift $\boldsymbol{\Delta}_c^{\text{est}} \in \mathbb{R}^{M \times d}$ for the pseudo-samples is computed as
\begin{equation}
	\boldsymbol{\Delta}_c^{\text{est}} = \text{Softmax}_{\text{top-}k}\left(\frac{1}{\tau}\mathbf{Q}_c \mathbf{K}^\top\right) \mathbf{D},
\end{equation}
where $\tau$ is the temperature parameter. Specifically, this operation effectively computes a weighted average of anchor drifts, where weights are determined by the semantic similarity between the old class samples and the anchors in the original feature space.

\textbf{Distribution compensation.}
The compensated pseudo-samples are obtained by applying the retrieved drift $\mathbf{Z}_c^{\text{cal}} = \mathbf{Z}_c + \boldsymbol{\Delta}_c^{\text{est}}$, where $\mathbf{Z}_c = [\mathbf{z}_1, \dots, \mathbf{z}_M]^\top$. Finally, we re-estimate the historical statistics for class $c$ using the compensated samples, as 
$\boldsymbol{\mu}_c  = \frac{1}{M} \sum_{j=1}^{M} \mathbf{z}^{(j)}$
and $\boldsymbol{\Sigma}_c = \frac{1}{M} \sum_{j=1}^{M} \mathbf{z}^{(j)} (\mathbf{z}^{(j)})^\top - \boldsymbol{\mu}_c \boldsymbol{\mu}_c^\top,$
where $\mathbf{z}^{(j)}$ denotes the $j$th sample row in $\mathbf{Z}_c^{\text{cal}}$. 

\textbf{Theoretical guarantee.}
The attention‑based retrieval in HopDC corresponds to the minimization of the MCHN energy function (see Appendix \ref{subsec:energy_landscape} for derivation), which can yields an theoretical error bound:

\begin{proposition}[Error bound of HopDC's drift estimation]
	\label{prop:error_bound}
	Let $\mathbf{z}$ be a query point. HopDC's estimate is $\hat{\boldsymbol{\delta}}(\mathbf{z}) = \sum_{i=1}^N p_i \delta(\mathbf{k}_i)$ with $p_i = \exp(\beta \mathbf{k}_i^\top\mathbf{z}) / \sum_j \exp(\beta \mathbf{k}_j^\top\mathbf{z})$, $\beta = 1/\tau$. Assume the true drift function $\delta$ is locally Lipschitz with constant $L$ (see Assumption~\ref{assump:lipschitz}). Then
	\begin{equation}
		\bigl\|\hat{\boldsymbol{\delta}}(\mathbf{z}) - \delta(\mathbf{z})\bigr\|
		\;\le\; L \sum_{i=1}^{N} p_i \|\mathbf{k}_i - \mathbf{z}\|.
		\label{eq:general_bound}
	\end{equation}
	Furthermore, when the anchor keys and the query are $\ell_2$-normalized, a temperature-dependent bound holds:
	\begin{equation}
		\bigl\|\hat{\boldsymbol{\delta}}(\mathbf{z}) - \delta(\mathbf{z})\bigr\| \le L \left( \min_i \|\mathbf{k}_i - \mathbf{z}\| + \sqrt{2\tau\log N} \right),
	\end{equation}
	where $\tau$ is the temperature parameter. For the sparse top-$k$ attention, $N$ is replaced by $k$ (see Appendix \ref{subsec:error_bound} and \ref{subsec:temperature_bound} for derivation).
\end{proposition}
\begin{remark}
	The bound indicates that the estimation error is governed by the attention‑weighted average distance to the anchors. A lower temperature $\tau$ sharpens the attention on the nearest anchor, which reduces the $\sqrt{2\tau\log N}$ term but may increase the estimation variance, as the bound reflects the worst‑case scenario. Empirically, we find $\tau = 0.05$ to strike a good balance in practice.
\end{remark}

\subsection{Algorithmic implementation.} Detailed pseudocodes for the proposed LR-RGDA and HopDC, along with the complete training pipeline, are presented in Appendix~\ref{app:algorithms} (see Algorithms~\ref{alg:main_pipeline}--\ref{alg:lr_rgda_inference}).

\section{Experiment Evaluations}
\begin{table*}[htbp]
	\centering
	\caption{CIL performance evaluations (\%) on four within-domain datasets using ViT/B-Sup21K over three random seeds. For methods including SLCA~\cite{zhang2023slca}, CoMA~\cite{10.1007/978-3-031-73209-6_18}, etc., we report their original results. For RanPAC~\cite{NEURIPS2023_2793dc35} and TSVD~\cite{peng2025tsvd}, we report the accuracies after FSA as presented in~\cite{peng2025tsvd}.}
	\setlength{\tabcolsep}{6pt}
	\renewcommand{\arraystretch}{0.85}
	\small
	\scalebox{0.94}{
		\begin{tabular}{@{}lccccccccccc@{}}
			\toprule
			\multirow{2}{*}{Method}
			& \multicolumn{2}{c}{CUB-200}
			& \multicolumn{2}{c}{Cars-196}
			& \multicolumn{2}{c}{CIFAR-100}
			& \multicolumn{2}{c}{ImageNet-R}
			& \multicolumn{2}{c}{Four datasets} \\
			\cmidrule(lr){2-3} \cmidrule(lr){4-5} \cmidrule(lr){6-7} \cmidrule(lr){8-9} \cmidrule(lr){10-11}
			& Last & Inc & Last & Inc & Last & Inc & Last & Inc & Last & Inc \\
			\midrule
			
			\textbf{Empirical upper bounds} \\
			\quad Joint Training & 88.43 & --- & 86.79 & --- & 93.56 & --- & 83.41 & --- & 88.05 & --- \\
			\midrule
			
			\textbf{Existing baselines} \\
			\quad CODA-Prompt \cite{smith2023coda} & 71.43 & 78.61 & 45.67 & 53.28 & 86.65 & 90.78 & 75.11 & 81.45 & 69.72 & 76.03 \\
			\quad LAE-Adapter \cite{gao2023unified} & 80.52 & 84.75 & 55.20 & 61.63 & 88.37 & 92.50 & 75.69 & 82.80 & 74.95 & 80.42 \\
			\quad RanPAC \cite{NEURIPS2023_2793dc35} & 88.13 & --- & 72.24 & --- & 87.02 & --- & 72.50 & --- & 79.97 & --- \\
			\quad SLCA \cite{zhang2023slca} & 84.71 & 90.94 & 67.73 & 76.93 & 91.53 & 94.09 & 77.00 & 81.17 & 80.24 & 85.78 \\
			\quad SLCA++ \cite{zhang2024slcaunleashpowersequential} & 86.59 & 91.63 & 73.97 & 79.49 & 91.46 & 94.20 & 78.09 & 82.95 & 82.53 & 87.07 \\
			\quad CoMA \cite{10.1007/978-3-031-73209-6_18} & 85.95 & 90.75 & 73.35 & 78.55 & \textbf{92.00} & 94.12 & 77.47 & 81.32 & 82.19 & 86.19 \\
			\quad HiDe-LoRA \cite{10970405} & 88.76 & 89.32 & 69.65 & 69.36 & 91.21 & 93.99 & 79.32 & 83.97 & 82.24 & 84.16 \\
			\quad TSVD \cite{peng2025tsvd} & \textbf{89.23} & --- & 75.13 & --- & 88.18 & --- & 73.63 & --- & 81.54 & --- \\
			\midrule
			
			\textbf{Proposed methods} \\
			\quad RanProj + \textit{LR-RGDA} & 85.73 & 90.92 & 66.71 & 77.18 & 91.28 & 94.01 & 75.93 & 79.64 & 79.91 & 85.44 \\
			\cmidrule(lr){2-11}
			
			\quad SeqFT + \textit{LR-RGDA} & 70.94 & 82.55 & 46.76 & 61.89 & 77.27 & 85.44 & 71.39 & 78.62 & 66.59 & 77.13 \\
			\quad \quad + \textit{HopDC} & 82.27 & 89.70 & 74.38 & 84.65 & 82.52 & 90.10 & 75.96 & 81.85 & 78.78 \textcolor{red}{\tiny +12.19} & 86.58 \textcolor{red}{\tiny +9.45} \\
			\cmidrule(lr){2-11}
			
			\quad SeqKD + \textit{LR-RGDA} & 84.20 & 91.10 & 71.42 & 83.12 & 89.49 & 94.03 & 78.40 & 83.93 & 80.88 & 88.05 \\
			\quad \quad + \textit{HopDC} & 87.22 & \textbf{92.14} & 82.50 & \textbf{87.50} & 90.91 & \textbf{94.52} & \textbf{80.14} & \textbf{84.60} & \textbf{85.19} \textcolor{red}{\tiny +4.31} & \textbf{89.69} \textcolor{red}{\tiny +1.64} \\
			\cmidrule(lr){2-11}
			
			\quad NSP-SeqFT + \textit{LR-RGDA} & 80.62 & 89.44 & 61.51 & 75.38 & 86.59 & 91.76 & 74.49 & 81.23 & 75.80 & 84.45 \\
			\quad \quad + \textit{HopDC} & 85.68 & 91.45 & 81.52 & 86.90 & 89.27 & 93.21 & 77.07 & 82.31 & 83.39 \textcolor{red}{\tiny +7.59} & 88.47 \textcolor{red}{\tiny +4.02} \\
			\cmidrule(lr){2-11}
			
			\quad NSP-SeqKD + \textit{LR-RGDA} & 85.23 & 91.24 & 69.78 & 81.37 & 89.49 & 94.03 & 77.88 & 83.32 & 80.60 & 87.49 \\
			\quad \quad + \textit{HopDC} & 87.65 & 92.12 & \textbf{82.61} & 87.36 & 90.91 & \textbf{94.52} & 79.42 & 83.96 & 85.15 \textcolor{red}{\tiny +4.55} & 89.49 \textcolor{red}{\tiny +2.00} \\
			\cmidrule(lr){2-11}
			
			\quad LoRA-SeqFT + \textit{LR-RGDA} & 83.82 & 89.74 & 57.26 & 69.37 & 88.14 & 91.92 & 76.83 & 81.95 & 76.51 & 83.25 \\
			\quad \quad + \textit{HopDC} & 86.72 & 91.53 & 80.94 & 86.77 & 87.56 & 92.32 & 78.32 & 83.27 & 83.39 \textcolor{red}{\tiny +6.88} & 88.47 \textcolor{red}{\tiny +5.22} \\
			\cmidrule(lr){2-11}
			
			\quad LoRA-SeqKD + \textit{LR-RGDA} & 87.07 & 91.63 & 80.52 & 86.02 & 91.30 & 93.95 & 79.97 & 83.68 & 84.72 & 88.82 \\
			\quad \quad + \textit{HopDC} & 87.07 & 91.65 & 80.73 & 86.07 & 91.60 & 93.93 & 79.80 & 83.94 & 84.80 \textcolor{red}{\tiny +0.08} & 88.90 \textcolor{red}{\tiny +0.08} \\
			\bottomrule
	\end{tabular}}
	\label{tab:results}
	\normalsize
\end{table*}
\subsection{Experimental Setup}
\textbf{Within-domain benchmarks.} 
We evaluate on four standard within-domain CIL datasets: CIFAR-100 \cite{krizhevsky2009learning}, ImageNet-R \cite{hendrycks2021many}, CUB-200 \cite{wah2011caltech}, and Cars-196 \cite{krause20133d}. Following \citet{sun2025pilot}, each dataset is split into 10 disjoint tasks. The average accuracy after the final task ($\textit{Last}$) and the incremental accuracy averaged over all tasks ($\textit{Inc}$) are reported.

\textbf{Baselines.} We apply LR-RGDA and HopDC to several Stage-1 adaptation strategies, which include frozen ViTs with random feature projection (RandProj)~\cite{NEURIPS2023_2793dc35, peng2025tsvd}, naive sequential fine-tuning (SeqFT), SeqFT with feature distillation (SeqKD)~\cite{heo2019comprehensive}, SeqFT with NSP (NSP-SeqFT)~\cite{wang2021training}, and SeqFT with NSP and feature distillation (NSP-SeqKD). We also evaluate LoRA-based strategies (LoRA-SeqFT and LoRA-SeqKD)~\cite{hu2022lora} with rank of 4. Results are presented in two phases: (1) using LR-RGDA alone as an analytic replacement for SGD-based classifiers, and (2) incorporating HopDC to mitigate distribution drift. 

\textbf{Implementation details.}
All experiments are conducted on NVIDIA RTX 4090 GPUs. In Stage 1, we use SGD with a learning rate of $5 \times 10^{-6}$ for full fine-tuning (e.g., SeqFT) and AdamW \cite{loshchilov2019decoupled} with $1 \times 10^{-4}$ for LoRA-based methods. Each task is trained for 1,000 or 1,500 steps with a cosine annealing scheduler. In Stage 2, following standard CIL protocols~\cite{zhang2023slca, wang2022learning}, we assume a uniform class prior distribution (i.e., $\pi_c = 1/|\mathcal{C}_t|$) for all datasets. The subspace rank of LR-RGDA is fixed at $r=64$; its regularization hyperparameters $\alpha_1=0.2$, $\alpha_2=2.0$, and $\alpha_3=0.5$ are selected via the cross-validation. For HopDC, we adopt $\tau=0.05$, $k=400$, and employ an auxiliary set of 1,024 unlabeled anchors drawn from ImageNet-1K \cite{deng2009imagenet}. Following the settings of SLDC \cite{rao2025compensating}, we also utilize samples of the current task as supplement anchors. Notably, we strictly keep the above default configurations throughout all our main experiments. The implementation details are provided in Appendix~\ref{app:implementation}.

\subsection{Results on Within-Domain Benchmarks}
\textbf{Performance on ViT/B-Sup21K.} Table~\ref{tab:results} summarizes the performance on ViT/B-Sup21K \cite{ridnik2021imagenet}. The average results on ViT/B-MoCoV3 \cite{chen2021empirical} are provided in Table~\ref{tab:results_mocov3_new_single}. Here, we focus on comparing LR-RGDA with strong CIL baselines across different Stage-1 strategies. Notable observations include:

\textit{1. New state-of-the-art performance.}
LR-RGDA paired with HopDC consistently establishes new benchmarks. On ViT/B-Sup21K (Table~\ref{tab:results}), our SeqKD variant achieves an average Last-Acc of 85.19\% across four datasets, surpassing the strongest competitor CoMA (82.19\%) by a significant margin of +3.0 pp. Similarly, on ViT/B-MoCoV3 (Appendix Table~\ref{tab:results_mocov3_new}), LoRA-SeqKD reaches 81.32\%, outperforming CoMA by +6.7 pp. It confirms that our analytic framework sets a new upper bound for CIL accuracy.

\textit{2. HopDC effectively resolves representation drift.}
HopDC yields substantial gains for highly plastic Stage-1 strategies. For SeqFT on MoCoV3, it boosts accuracy by +15.9 pp (62.6\% $\to$ 78.5\%), recovering its performance to match stable Stage-1 methods such as LoRA-SeqFT and NSP-SeqFT. This validates that the primary cause of degradation in fine-tuning is statistical misalignment, which HopDC effectively recalibrates.

\textit{3. Universality across optimization strategies.}
Our method generalizes across diverse backbones, supporting both stable (SeqKD) and plastic (SeqFT/LoRA) strategies. While effective on learned features, LR-RGDA slightly trails RanPAC on fixed random projections (RandProj) for complex datasets, likely due to the lower projection dimension and non-Gaussian effects of random ReLU projection.
\subsection{Results on Cross-Domain Benchmarks}
\begin{table}[htbp]
	\centering
	\caption{CIL performance evaluations (\%) on four within-domain datasets using ViT/B-MoCoV3 over three random seeds. The detailed results for each dataset are presented in Table \ref{tab:results_mocov3_new} of Appendix.}
	\setlength{\tabcolsep}{8pt}
	\renewcommand{\arraystretch}{0.7}
	\small
	\scalebox{0.95}{
		\begin{tabular}{@{}lcc@{}}
			\toprule
			\multirow{2}{*}{Method}
			& \multicolumn{2}{c}{Average (Four datasets)} \\
			\cmidrule(lr){2-3}
			& Last & Inc \\
			\midrule
			\textbf{Empirical upper bounds} \\
			\quad Joint-Training & 82.20 & --- \\
			\midrule
			\textbf{Existing baselines} \\
			\quad RanPAC & 73.31 & 80.91 \\
			\quad SLCA++ & 74.74 & 80.72 \\
			\quad CoMA & 74.63 & 81.08 \\
			\midrule
			\textbf{Proposed methods} \\
			\quad RanProj + \textit{LR-RGDA} & 71.91 & 78.52 \\
			\cmidrule(lr){2-3}
			
			\quad SeqFT + \textit{LR-RGDA} & 62.61 & 72.95 \\
			\quad \quad + \textit{HopDC} & 78.50 \textcolor{red}{\tiny +15.89} & 84.56 \textcolor{red}{\tiny +11.61} \\
			\cmidrule(lr){2-3}
			
			\quad SeqKD + \textit{LR-RGDA} & 71.81 & 81.44 \\
			\quad \quad + \textit{HopDC} & 81.22 \textcolor{red}{\tiny +9.41} & 85.81 \textcolor{red}{\tiny +4.37} \\
			\cmidrule(lr){2-3}
			
			\quad NSP-SeqFT + \textit{LR-RGDA} & 74.90 & 82.53 \\
			\quad \quad + \textit{HopDC} & 79.37 \textcolor{red}{\tiny +4.47} & 84.33 \textcolor{red}{\tiny +1.80} \\
			\cmidrule(lr){2-3}
			
			\quad NSP-SeqKD + \textit{LR-RGDA} & 75.87 & 83.28 \\
			\quad \quad + \textit{HopDC} & 79.90 \textcolor{red}{\tiny +4.03} & 84.83 \textcolor{red}{\tiny +1.55} \\
			\cmidrule(lr){2-3}
			
			\quad LoRA-SeqFT + \textit{LR-RGDA} & 63.02 & 73.23 \\
			\quad \quad + \textit{HopDC} & 78.06 \textcolor{red}{\tiny +15.04} & 84.99 \textcolor{red}{\tiny +11.76} \\
			\cmidrule(lr){2-3}
			
			\quad LoRA-SeqKD + \textit{LR-RGDA} & 74.84 & 83.93 \\
			\quad \quad + \textit{HopDC} & \textbf{81.32} \textcolor{red}{\tiny +6.48} & \textbf{86.82} \textcolor{red}{\tiny +2.89} \\
			\bottomrule
		\end{tabular}
	}
	\label{tab:results_mocov3_new_single}
	\normalsize
\end{table}

\begin{table*}[htbp]
	\centering
	\caption{CIL performance evaluations (\%) on a 128-shot 14-task cross-domain benchmark (1,001 classes, seven datasets). We compare the final class-iwse average accuracy (\%) of our analytic classifier (LR-RGDA) against iterative SGD and linear LDA classifiers across diverse backbone strategies and ViTs.}
	\renewcommand{\arraystretch}{0.8}
	\small
	\begin{tabular}{lccccccccc}
		\toprule
		\multirow{2}{*}{Training Method} 
		& \multicolumn{3}{c}{ViT/B-MoCoV3} 
		& \multicolumn{3}{c}{ViT/B-Sup21K} 
		& \multicolumn{3}{c}{ViT/B-CLIP} \\
		\cmidrule(lr){2-4} \cmidrule(lr){5-7} \cmidrule(lr){8-10}
		& LR-RGDA & SGD & LDA & LR-RGDA & SGD & LDA & LR-RGDA & SGD & LDA \\
		\midrule
		SeqFT               & \textbf{67.13} & 62.21 & 62.30 & \textbf{75.44} & 73.07 & 71.65 & \textbf{65.40} & 65.18 & 64.50 \\
		\quad + $\alpha_1$-SLDC & \textbf{69.55} & 65.09 & 64.43 & \textbf{74.84} & 72.56 & 72.04 & 66.77 & \textbf{67.00} & 66.45 \\
		\quad + HopDC       & \textbf{71.45} & 68.35 & 66.25 & \textbf{78.31} & 76.42 & 73.31 & \textbf{71.40} & 69.42 & 68.68 \\
		\midrule
		SeqKD               & \textbf{73.39} & 70.39 & 68.13 & \textbf{81.21} & 79.50 & 78.20 & \textbf{71.33} & 70.18 & 69.42 \\
		\quad + $\alpha_1$-SLDC & \textbf{73.52} & 71.08 & 68.24 & \textbf{80.97} & 79.65 & 77.45 & \textbf{76.24} & 74.47 & 74.54 \\
		\quad + HopDC       & \textbf{74.60} & 72.42 & 69.50 & \textbf{81.71} & 80.41 & 77.99 & \textbf{77.02} & 74.71 & 73.93 \\
		\midrule
		NSP-SeqKD           & \textbf{73.31} & 71.41 & 68.94 & \textbf{80.32} & 78.93 & 76.49 & \textbf{80.70} & 79.41 & 78.81 \\
		\quad + $\alpha_1$-SLDC & \textbf{72.84} & 72.13 & 68.01 & \textbf{80.38} & 79.54 & 76.36 & \textbf{81.00} & 79.60 & 78.90 \\
		\quad + HopDC       & \textbf{74.11} & 73.05 & 69.78 & \textbf{81.42} & 80.45 & 77.20 & \textbf{82.31} & 80.84 & 79.80 \\
		\midrule
		LoRA-SeqFT          & \textbf{68.44} & 61.84 & 63.14 & \textbf{77.47} & 74.43 & 73.18 & \textbf{72.62} & 71.26 & 70.79 \\
		\quad + $\alpha_1$-SLDC & \textbf{67.86} & 63.38 & 63.57 & \textbf{76.27} & 74.07 & 72.61 & \textbf{76.23} & 73.04 & 73.98 \\
		\quad + HopDC       & \textbf{72.44} & 68.24 & 66.88 & \textbf{79.37} & 77.10 & 74.33 & \textbf{77.16} & 75.27 & 74.33 \\
		\midrule
		LoRA-SeqKD          & \textbf{74.64} & 70.87 & 69.36 & \textbf{81.90} & 80.39 & 78.25 & \textbf{80.52} & 78.88 & 78.15 \\
		\quad + $\alpha_1$-SLDC & \textbf{74.20} & 71.22 & 69.15 & \textbf{81.67} & 80.20 & 77.86 & \textbf{80.88} & 78.96 & 78.26 \\
		\quad + HopDC       & \textbf{75.72} & 73.08 & 70.91 & \textbf{82.60} & 81.41 & 78.56 & \textbf{81.86} & 80.44 & 78.98 \\
		\bottomrule
	\end{tabular}
	\label{tab:merged_performance}
	\normalsize
\end{table*}

To evaluate robustness in a large-scale CIL scenario, we construct a 14-task cross-domain benchmark spanning seven diverse datasets: CIFAR-100, ImageNet-R, Cars-196, CUB-200, Caltech-101 \cite{fei2004learning}, Flower-102 \cite{nilsback2008automated}, and Food-101 \cite{bossard2014food}. Each dataset is split into two disjoint tasks, and we enforce class balance by capping samples at 128 per class. From Table~\ref{tab:merged_performance}, the key observations are:

\textit{1. LR-RGDA consistently outperforms SGD-based classifiers.}
Across all settings, LR-RGDA achieves higher accuracy than SGD classifiers. For example, with ViT/B-Sup21K, LR-RGDA surpasses SGD by an average of +0.8 pp to +2.4 pp across different backbone strategies.

\textit{2. HopDC acts as a universal drift compensator.}
HopDC consistently boosts performance over both the baseline (no compensation) and the competing method $\alpha_1$-SLDC \cite{rao2025compensating}. On ViT/B-CLIP with SeqFT, HopDC yields a +6.0 pp gain (65.40\% $\to$ 71.40\%) over the baseline and outperforms $\alpha_1$-SLDC by +4.6 pp. It confirms that using associative memory with unlabelled anchors is more effective than simple linear transformation of $\alpha_1$-SLDC for mitigating semantic drift.


\subsection{Ablation studies on LR-RGDA}
\begin{figure}
	\centering
	\includegraphics[width=1.0\linewidth]{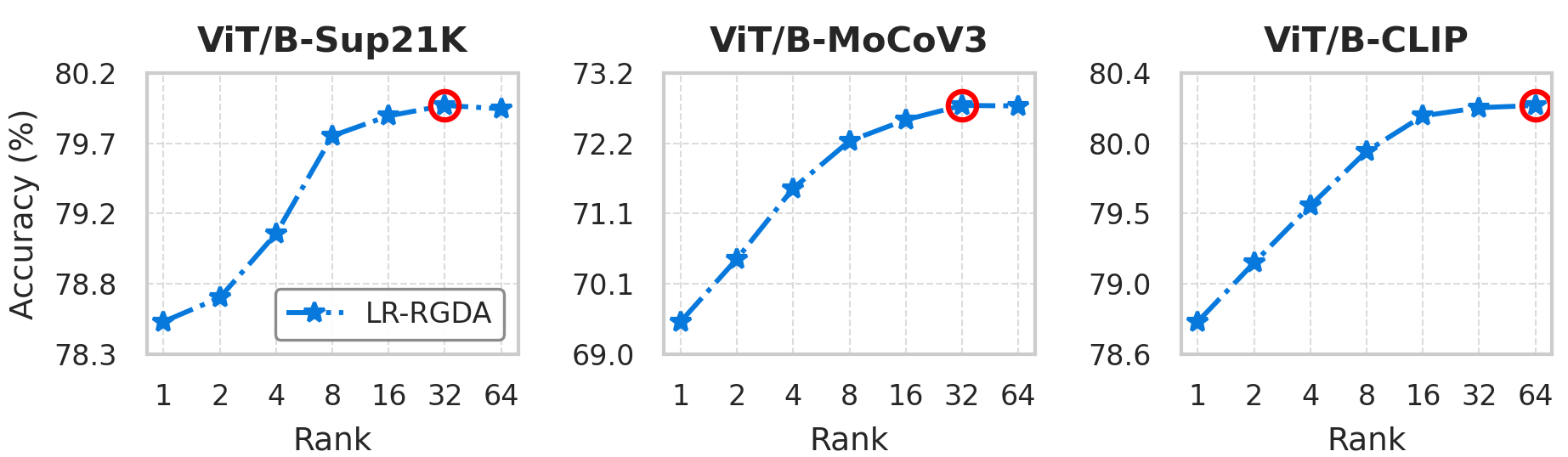}
	\caption{Ablation study on the low-rank dimension $r$ of LR-RGDA on the cross-domain CIL dataset with different ViTs.}
	\label{fig:rankablationcomparison}
	\vspace{-1pt}
\end{figure}
\textbf{Impact of subspace dimension.} Here we assess the influence of the subspace rank \( r \) in LR-RGDA on the cross-domain CIL dataset with fixed hyperparameters (\(\alpha_1 = 0.2, \alpha_2 = 2.0, \alpha_3 = 0.5\)). The ViT models undergo short-term full fine-tuning for 500 iterations. As shown in Figure \ref{fig:rankablationcomparison}, performance across all ViTs improves sharply as \( r \) increases from 1 to 16. It implies that the leading principal components of the class covariance matrices encapsulate the most discriminative information. In Appendix \ref{sec:ablation_rank}, we further present a comprehensive ablation study on the impact of subspace rank across cross-domain and within-domain datasets, both with and without additional fine-tuning (see Figures \ref{fig:rankablationcomparisonfour}, \ref{fig:rankablationcomparisonfour_iter500}, \ref{fig:rankablationcomparison_frozen}, and \ref{fig:rankablationcomparison_finetuned}).

\textbf{Sensitivity to regularization hyperparameters.}
\begin{figure}[t]
	\centering
	\includegraphics[width=1.0\linewidth]{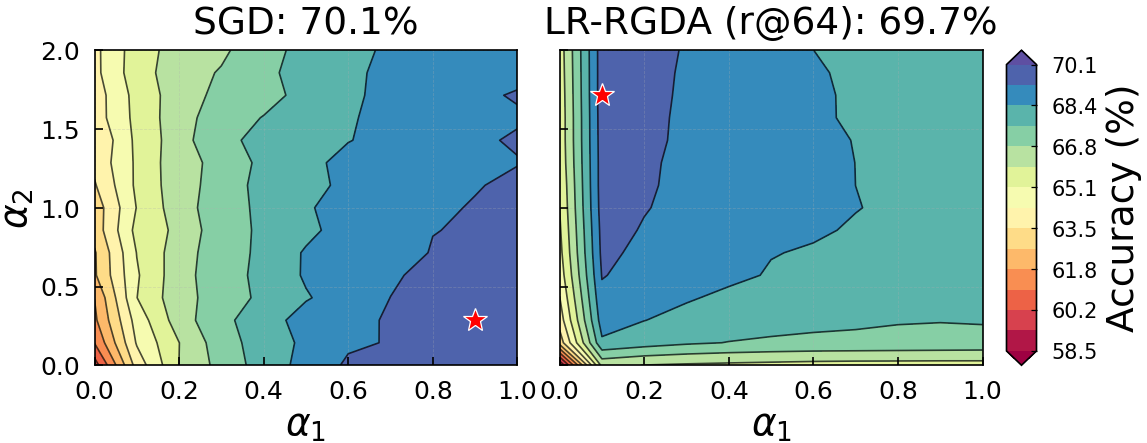}
	\caption{Sensitivity analysis of $\alpha_1$ and $\alpha_2$ in regularized covariance matrices for SGD-based and LR-RGDA classifiers.}
	\label{fig:contourcombinedvit-b-p16classwise}
\end{figure}
Utilizing the frozen ViT/B-Sup21K backbone while fixing $\alpha_3 = 0.5$, we further examine the hyperparameter sensitivity concerning $\alpha_1$ and $\alpha_2$. As illustrated by the accuracy contours in Figure \ref{fig:contourcombinedvit-b-p16classwise}, a sharp divergence emerges between the two classifiers: whereas SGD favors class-specific covariance ($\alpha_1 \to 1, \alpha_2 \to 0$), the analytic LR-RGDA relies predominantly on global regularization, characterized by small $\alpha_1 \approx 0.1$ and $\alpha_2 > 1.0$. This disparity underscores the distinct mechanisms that differentiate iterative optimization from analytic optimization, which presents an open question for future theoretical research. Furthermore, in Appendix \ref{sec:contours}, we provide a more detailed ablation study that investigates dataset-wise sensitivity and the sensitivity contours of different ViTs (see Figures \ref{fig:exp3datasetsensitivityvit-b-p16-clipalpha1} and \ref{fig:multirankcontourvit-b-p16iter0}).

\textbf{Empirical inference speed comparison.}
\begin{figure}[!htbp]
	\centering
	\includegraphics[width=0.95\linewidth]{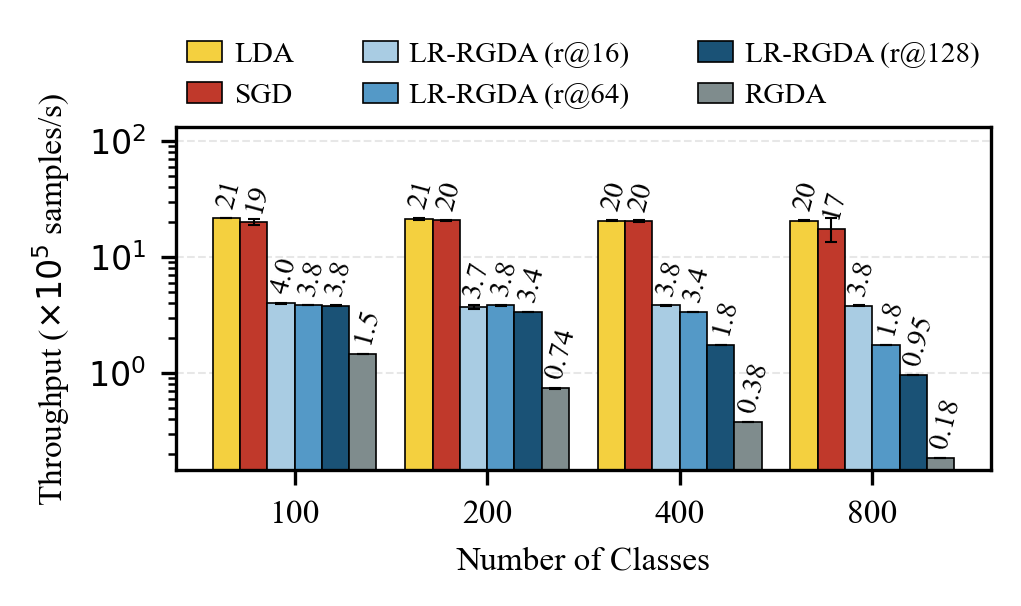}
	\caption{Inference throughput comparison of different classifiers (log scale). We compare the samples-per-second throughput of LR-RGDA against linear baselines (LDA, SGD) and the full-rank QDA as the number of classes increases.}
	\label{fig:throughputrot45}
\end{figure}
We evaluate the inference throughput on an RTX 4090 GPU as the number of classes increases from 100 to 800. As shown in Figure \ref{fig:throughputrot45}, while linear baselines (LDA, SGD) maintain consistently high speeds due to efficient parallelization, QDA suffers from a sharp decline because its computational complexity scales as $\mathcal{O}(Cd^2)$. LR-RGDA effectively bridges this performance gap: when operating at low ranks (e.g., $r=16$), the overhead introduced by the quadratic term is negligible, which enables LR-RGDA to sustain high throughput that remains almost constant as class count grows. Even at a rank of $r=128$, LR-RGDA is orders of magnitude faster than QDA. A more comprehensive empirical analysis, which includes construction time from scratch, classifier-only throughput, and end-to-end throughput, is provided in Appendix \ref{sub:comparison} (see Figure \ref{fig:computation_comparison}).

\subsection{Ablation studies on HopDC}

We evaluate the sensitivity of HopDC to sparsity $k$, temperature $\tau$, and auxiliary set size $N$, whose detailed quantitative results are provided in Appendix~\ref{sec:app_robustness}. Empirically, while the method demonstrates remarkable robustness to $k$ where varying neighbor counts yields negligible fluctuations (Table~\ref{tab:app_hyperparams}), the optimal temperature $\tau$ proves task-dependent, favoring lower values that filter out semantic noise. Furthermore, regarding data efficiency (Table~\ref{tab:app_data_efficiency}), HopDC maintains near-constant accuracy even when $N$ is reduced to 256 samples, which outperforms the recent optimal compensation method $\alpha_1$-SLDC~\cite{rao2025compensating} which suffers from unreliable global estimates under data scarcity.

\section{Conclusion}
This work addresses two critical bottlenecks in exemplar-free CIL with pre-trained ViTs: the computational inefficiency of SGD-based classifier reconstruction and the distribution misalignment induced by backbone updates. To this end, we propose a unified framework integrating LR-RGDA and HopDC. Specifically, as an alternative to SGD-based classifiers, LR-RGDA exploits the Woodbury matrix identity to decompose the high-dimensional discriminant function into a global affine term refined by low-rank quadratic perturbations. This design effectively reconciles the expressivity of RGDA with the linear inference efficiency of LDA. Furthermore, HopDC harnesses the associative memory dynamics of MCHNs to rectify historical distribution drifts using unlabeled anchors. Crucially, we theoretically ground this mechanism by deriving an explicit upper bound on the estimation error. Extensive experiments across diverse backbone optimization strategies, pre-trained ViTs, and CIL benchmarks confirm the superiority of our approach.

Despite these contributions, limitations remain: LR-RGDA has yet to be validated in online CIL scenarios, and HopDC is currently inapplicable to settings with undefined task boundaries. Additionally, the attention mechanism in HopDC introduces minor computational overhead. Future work will focus on optimizing the efficiency of the associative memory dynamics to facilitate broader deployment.


\section*{Impact Statement}
This paper presents work whose goal is to advance the field of Machine Learning, specifically in addressing efficiency and robustness challenges within class-incremental learning. There are many potential societal consequences of our work, none which we feel must be specifically highlighted here.

\bibliography{example_paper}
\bibliographystyle{icml2026}

\newpage
\onecolumn
\appendix

\section{Theoretical Analysis and Proofs}
\label{app:theory}

In this appendix, we provide complete mathematical derivations to support the proposed LR-RGDA framework. First, we present the formal proof of the Bayes-optimality of the RGDA's discriminant function (\textbf{Lemma \ref{eq:bayes_opt}}). Second, we proof the decomposition of the LR-RGDA's discriminant function (\textbf{Proposition \ref{prop:decomposed}}), which demonstrates that how the Woodbury matrix identity (\textbf{Lemma \ref{eq:woodbury_identity}}) enables the separation of the global affine term and the local low-rank quadratic corrections. Finally, we offer a detailed complexity analysis on comparing the storage, construction, and inference complexity of our method with LDA and RGDA classifiers, so as to theoretically validate the efficiency gains reported in the main paper.

\subsection{Proof of Bayes Optimality of RGDA}
\label{app:proof_bayes}

\noindent \textbf{Lemma \ref{eq:bayes_opt}} (Bayes Optimality of RGDA). \textit{Under the assumption that the class-wise features follow a Gaussian distribution $\mathcal{N}\left( \boldsymbol{\mu}_{c}, \boldsymbol{\Sigma}_{c}^{\mathrm{reg}} \right)$, maximizing the RGDA discriminant function corresponds to the Bayes-optimal decision rule.}

\begin{proof}
	Under the Gaussian distribution assumption for the class-wise deep features captured by the ViT mapping, the conditional probability density of a feature $\mathbf{x}\in\mathbb{R}^{d}$ under class $c$ is
	\begin{align}
		p(\mathbf{x} \mid c)=
		\frac{\exp\!\left(-\frac{1}{2} (\mathbf{x}-\boldsymbol{\mu}_{c})^{\top}
			\bigl(\boldsymbol{\Sigma}_{c}^{\mathrm{reg}}\bigr)^{-1}
			(\mathbf{x}-\boldsymbol{\mu}_{c})\right)}
		{(2\pi)^{d/2}|\boldsymbol{\Sigma}_{c}^{\mathrm{reg}}|^{1/2}},
	\end{align}
	where the regularized covariance matrix is defined as
	$\boldsymbol{\Sigma}_{c}^{\mathrm{reg}}
	= \alpha_{1}\boldsymbol{\Sigma}_{c}
	+ \alpha_{2}\boldsymbol{\Sigma}_{\mathrm{avg}}
	+ \alpha_{3}I_{d}$.
	Let the prior probability of class $c$ be $\pi_{c}=p(c)$, which satisfies
	$\sum_{c=1}^{C}\pi_{c}=1$.
	The Bayes-optimal classifier selects the class $\hat{c}$ that maximizes the posterior probability as
	\begin{align}
		\hat{c}(\mathbf{x})=\mathop{\arg\max}_{c}\; p(c\mid \mathbf{x})
		=\mathop{\arg\max}_{c}\; \bigl[ p(c)p(\mathbf{x}\mid c) \bigr].
	\end{align}
	By taking logarithms and dropping class-independent constants (specifically, the term $-\frac{d}{2}\log(2\pi)$), the maximization of the log-posterior is equivalent to maximizing
	\begin{align}
		\mathcal{J}(\mathbf{x}, c)
		&= \log \pi_{c}
		-\frac{1}{2} (\mathbf{x}-\boldsymbol{\mu}_{c})^{\top}
		\bigl(\boldsymbol{\Sigma}_{c}^{\mathrm{reg}}\bigr)^{-1}
		(\mathbf{x}-\boldsymbol{\mu}_{c})
		-\frac{1}{2} \log |\boldsymbol{\Sigma}_{c}^{\mathrm{reg}}| .
	\end{align}
	This expression is exactly the discriminant function of RGDA, which is denoted by
	\begin{align}
		g_{c}^{\mathrm{RGDA}}(\mathbf{x})=
		-\frac{1}{2} (\mathbf{x}-\boldsymbol{\mu}_{c})^{\top}
		\bigl(\boldsymbol{\Sigma}_{c}^{\mathrm{reg}}\bigr)^{-1}
		(\mathbf{x}-\boldsymbol{\mu}_{c})
		-\frac{1}{2} \log |\boldsymbol{\Sigma}_{c}^{\mathrm{reg}}|
		+\log \pi_{c}.
	\end{align}
	Therefore, for any input $\mathbf{x}$, we have
	\begin{align}
		\hat{c}(\mathbf{x}) = \mathop{\arg\max}_{c}\; p(c\mid \mathbf{x}) \iff \hat{c}(\mathbf{x}) = \mathop{\arg\max}_{c}\; g_{c}^{\mathrm{RGDA}}(\mathbf{x}).
	\end{align}
	It confirms that the RGDA classifier implements the Bayes-optimal decision rule under the assumed regularized Gaussian distributions. \qedhere
\end{proof}

\subsection{Derivation of Discriminant Decomposition (Proof of Proposition~\ref{prop:decomposed})}
\label{app:proof_decomposed}

\noindent \textbf{Proposition \ref{prop:decomposed}} (Decomposition of LR-RGDA's discriminant function).
\textit{
	The discriminant function of LR-RGDA, denoted by $g_c^{\mathrm{LR\text{-}RGDA}}(\mathbf{x})$, can be decomposed into a global affine term refined by a class-specific quadratic correction in an $r$-dimensional subspace as
	\begin{equation}
		g_c^{\mathrm{LR\text{-}RGDA}}(\mathbf{x}) = \mathcal{L}_c(\mathbf{x}) + \mathcal{Q}_c(\mathbf{x}).
	\end{equation}
	Here, $\mathcal{L}_c(\mathbf{x}) = \mathbf{w}_c^\top \mathbf{x} + b_c$ represents the affine discriminant function determined with parameters
	\begin{align}
		\mathbf{w}_c &= \mathbf{B}^{-1}\boldsymbol{\mu}_c, \\
		b_c &= -\frac{1}{2}\boldsymbol{\mu}_c^\top \mathbf{B}^{-1}\boldsymbol{\mu}_c
		-\frac{1}{2}\log\det(\boldsymbol{\Sigma}_c^{\mathrm{reg}})
		+ \log\pi_c.
	\end{align}
	The term $\mathcal{Q}_c(\mathbf{x}) = \frac{1}{2}\mathbf{u}_c^\top \mathbf{M}_c^{-1}\mathbf{u}_c$ represents the class-specific quadratic correction term, where
	\begin{equation}
		\mathbf{u}_c = \mathbf{P}_{c}^{\mathrm{proj}}(\mathbf{x}-\boldsymbol{\mu}_c)
		\in \mathbb{R}^r
	\end{equation}
	is the projection of the centered feature $(\mathbf{x} - \boldsymbol{\mu}_c)$ onto the $r$-dimensional principal subspace of $\boldsymbol{\Sigma}_{c}$, via the projection matrix $\mathbf{P}_{c}^{\mathrm{proj}} = \widetilde{\mathbf{U}}_c^\top \mathbf{B}^{-1} \in \mathbb{R}^{r \times d}$.
}

\begin{proof}
	We start from the RGDA discriminant function
	\begin{align}
		g_c^{\mathrm{RGDA}}(\mathbf{x})=
		-\frac{1}{2}(\mathbf{x}-\boldsymbol{\mu}_c)^\top
		\bigl(\boldsymbol{\Sigma}_c^{\mathrm{reg}}\bigr)^{-1}
		(\mathbf{x}-\boldsymbol{\mu}_c) -\frac{1}{2}\log\det(\boldsymbol{\Sigma}_c^{\mathrm{reg}})
		+\log\pi_c.
	\end{align}
	By applying the Woodbury matrix identity to $\bigl(\boldsymbol{\Sigma}_c^{\mathrm{reg}}\bigr)^{-1}$, we obtain
	\begin{align}
		g_c^{\mathrm{RGDA}}(\mathbf{x}) &=
		-\frac{1}{2}(\mathbf{x}-\boldsymbol{\mu}_c)^\top
		\Bigl(\mathbf{B}^{-1}
		-\mathbf{B}^{-1}\widetilde{\mathbf{U}}_c\mathbf{M}_c^{-1}
		\widetilde{\mathbf{U}}_c^\top\mathbf{B}^{-1}\Bigr)
		(\mathbf{x}-\boldsymbol{\mu}_c) \nonumber \\
		&\quad - \frac{1}{2}\log\det(\boldsymbol{\Sigma}_c^{\mathrm{reg}})
		+\log\pi_c,
	\end{align}
	where $\mathbf{M}_c = \mathbf{I}_r + \widetilde{\mathbf{U}}_c^\top \mathbf{B}^{-1} \widetilde{\mathbf{U}}_c$. Expanding the quadratic form yields
	\begin{align}
		\label{eq:232}
		(\mathbf{x}-\boldsymbol{\mu}_c)^\top
		\bigl(\boldsymbol{\Sigma}_c^{\mathrm{reg}}\bigr)^{-1}
		(\mathbf{x}-\boldsymbol{\mu}_c)
		&=
		(\mathbf{x}-\boldsymbol{\mu}_c)^\top\mathbf{B}^{-1}(\mathbf{x}-\boldsymbol{\mu}_c) \nonumber \\
		&\quad -\bigl(\widetilde{\mathbf{U}}_c^\top
		\mathbf{B}^{-1}(\mathbf{x}-\boldsymbol{\mu}_c)\bigr)^{\top}
		\mathbf{M}_c^{-1}
		\bigl(\widetilde{\mathbf{U}}_c^\top
		\mathbf{B}^{-1}(\mathbf{x}-\boldsymbol{\mu}_c)\bigr).
	\end{align}
	Let the projected residual be $\mathbf{u}_c = \widetilde{\mathbf{U}}_c^\top \mathbf{B}^{-1}(\mathbf{x}-\boldsymbol{\mu}_c)$. Then, \eqref{eq:232} becomes
	\begin{align}
		\label{eq:quadratic_rewrite}
		(\mathbf{x}-\boldsymbol{\mu}_c)^\top
		\bigl(\boldsymbol{\Sigma}_c^{\mathrm{reg}}\bigr)^{-1}
		(\mathbf{x}-\boldsymbol{\mu}_c)
		&= (\mathbf{x}-\boldsymbol{\mu}_c)^\top \mathbf{B}^{-1} (\mathbf{x}-\boldsymbol{\mu}_c)
		-\mathbf{u}_c^\top \mathbf{M}_c^{-1}\mathbf{u}_c.
	\end{align}
	The first term in \eqref{eq:quadratic_rewrite} expands as
	\begin{align}
		(\mathbf{x}-\boldsymbol{\mu}_c)^\top \mathbf{B}^{-1} (\mathbf{x}-\boldsymbol{\mu}_c)
		= \mathbf{x}^\top\mathbf{B}^{-1}\mathbf{x}
		-2\boldsymbol{\mu}_c^\top\mathbf{B}^{-1}\mathbf{x}
		+\boldsymbol{\mu}_c^\top\mathbf{B}^{-1}\boldsymbol{\mu}_c .
	\end{align}
	Since the term $\mathbf{x}^\top\mathbf{B}^{-1}\mathbf{x}$ is class-independent, it can be dropped when applying the $\arg \max$ operator for classification. Collecting the remaining terms gives the discriminant function of LR-RGDA as
	\begin{align}\label{eq:lr_rgda}
		g_c^{\mathrm{LR\text{-}RGDA}}(\mathbf{x})
		= \boldsymbol{\mu}_c^\top\mathbf{B}^{-1}\mathbf{x}
		-\frac{1}{2}\boldsymbol{\mu}_c^\top\mathbf{B}^{-1}\boldsymbol{\mu}_c
		+\frac{1}{2}\mathbf{u}_c^\top \mathbf{M}_c^{-1}\mathbf{u}_c - \frac{1}{2}\log\det(\boldsymbol{\Sigma}_c^{\mathrm{reg}})
		+\log\pi_c .
	\end{align}
	Finally, by regrouping the terms in \eqref{eq:lr_rgda}, we identify the affine part $\mathcal{L}_c(\mathbf{x})$ and the quadratic correction $\mathcal{Q}_c(\mathbf{x})$ as
	\begin{align}
		\mathcal{L}_c(\mathbf{x})
		&= \underbrace{\boldsymbol{\mu}_c^\top\mathbf{B}^{-1}}_{\mathbf{w}_c^\top}\mathbf{x}
		\underbrace{-\frac{1}{2}\boldsymbol{\mu}_c^\top\mathbf{B}^{-1}\boldsymbol{\mu}_c
			-\frac{1}{2}\log\det(\boldsymbol{\Sigma}_c^{\mathrm{reg}})
			+ \log\pi_c}_{b_c}, \\
		\mathcal{Q}_c(\mathbf{x})
		&= \frac{1}{2}\mathbf{u}_c^\top \mathbf{M}_c^{-1}\mathbf{u}_c.
	\end{align}
	Thus, $g_c^{\mathrm{LR\text{-}RGDA}}(\mathbf{x}) = 	\mathcal{L}_c(\mathbf{x}) + \mathcal{Q}_c(\mathbf{x})$.
	\qedhere
\end{proof}

\subsection{Complexity Comparison Analysis}
\label{sec:complexity_comp}
In this section, we provide the detailed comparison of the storage, construction, and inference complexities for the proposed LR-RGDA against standard LDA and RGDA classifiers. Our analysis assumes that class-wise statistics ($\boldsymbol{\mu}_c, \boldsymbol{\Sigma}_c$) are available from the incremental accumulation stage. Therefore, ``construction'' refers strictly to the computation required to derive discriminant parameters from these statistics. We first summarize the pre-computed terms required by each classifier in Table~\ref{tab:precomputation} and the resulting complexity metrics in Table~\ref{tab:complexities}.

\begin{table}[htbp]
	\centering
	\caption{Summary of pre-computed terms during the inference of different discriminant analysis classifiers under the provided statistics $\mathcal{N}\left( \boldsymbol{\mu}_{c}, \boldsymbol{\Sigma}_{c} \right)$ for $c \in \mathcal{C}$. Note that the cost of calculating the bias terms $b_c$, which involves the quadratic form and determinant, can be absorbed into the listed complexities due to the reuse of intermediate variables (e.g., $\mathbf{w}_c$ and $\mathbf{M}_c$). Here, $d$ is the feature dimension, $C$ is the number of classes, and $r$ is the rank of the class-specific subspace ($r \ll d$).}
	\begin{tabular}{l l l l}
		\toprule
		Method & Scope & Pre-computed terms during inference & Computational cost \\
		\midrule
		\multirow{2}{*}{LDA} & Global & Precision matrix $\boldsymbol{\Sigma}_{\mathrm{avg}}^{-1}$ & $\mathcal{O}\left(d^3\right)$ \\
		\cmidrule{2-4}
		& Class-specific & Affine weights $\mathbf{w}_c, b_c$ & $\mathcal{O}\left(C d^2\right)$ \\
		\midrule
		\multirow{2}{*}{RGDA} & \multirow{2}{*}{Class-specific} & Precision matrix $\left(\boldsymbol{\Sigma}_c^{\mathrm{reg}}\right)^{-1}$ & \multirow{2}{*}{$\mathcal{O}\left(C d^3\right)$} \\
		& & Determinant $\det\left(\boldsymbol{\Sigma}_c^{\mathrm{reg}}\right)$ & \\
		\midrule
		\multirow{5}{*}{LR-RGDA} & Global & Base inverse $\mathbf{B}^{-1}$ & $\mathcal{O}\left(d^3\right)$ \\
		\cmidrule{2-4}
		& \multirow{4}{*}{Class-specific} & SVD component $\widetilde{\mathbf{U}}_c$ & $\mathcal{O}\left(C d^2 r\right)$ \\
		& & Core inverse $\mathbf{M}_c^{-1}$ & $\mathcal{O}\left(C d^2 r\right)$ \\
		& & Projection matrix $\mathbf{P}_c^{\mathrm{proj}}$ & $\mathcal{O}\left(C d^2 r\right)$ \\
		& & Affine weights $\mathbf{w}_c, b_c$ & $\mathcal{O}\left(C d^2\right)$ \\
		\bottomrule
	\end{tabular}
	\label{tab:precomputation}
\end{table}

\begin{table}[htbp]
	\centering
	\caption{Storage, construction and inference complexities comparison of LDA, RGDA and LR-RGDA classifiers. }
	\begin{tabular}{llll}
		\toprule
		Method & Storage complexity & Construction complexity & Inference complexity (per-sample)\\
		\midrule
		LDA &
		$\mathcal{O}\left(d^{2}+Cd\right)$ &
		$\mathcal{O}\left(d^{3} + Cd^{2}\right)$ &
		$\mathcal{O}\left(Cd\right)$ \\
		RGDA &
		$\mathcal{O}\left(Cd^{2}\right)$ &
		$\mathcal{O}\left(Cd^{3}\right)$&
		$\mathcal{O}\left(Cd^{2}\right)$ \\
		LR-RGDA &
		$\mathcal{O}\left(d^{2}+Cdr\right)$ &
		$\mathcal{O}\left(d^{3}+Cd^{2}r\right)$ &
		$\mathcal{O}\left(d^{2}+Cdr\right)$ \\
		\bottomrule
	\end{tabular}
	\label{tab:complexities}
\end{table}

\noindent \subsubsection{Complexities of LDA}
LDA assumes a homoscedastic Gaussian distribution where all classes share a single global covariance matrix $\boldsymbol{\Sigma}_{\mathrm{avg}}$. Consequently, the discriminant function can be simplified into a linear boundary as
\begin{equation}
	g_c^{\mathrm{LDA}}(\mathbf{x}) = \mathbf{w}_c^\top \mathbf{x} + b_c, 
\end{equation}
where the weight vector is $\mathbf{w}_c = \boldsymbol{\Sigma}_{\mathrm{avg}}^{-1}\boldsymbol{\mu}_c$, and the bias term is defined as $b_c = -\frac{1}{2}\boldsymbol{\mu}_c^\top \mathbf{w}_c + \log \pi_c$.

\begin{itemize}
	\item \textbf{Storage}: $\mathcal{O}(d^2 + Cd)$. 
	Although a linear classifier typically requires only $\mathcal{O}(Cd)$ storage for parameters, in the CIL context, we additionally maintain the shared inverse covariance matrix $\boldsymbol{\Sigma}_{\mathrm{avg}}^{-1}$ to facilitate future incremental updates. Thus, the total storage includes the shared matrix plus the class-specific vector $\mathbf{w}_c$ and scalar $b_c$.
	
	\item \textbf{Construction}: $\mathcal{O}(d^3 + Cd^2)$. 
	The construction phase consists of two primary operations. First, inverting the shared covariance matrix $\boldsymbol{\Sigma}_{\mathrm{avg}}$ incurs a one-time cost of $\mathcal{O}(d^3)$. Second, computing $\mathbf{w}_c$ requires a dense matrix-vector multiplication with the cost of $\mathcal{O}(d^2)$. Note that the bias $b_c$ is efficiently derived via the dot product $\mathbf{w}_c^\top \boldsymbol{\mu}_c$ with the cost of $\mathcal{O}(d)$, which imposes negligible overhead compared to computing $\mathbf{w}_c$.
	
	\item \textbf{Inference}: $\mathcal{O}(Cd)$. 
	During the inference, classification depends on the pre-computed parameters only. Evaluating $g_c^{\mathrm{LDA}}(\mathbf{x})$ involves a simple dot product between the input $\mathbf{x}$ and $\mathbf{w}_c$ for each class. However, while this linear scaling makes LDA extremely fast, it limits LDA's ability to leverage the class-specific fine-grained information contained in covariance matrices.
\end{itemize}

\noindent \subsubsection{Complexities of RGDA}
RGDA relaxes the homoscedastic assumption by modeling each class with a unique regularized covariance matrix $\boldsymbol{\Sigma}_c^{\mathrm{reg}}$. The discriminant function is dominated by the class-specific Mahalanobis distance
\begin{equation}
	g_c^{\mathrm{RGDA}}(\mathbf{x}) = -\frac{1}{2}(\mathbf{x}-\boldsymbol{\mu}_c)^\top (\boldsymbol{\Sigma}_c^{\mathrm{reg}})^{-1} (\mathbf{x}-\boldsymbol{\mu}_c) + \log\pi_c - \frac{1}{2}\log\det(\boldsymbol{\Sigma}_c^{\mathrm{reg}}).
\end{equation}
\begin{itemize}
	\item \textbf{Storage}: $\mathcal{O}(Cd^2)$. 
	Unlike LDA, RGDA must store a unique precision matrix $(\boldsymbol{\Sigma}_c^{\mathrm{reg}})^{-1}$ of size $d \times d$ for every class $c \in \{1, \dots, C\}$. This quadratic dependency on $d$ multiplied by the number of classes $C$ creates a severe memory bottleneck in large-scale scenarios.
	
	\item \textbf{Construction}: $\mathcal{O}(Cd^3)$. 
	This represents the most significant computational bottleneck. Constructing the classifier requires inverting the covariance matrix for each individual class. Since a single matrix inversion scales as $\mathcal{O}(d^3)$, performing this operation for all $C$ classes results in the  $\mathcal{O}(Cd^3)$ complexity.
	
	\item \textbf{Inference}: $\mathcal{O}(Cd^2)$. 
	The inference bottleneck lies in evaluating the quadratic term $(\mathbf{x}-\boldsymbol{\mu}_c)^\top (\boldsymbol{\Sigma}_c^{\mathrm{reg}})^{-1} (\mathbf{x}-\boldsymbol{\mu}_c)$. This operation necessitates a vector-matrix-vector multiplication chain: first computing $(\boldsymbol{\Sigma}_c^{\mathrm{reg}})^{-1} (\mathbf{x}-\boldsymbol{\mu}_c)$ with cost $\mathcal{O}(d^2)$, which is followed by a dot product. Repeating it for $C$ classes renders RGDA orders of magnitude slower than LDA.
\end{itemize}

\noindent \subsubsection{Complexities of LR-RGDA}
Our proposed LR-RGDA bridges the gap between the speed of LDA and the expressivity of RGDA. By decomposing the covariance into a shared base $\mathbf{B}$ and a rank-$r$ class-specific update, and omitting the class-independent shared term which can be dropped during the $\arg\max$ operation, we derive the efficient discriminant function as
\begin{equation}
	g_c^{\mathrm{LR\text{-}RGDA}}(\mathbf{x}) = \underbrace{\mathbf{w}_c^\top \mathbf{x} + b_c}_{\text{Affine term}} + \underbrace{\frac{1}{2}\mathbf{u}_c^\top \mathbf{M}_c^{-1} \mathbf{u}_c}_{\text{Correction}},
\end{equation}
where $\mathbf{u}_c = \mathbf{P}_c^{\mathrm{proj}}(\mathbf{x}-\boldsymbol{\mu}_c)$ projects the centered features $(\mathbf{x}-\boldsymbol{\mu}_c)$ into an efficient $r$-dimensional subspace.

\begin{itemize}
	\item \textbf{Storage}: $\mathcal{O}(d^2 + Cdr)$. 
	We reduce the memory usage of RGDA by replacing the $C$ full dense matrices of RGDA with compact low-rank factors. Specifically, we store the single shared base inverse $\mathbf{B}^{-1}$ (size $d \times d$) and, for each class, a projection matrix $\mathbf{P}_c^{\mathrm{proj}}$ (size $r \times d$) and a small core matrix $\mathbf{M}_c^{-1}$ (size $r \times r$). Since the rank $r \ll d$, the class-specific storage scaling as $\mathcal{O}(Cdr)$ is minimal compared to $\mathcal{O}(Cd^2)$.
	
	\item \textbf{Construction}: $\mathcal{O}(d^3 + Cd^2r)$. 
	 The construction process involves one global inversion of $\mathbf{B}$ costing $\mathcal{O}(d^3)$, which is then followed by SVD for each class scaling as $\mathcal{O}(d^2r)$. Crucially, we optimize the computation of the bias term $b_c$ by applying the \textit{Matrix Determinant Lemma}, as
	\begin{equation}
		\log\det(\boldsymbol{\Sigma}_c^{\mathrm{reg}}) = \log\det(\mathbf{M}_c) + \log\det(\mathbf{B}).
	\end{equation}
	It allows us to compute the determinant on the small $r \times r$ core matrix $\mathbf{M}_c$ with negligible cost $\mathcal{O}(r^3)$, rather than the expensive $\mathcal{O}(d^3)$ required for the full matrix.
	
	\item \textbf{Inference}: $\mathcal{O}(d^2 + Cdr)$. 
	The inference is accelerated by discarding the computationally heavy shared quadratic term $-\frac{1}{2}\mathbf{x}^\top \mathbf{B}^{-1} \mathbf{x}$, as it is common to all classes and redundant for the decision rule. The computational cost is thus dominated by the class-specific evaluations: an affine score costing $\mathcal{O}(d)$ and a low-rank quadratic correction. The correction involves projecting $\mathbf{x}$ to $\mathbb{R}^r$ with a cost of $\mathcal{O}(dr)$ and evaluating the form in the subspace requiring $\mathcal{O}(r^2)$. Summing them over $C$ classes yields a highly scalable complexity of $\mathcal{O}(Cdr)$.
\end{itemize}

\subsection{Empirical computational comparison of different classifiers}
\label{sub:comparison}
\begin{figure*}[htbp]
	\centering
	\small
	\begin{subfigure}{0.55\textwidth}
		\centering
		\includegraphics[width=\linewidth]{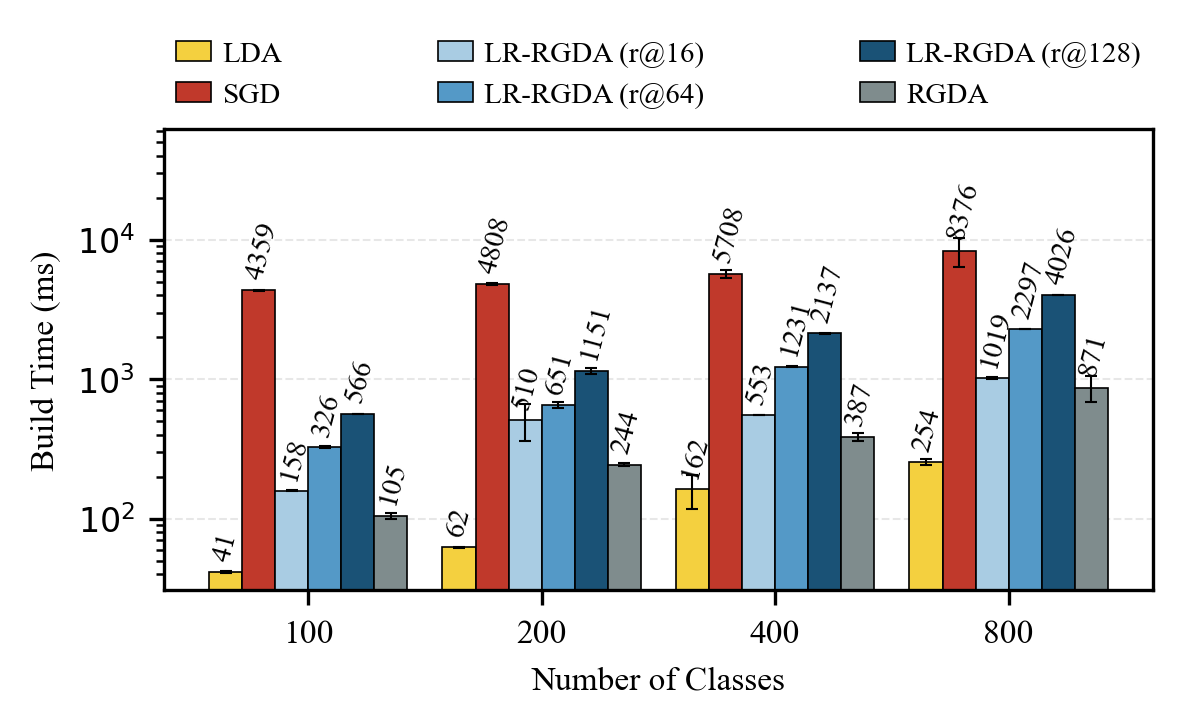}
		\caption{Construction time from scratch of different classifiers (ms)}
		\label{fig:build}
	\end{subfigure}
	\hfill
	\begin{subfigure}{0.55\textwidth}
		\centering
		\includegraphics[width=\linewidth]{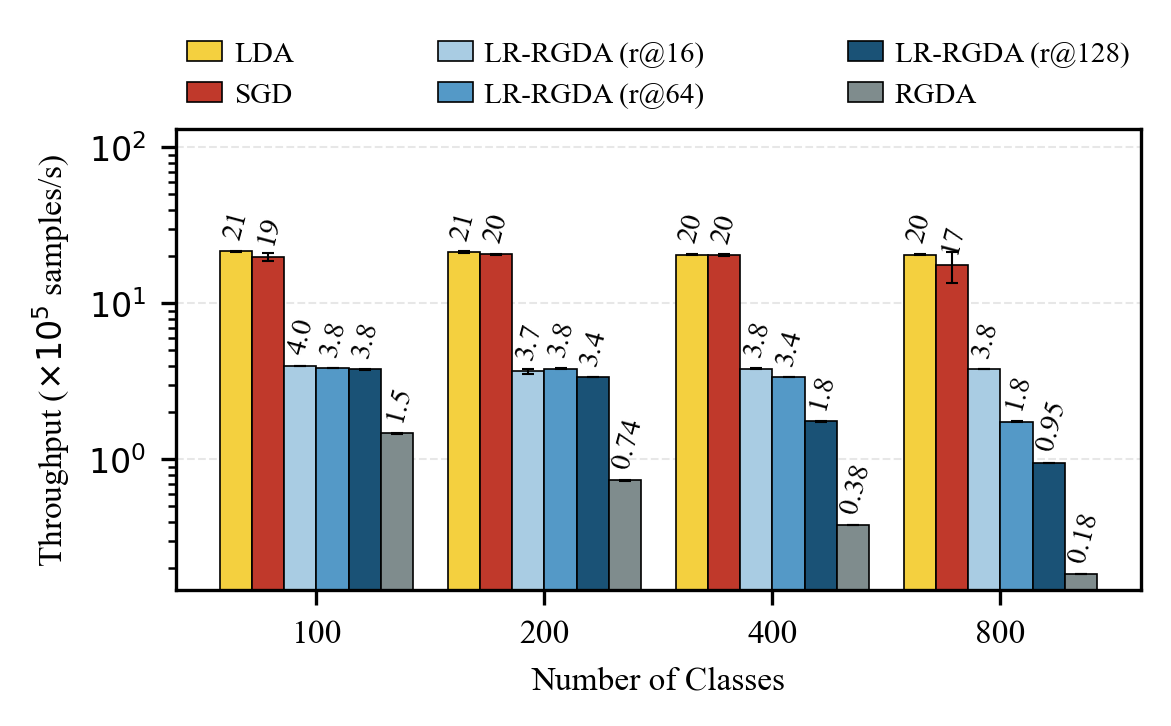}
		\caption{Classifier-only  throughput}
		\label{fig:throughput}
	\end{subfigure}
	\hfill
	\begin{subfigure}{0.55\textwidth}
		\centering
		\includegraphics[width=\linewidth]{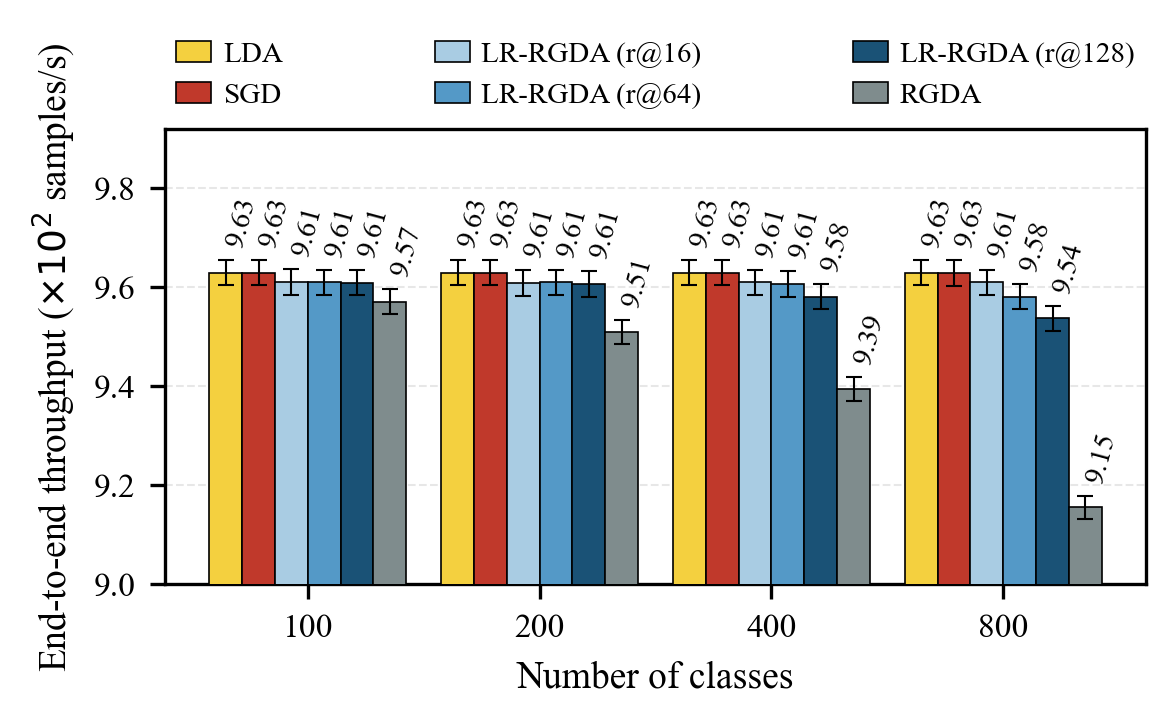}
		\caption{End-to-end throughput}
		\label{fig:end2end}
	\end{subfigure}
	
	\caption{\textbf{Empirical computational comparison of different classifiers.} 
		(a) Comparison of total construction time (log scale), which reveals that while SGD (red) incurs high absolute latency due to iteration, LR-RGDA supports efficient incremental updates since its modular structure allows historical SVD components to be cached. 
		(b) Classifier-only inference throughput (log scale), where the full-rank RGDA (gray) experiences a significant drop due to quadratic scaling, whereas LR-RGDA (blue) maintains near-linear speeds that are comparable to LDA and SGD. 
		(c) End-to-end throughput including the ViT backbone, which demonstrates that LR-RGDA eliminates the computational bottleneck inherent in standard quadratic classifiers, thereby matching the practical speeds of linear baselines when subspace rank is small (e.g., 16 and 64).}
	\label{fig:computation_comparison}
\end{figure*}

Here, we evaluate the computational efficiency of the proposed LR-RGDA against linear baselines LDA and SGD-based classifiers and the full-rank RGDA across three dimensions, i.e., the construction time, classifier inference throughput, and end-to-end system latency.

\paragraph{Construction complexity.} 
As illustrated in Figure~\ref{fig:build}, which reports the total time required to construct classifier parameters from scratch(i.e., given all Gaussian distributions), analytic methods (e.g., RGDA and LR-RGDA) exhibit an increase in latency as the number of classes grows, which is attributed to the accumulation of per-class matrix operations. 
Although SGD-based classifiers (red bars) appear to maintain a flatter growth rate, they incur a significantly higher absolute latency due to the necessity of iterative optimization.
Crucially, we highlight that this ``from-scratch'' comparison does not fully reflect the efficiency of LR-RGDA in CIL contexts where historical statistics are strictly modularized.
Unlike SGD-based methods that often require re-optimization over the buffered data to adjust global decision boundaries, LR-RGDA enables a streaming update mechanism where the SVD results of historical classes are cached and reused.
Consequently, the actual computational cost during an incremental step is determined solely by the decomposition of the newly added classes, which corresponds to the negligible latency observed in the few-class regime (the far-left bars), making LR-RGDA orders of magnitude faster than SGD in practical streaming scenarios.

\paragraph{Classifier inference throughput.}
Figure~\ref{fig:throughput} reports the throughput of the classifier module in isolation. 
It is evident that as the class count increases, the $\mathcal{O}(Cd^2)$ complexity of the full-rank RGDA causes a precipitous drop in throughput.
In contrast, LR-RGDA, which employs a low-rank factorization to reduce the quadratic overhead, successfully bridges this gap.
Particularly at lower ranks ($r=16$), LR-RGDA maintains a throughput that is comparable to linear classifiers (LDA), thereby ensuring scalability for large-scale datasets.

\textbf{End-to-end inference throughput.}
Figure~\ref{fig:end2end} illustrates the practical throughput when the ViT backbone is included. 
Although the heavy backbone computation dominates the total system latency, the results demonstrate that the computational overhead of the standard RGDA classifier creates a noticeable bottleneck that drags down overall throughput.
LR-RGDA effectively eliminates this bottleneck, which ensures that the system throughput is limited only by the feature extractor. It also implies that the LR-RGDA can achieve the end-to-end speeds that are on par with SGD and LDA classifiers.

\newpage
\section{Theoretical Analysis of HopDC}
\label{app:hopdc_theory}

This section provides the theoretical foundation for the proposed HopDC. We first establish that the drift retrieval process is mathematically equivalent to performing one step of energy minimization in a MCHN \cite{DBLP:conf/iclr/RamsauerSLSWGHA21}. Subsequently, we derive a theoretical error bound for the drift estimation, demonstrating that the estimation error is constrained by the local Lipschitz properties of the semantic drift and the attention-weighted distance to the anchors.

\subsection{Energy Landscape and Retrieval Dynamics}
\label{subsec:energy_landscape}
Following \citet{DBLP:conf/iclr/RamsauerSLSWGHA21}, let $\mathbf{F}^{\text{old}} = [\mathbf{k}_1, \dots, \mathbf{k}_N]^\top \in \mathbb{R}^{N \times d}$ be the stored patterns (Keys) representing the historical anchor features. For a query state $\mathbf{z} \in \mathbb{R}^d$ (a pseudo-feature from an old class), the energy function $E(\mathbf{z})$ is defined as
\begin{equation}
	E(\mathbf{z}) = -\frac{1}{\beta}\, \text{lse}\bigl(\beta, \mathbf{K}\mathbf{z}\bigr) 
	+ \frac{1}{2}\,\mathbf{z}^\top\mathbf{z} 
	+ \frac{1}{2}\max_i \|\mathbf{k}_i\|^2,
	\label{eq:energy_function}
\end{equation}
where $\beta = 1/\tau > 0$ is the inverse temperature parameter, $C = \frac{1}{2}\max_i \|\mathbf{k}_i\|^2$ is a constant term ensuring boundedness, and $\text{lse}(\beta, \mathbf{x}) = \log\left(\sum_{i=1}^N \exp(\beta x_i)\right)$ is the Log-Sum-Exp function. Crucially, the Log-Sum-Exp function is convex. The minima of this energy function correspond to the fixed points that represent the stored memories (or their metastable states).

\textbf{Retrieval as energy minimization.}
The retrieval process consists of updating the state $\mathbf{z}$ in the direction that decreases $E(\mathbf{z})$. Computing the gradient of $E$ with respect to $\mathbf{z}$ yields
\begin{align}
	\nabla_{\mathbf{z}} E(\mathbf{z})
	&= -\frac{1}{\beta}\, \nabla_{\mathbf{z}} \text{lse}\bigl(\beta, \mathbf{K}\mathbf{z}\bigr) + \mathbf{z} \nonumber = - \sum_{i=1}^N 
	\frac{\exp(\beta \mathbf{k}_i^\top\mathbf{z})}
	{\sum_{j=1}^N \exp(\beta \mathbf{k}_j^\top\mathbf{z})} \,
	\mathbf{k}_i \;+\; \mathbf{z} \nonumber \\
	&= -\, \text{Softmax}\left(\beta \mathbf{K}\mathbf{z}\right)^{\!\top} \mathbf{K} \;+\; \mathbf{z}.
\end{align}
A gradient‑descent step with a specific step size therefore gives the update
\begin{equation}
	\mathbf{z}^{\text{new}} = \text{Softmax}(\beta \mathbf{K}\mathbf{z})^{\!\top} \mathbf{K}.
	\label{eq:mchn_update}
\end{equation}
Equation~\eqref{eq:mchn_update} shows that a single step of energy minimization is exactly the attention mechanism that computes a convex combination of the stored patterns, with weights determined by the softmax output of the dot products. Crucially, for MCHNs, this update often converges to a fixed point in one step \cite{DBLP:conf/iclr/RamsauerSLSWGHA21}, which justifies the use of a single forward pass in HopDC.
 
In HopDC we employ a hetero‑associative variant of this dynamics.  While the attention weights are obtained from the feature manifold (the Keys $\mathbf{K} = \mathbf{F}^{\text{old}}$), we use them to retrieve the associated drift vectors (the Values $\mathbf{D}$) rather than the features themselves. Hence the estimated drift for a set of pseudo‑features $\mathbf{Q}$ becomes
\begin{equation}
	\boldsymbol{\Delta}^{\text{est}} = 
	\text{Softmax}\!\left(\frac{1}{\tau}\,\mathbf{Q} \mathbf{K}^{\!\top}\right) \mathbf{D},
\end{equation}
which is precisely the one‑step energy‑minimizing recall that interpolates the drift features of the anchor points.

\subsection{Theoretical Error Bound of Drift Estimation}
\label{subsec:error_bound}

To quantify the reliability of the drift estimation, we derive an upper bound on the approximation error.  The bound depends on two factors: the local smoothness of the true drift function and the distances between the query points and the anchor points, weighted by the attention distribution.

\begin{assumption}[Local Lipschitz continuity of the drift function]
	\label{assump:lipschitz}
	Let $\delta: \mathbb{R}^d \to \mathbb{R}^d$ be the true drift function of the ViT mapping between successive tasks, i.e., $\delta(\mathbf{k}_i) = \mathbf{d}_i$ for each anchor $\mathbf{k}_i$.  
	We assume that $\delta$ is locally Lipschitz continuous, namely, for any query $\mathbf{z}$ and any anchor $\mathbf{k}_i$, there exists a finite constant $L_{\mathbf{z},i} \ge 0$ such that
	\begin{equation}
		\bigl\|\delta(\mathbf{k}_i) - \delta(\mathbf{z})\bigr\|
		\;\le\; L_{\mathbf{z},i}\, \|\mathbf{k}_i - \mathbf{z}\|.
		\label{eq:local_lipschitz}
	\end{equation}
	The constant $L_{\mathbf{z},i}$ may depend on the pair $(\mathbf{z},\mathbf{k}_i)$. Intuitively, as $\mathbf{k}_i \to \mathbf{z}$, the drift discrepancy vanishes.  
	For the purpose of deriving a uniform bound we define $L = \sup_{i} L_{\mathbf{z},i}$ over the anchors that receive non‑negligible attention.
\end{assumption}

\begin{proposition}[General error bound of HopDC's drift estimation]
	\label{prop:general_bound}
	Let $\mathbf{z}$ be a query sample. The HopDC's drift estimation is given by $\hat{\boldsymbol{\delta}}(\mathbf{z}) = \sum_{i=1}^N p_i \delta(\mathbf{k}_i)$, where $p_i = \frac{\exp\left(\beta \mathbf{k}_i^\top\mathbf{z}\right)}
	{\sum_j \exp\left(\beta \mathbf{k}_j^\top\mathbf{z}\right)}$
	are the attention weights produced by the MCHN dynamics. The estimation error of HopDC is bounded by
	\begin{equation}
		\bigl\|\hat{\boldsymbol{\delta}}(\mathbf{z}) - \delta(\mathbf{z})\bigr\|
		\;\le\; L \sum_{i=1}^{N} p_i \|\mathbf{k}_i - \mathbf{z}\|.
		\label{eq:general_bound_}
	\end{equation}
\end{proposition}

\begin{proof}
	Since $\sum_i p_i = 1$, we have
	\begin{align*}
		\hat{\boldsymbol{\delta}}(\mathbf{z}) - \delta(\mathbf{z})
		= \sum_{i=1}^{N} p_i \delta(\mathbf{k}_i) - \delta(\mathbf{z})
		= \sum_{i=1}^{N} p_i \left( \delta(\mathbf{k}_i) - \delta(\mathbf{z}) \right).
	\end{align*}
	Taking norms and applying the Jensen's inequality gives
	\begin{align}
		\bigl\|\hat{\boldsymbol{\delta}}(\mathbf{z}) - \delta(\mathbf{z})\bigr\|
		\le \sum_{i=1}^{N} p_i \bigl\|\delta(\mathbf{k}_i) - \delta(\mathbf{z})\bigr\|.
		\label{eq:jeson}
	\end{align}
	Substituting the Lipschitz condition~\eqref{eq:local_lipschitz} into \eqref{eq:jeson} yields the bound~\eqref{eq:general_bound}.
\end{proof}

The bound~\eqref{eq:general_bound_} has a clear interpretation: the error is controlled by the expected distance between the query and the anchors, where the expectation is taken with respect to the attention distribution $\{p_i\}$.  Consequently, the estimation becomes accurate when the attention mechanism assigns large weights to anchors that are close to $\mathbf{z}$.

\subsection{Temperature-Dependent Guarantees and Practical Implications}
\label{subsec:temperature_bound}

The general bound in Proposition \ref{prop:general_bound} can be further refined to yield explicit, computable guarantees that highlight the roles of the temperature $\tau$ and the anchor set size $N$.

\begin{corollary}[Nearest‑neighbour limit]
	\label{cor:nn_limit}
	When the attention distribution becomes one‑hot on the nearest anchor \(i^* = \arg\min_i \|\mathbf{k}_i - \mathbf{z}\|\) (e.g., as $\tau \to 0$), we recover a nearest‑neighbour guarantee:
	\begin{equation}
		\bigl\|\hat{\boldsymbol{\delta}}(\mathbf{z}) - \delta(\mathbf{z})\bigr\| \le L \min_i \|\mathbf{k}_i - \mathbf{z}\|.
		\label{eq:nn_bound}
	\end{equation}
\end{corollary}

In practice, the softmax weights are distributed, but we can still obtain a bound that explicitly involves the temperature $\tau$ and the number of anchors $N$.

\begin{corollary}[Temperature‑dependent bound]
	\label{cor:temp_bound}
	Assume the anchor keys $\mathbf{k}_i$ and the query $\mathbf{z}$ are $\ell_2$-normalised (as done in HopDC). Then
	$$
	\bigl\|\hat{\boldsymbol{\delta}}(\mathbf{z}) - \delta(\mathbf{z})\bigr\| \le L \left( \min_i \|\mathbf{k}_i - \mathbf{z}\| + \sqrt{2\tau\log N} \right).
	$$
\end{corollary}

\begin{proof}
	Because $\|\mathbf{k}_i\| = \|\mathbf{z}\| = 1$, we have $\mathbf{k}_i^\top \mathbf{z} = 1 - \frac{1}{2}\|\mathbf{k}_i - \mathbf{z}\|^2$. Hence
	$$
	p_i = \frac{\exp\!\bigl( \frac{1}{\tau} (1 - \frac{1}{2}\|\mathbf{k}_i - \mathbf{z}\|^2) \bigr)}{\sum_j \exp\!\bigl( \frac{1}{\tau} (1 - \frac{1}{2}\|\mathbf{k}_j - \mathbf{z}\|^2) \bigr)}
	= \frac{\exp\!\bigl( -\frac{\|\mathbf{k}_i - \mathbf{z}\|^2}{2\tau} \bigr)}{\sum_j \exp\!\bigl( -\frac{\|\mathbf{k}_j - \mathbf{z}\|^2}{2\tau} \bigr)}.
	$$
	Define $a_i = \frac{\|\mathbf{k}_i - \mathbf{z}\|^2}{2\tau}$. Then $ p_i = e^{-a_i} / \sum_j e^{-a_j} $. From the Gibbs variational principle \cite{boyd2004convex},
	$$
	\sum_i p_i a_i = \frac{\sum_i a_i e^{-a_i}}{\sum_j e^{-a_j}} \le -\log\!\Bigl( \sum_i e^{-a_i} \Bigr) \le \min_i a_i + \log N.
	$$
	Multiplying by $2\tau$ gives
	$$
	\sum_i p_i \|\mathbf{k}_i - \mathbf{z}\|^2 \le \min_i \|\mathbf{k}_i - \mathbf{z}\|^2 + 2\tau \log N.
	$$
	By Jensen’s inequality, we have $\bigl( \sum_i p_i \|\mathbf{k}_i - \mathbf{z}\| \bigr)^2 \le \sum_i p_i \|\mathbf{k}_i - \mathbf{z}\|^2$. Therefore,
	$$
	\sum_i p_i \|\mathbf{k}_i - \mathbf{z}\| \le \sqrt{ \min_i \|\mathbf{k}_i - \mathbf{z}\|^2 + 2\tau \log N }
	\le \min_i \|\mathbf{k}_i - \mathbf{z}\| + \sqrt{2\tau \log N}.
	$$
Substituting this into Proposition \ref{prop:general_bound} yields the claimed bound.
\end{proof}

\begin{remark}[Coverage and scaling]
	\label{rem:coverage}
	If the anchors form an $\epsilon$-cover of the query distribution, then $\min_i \|\mathbf{k}_i - \mathbf{z}\| \le \epsilon$. The covering number typically grows as $N \sim (1/\epsilon)^d$ where $d$ is the intrinsic dimension of the feature manifold, so $\epsilon = O(N^{-1/d})$. Hence the total error bound is $O(N^{-1/d} + \sqrt{\tau \log N})$. For large $N$, the first term dominates, showing that increasing the anchor set improves accuracy.
\end{remark}

\begin{remark}[Bias–variance trade‑off]
	\label{rem:bias_variance}
	The temperature $\tau$ controls the softness of the attention. A smaller $\tau$ reduces the $\sqrt{2\tau \log N}$ term but makes the distribution sharper, which may increase variance if individual anchor drifts are noisy. A larger $\tau$ smooths the estimate, reducing variance at the cost of a larger bias term. The optimal $\tau$ balances these effects; in our experiments we found $\tau = 0.05$ to work well across diverse settings.
\end{remark}

\begin{remark}[Effect of top‑$k$ sparsification]
	\label{rem:topk}
	The top‑$k$ operation restricts the sum to the $k$ anchors with largest $p_i$. In that case, the same proof applies with $N$ replaced by $k$, and the minimum distance is taken over the selected set (which is at most the global minimum plus a small margin). Consequently, the bound becomes $L \bigl( \min_{i \in \text{top-}k} \|\mathbf{k}_i - \mathbf{z}\| + \sqrt{2\tau \log k} \bigr)$, which is even tighter due to the reduced $\log k$ term.
\end{remark}

\subsection{Discussion}

The analysis above provides a rigorous foundation for HopDC.  By viewing the drift retrieval as one step of energy minimization in an MCHN, we connect the practical attention‑based formula to a well‑studied dynamical system with provable convergence properties.  The derived error bounds (Proposition \ref{prop:general_bound}, Corollaries \ref{cor:nn_limit} and \ref{cor:temp_bound}) highlight how the accuracy depends on the interplay between the local smoothness of the drift, the geometry and density of the anchor points, and the hyperparameters ($\tau$, $k$).

Importantly, Corollary \ref{cor:temp_bound} and the subsequent remarks provide explicit, actionable guarantees: the estimation error can be made arbitrarily small by using a sufficiently dense set of anchors ($N \to \infty$) and an appropriately chosen temperature $\tau$. The bound also justifies the use of a sparse top‑$k$ attention in practice: restricting the sum not only accelerates computation but also tightens the theoretical guarantee by reducing the $\log N$ term to $\log k$ (Remark \ref{rem:topk}).

In summary, HopDC leverages the associative memory dynamics of MCHNs to construct a non‑parametric, training‑free estimator of representation drift.  Its theoretical guarantees, together with its empirical effectiveness shown in the main text, make it a reliable component for mitigating distribution shift in class‑incremental learning.

\newpage
\section{Algorithms and Implementation Details}
\label{app:implementation}

This section outlines the implementation details. We first provide the pseudocodes for the proposed framework, which is decomposed into the overall CIL pipeline, the drift compensation mechanism (HopDC), and the analytic construction and inference phases of LR-RGDA. Subsequently, we detail the specific experimental protocols, including dataset preprocessing, ViT configurations, and the hyperparameter settings used for both Stage-1 backbone optimization (e.g., SeqFT, LoRA) and Stage-2 classifier reconstruction.

\subsection{Pseudocode of Proposed Algorithms}
\label{app:algorithms}

The proposed framework is summarized in four algorithms. Algorithm~\ref{alg:main_pipeline} outlines the complete CIL pipeline, which consists of backbone optimization, distribution compensation, and classifier construction. Algorithm~\ref{alg:hopdc} details the HopDC for recalibrating historical statistics. Algorithm~\ref{alg:lr_rgda_construction} describes the construction phase of LR-RGDA. Specifically, we focus on the low-rank factorization and the pre-computation of projection matrices. Finally, Algorithm~\ref{alg:lr_rgda_inference} presents the efficient inference mechanism of LR-RGDA that leverages these pre-computed parameters.

\begin{algorithm}[htbp]
	\caption{The Overall CIL Framework with LR-RGDA \& HopDC}
	\label{alg:main_pipeline}
	\begin{algorithmic}[1]
		\STATE {\bfseries Input:} Datasets $\{\mathcal{D}_t\}_{t=1}^T$, unlabeled anchor set $\mathcal{A}$, Pre-trained ViT $\mathcal{F}_{\boldsymbol{\theta}}$
		\STATE {\bfseries Output:} Classifier parameters $\mathcal{P}_{\text{all}}$ for all observed classes
		\STATE Initialize class-wise feature statistics $\mathcal{H}_0 \leftarrow \emptyset$, Anchor features $\mathbf{F}^{\text{old}} \leftarrow \mathcal{F}_{\boldsymbol{\theta}}(\mathcal{A})$
		
		\FOR{task $t=1$ {\bfseries to} $T$}
		\STATE \textcolor{gray}{\itshape // Stage 1: Backbone Optimization}
		\STATE Update $\boldsymbol{\theta}$ on $\mathcal{D}_t$ (e.g., via SeqFT, LoRA-SeqFT, etc)
		
		\STATE \textcolor{gray}{\itshape // Stage 2: Distribution Compensation}
		\IF{$t > 1$}
		\STATE $\mathcal{H}_{t-1} \leftarrow$ \textbf{HopDC}$(\mathcal{H}_{t-1}, \mathcal{A}, \mathbf{F}^{\text{old}}, \mathcal{F}_{\boldsymbol{\theta}})$ (see Algorithm~\ref{alg:hopdc})
		\ENDIF
		
		\STATE \textcolor{gray}{\itshape // Update Statistics for New Classes}
		\STATE Compute new statistics $\mathcal{S}_{\text{new}} \leftarrow \{ (\boldsymbol{\mu}_c, \boldsymbol{\Sigma}_c) \mid c \in \mathcal{Y}_t \}$ via updated $\mathcal{F}_{\boldsymbol{\theta}}$
		\STATE Update statistics $\mathcal{H}_t \leftarrow \mathcal{H}_{t-1} \cup \mathcal{S}_{\text{new}}$
		\STATE Update anchor history $\mathbf{F}^{\text{old}} \leftarrow \mathcal{F}_{\boldsymbol{\theta}}(\mathcal{A})$
		
		\STATE \textcolor{gray}{\itshape // Analytic Classifier Construction}
		\STATE $\mathcal{P}_{\text{all}} \leftarrow$ \textbf{LR-RGDA-Construction}$(\mathcal{H}_t)$ (see Algorithm~\ref{alg:lr_rgda_construction})
		\ENDFOR
		\STATE {\bfseries Return} $\mathcal{P}_{\text{all}}$
	\end{algorithmic}
\end{algorithm}

\begin{algorithm}[tb]
	\caption{Hopfield-based Distribution Compensator (HopDC)}
	\label{alg:hopdc}
	\begin{algorithmic}[1]
		\STATE {\bfseries Input:} Historical statistics set $\mathcal{H}_{t-1}$, Anchor set $\mathcal{A}$, Previous feature matrix $\mathbf{F}^{\text{old}}$, Current backbone $\mathcal{F}_{\boldsymbol{\theta}}$
		\STATE {\bfseries Hyperparameters:} Temperature $\tau$, Sparsity $k$, Sample count $M$
		\STATE {\bfseries Output:} Calibrated Statistics $\mathcal{H}_{t-1}^{\text{cal}}$
		
		\STATE Compute the new anchor feature matrix $\mathbf{F}^{\text{new}} \leftarrow \mathcal{F}_{\boldsymbol{\theta}}(\mathcal{\tilde{A}})$.
		\STATE Compute the drift matrix $\mathbf{D} \leftarrow \mathbf{F}^{\text{new}} - \mathbf{F}^{\text{old}}$
		\STATE Construct Keys: $\mathbf{K} \leftarrow l_2\text{-Normalize}(\mathbf{F}^{\text{old}})$
		
		\FOR{each old class $c \in \mathcal{C}_{1:t-1}$}
		\STATE Retrieve $\boldsymbol{\mu}_c, \boldsymbol{\Sigma}_c$ from $\mathcal{H}_{t-1}$
		\STATE Sample $M$ pseudo-features $\mathbf{Z}_c \sim \mathcal{N}(\boldsymbol{\mu}_c, \boldsymbol{\Sigma}_c)$
		\STATE Construct Queries: $\mathbf{Q}_c \leftarrow l_2\text{-Normalize}(\mathbf{Z}_c)$
		\STATE Estimate drift: $\boldsymbol{\Delta}_c^{\text{est}} \leftarrow \text{Softmax}_{\text{top-}k}(\frac{1}{\tau}\mathbf{Q}_c \mathbf{K}^\top) \mathbf{D}$
		\STATE Calibrate samples: $\mathbf{Z}_c^{\text{cal}} \leftarrow \mathbf{Z}_c + \boldsymbol{\Delta}_c^{\text{est}}$
		\STATE Re-estimate calibrated $\boldsymbol{\mu}_c, \boldsymbol{\Sigma}_c$ using $\mathbf{Z}_c^{\text{cal}}$
		\ENDFOR
		\STATE {\bfseries Return} Updated history $\mathcal{H}_{t-1}^{\text{cal}}$ using calibrated statistics
	\end{algorithmic}
\end{algorithm}

\begin{algorithm}[tb]
	\caption{LR-RGDA Construction (Pre-computation)}
	\label{alg:lr_rgda_construction}
	\begin{algorithmic}[1]
		\STATE {\bfseries Input:} Historical statistics $\mathcal{H}_t$
		\STATE {\bfseries Hyperparameters:} $\alpha_{1}, \alpha_{2}, \alpha_{3}$, Rank $r$
		\STATE {\bfseries Output:} Inference Parameters $\mathcal{P}_{\text{all}} = \{\mathbf{w}_c, b_c, \mathbf{P}_c, \mathbf{M}_c^{-1}\}_{c \in \mathcal{C}_t}$
		
		\STATE Compute the global average covariance $\boldsymbol{\Sigma}_{\text{avg}}$ from $\mathcal{H}_t$
		\STATE Compute the base inverse: $\mathbf{B}^{-1} = (\alpha_2 \boldsymbol{\Sigma}_{\text{avg}} + \alpha_3 \mathbf{I})^{-1}$
		
		\FOR{each class $c \in \mathcal{C}_{t}$}
		\STATE Decompose residual: $\mathbf{U}_c \mathbf{S}_c \mathbf{U}_c^\top \approx \alpha_1 \boldsymbol{\Sigma}_c$ via rank-$r$ SVD
		\STATE Compute the scaled eigenvector: $\widetilde{\mathbf{U}}_c \leftarrow \alpha_1^{1/2} \mathbf{U}_c \mathbf{S}_c^{1/2}$
		\STATE Compute the core inverse: $\mathbf{M}_c^{-1} \leftarrow (\mathbf{I}_r + \widetilde{\mathbf{U}}_c^\top \mathbf{B}^{-1} \widetilde{\mathbf{U}}_c)^{-1}$
		\STATE \textcolor{gray}{\itshape // Pre-compute projection matrix for fast inference }
		\STATE $\mathbf{P}_c \leftarrow \widetilde{\mathbf{U}}_c^\top \mathbf{B}^{-1}$
		\STATE Compute linear weights: $\mathbf{w}_c \leftarrow \mathbf{B}^{-1} \boldsymbol{\mu}_c$
		\STATE Compute the bias $b_c$ (using matrix determinant Lemma)
		\ENDFOR
		\STATE {\bfseries Return} $\{\mathbf{w}_c, b_c, \mathbf{P}_c, \mathbf{M}_c^{-1}\}_{c}$
	\end{algorithmic}
\end{algorithm}

\begin{algorithm}[tb]
	\caption{Efficient Inference Mechanism of LR-RGDA}
	\label{alg:lr_rgda_inference}
	\begin{algorithmic}[1]
		\STATE {\bfseries Input:} Test sample $\mathbf{x}$, Parameters $\mathcal{P}_{\text{all}}$
		\STATE Extract the feature $\mathbf{f} = \mathcal{F}_{\boldsymbol{\theta}}(\mathbf{x})$
		\STATE Initialize the score vector $\mathbf{s} \in \mathbb{R}^{|\mathcal{C}_t|}$
		
		\FOR{each class $c \in \mathcal{C}_t$}
		\STATE \textcolor{gray}{\itshape // 1. Global Linear Score ($\mathcal{O}(d)$)}
		\STATE $\mathcal{L}_c \leftarrow \mathbf{w}_c^\top \mathbf{f} + b_c$
		
		\STATE \textcolor{gray}{\itshape // 2. Low-Rank Quadratic Correction ($\mathcal{O}(rd)$)}
		\STATE \textcolor{gray}{\itshape // Use pre-computed $\mathbf{P}_c$ to avoid $\mathcal{O}(d^2)$ cost}
		\STATE Project the residual: $\mathbf{u}_c \leftarrow \mathbf{P}_c (\mathbf{f} - \boldsymbol{\mu}_c)$
		\STATE Compute the correction: $\mathcal{Q}_c \leftarrow \frac{1}{2} \mathbf{u}_c^\top \mathbf{M}_c^{-1} \mathbf{u}_c$
		
		\STATE \textcolor{gray}{\itshape // 3. Final Discriminant Score}
		\STATE $\mathbf{s}[c] \leftarrow \mathcal{L}_c + \mathcal{Q}_c$
		\ENDFOR
		\STATE {\bfseries Return:} Predictive class $\hat{y} = \arg\max_c \mathbf{s}[c]$
	\end{algorithmic}
\end{algorithm}

\subsection{Datasets and Preprocessing Protocols}
\label{app:datasets}
Our experiments are conducted on standardized CIL benchmarks using a unified preprocessing pipeline. Input images are consistently resized to $224 \times 224$ and normalized using the canonical ImageNet mean and standard deviation.

\textbf{Within-domain CIL benchmarks.} We employ four datasets with distinct granularity: CIFAR-100, CUB-200, Cars-196, and ImageNet-R. For CIFAR-100, the dataset is evenly partitioned into 10 disjoint tasks (10 classes per task). The remaining three datasets adopt a setup starting with 20 base classes followed by 20-class increments ($T=10$), with the exception of Cars-196, where the final task comprises the residual 16 classes.

\textbf{Cross-domain CIL benchmarks.} To stress-test the robustness of LR-RGDA against domain shifts, we construct a large-scale CIL benchmark spanning seven diverse datasets: CIFAR-100, ImageNet-R, Cars-196, CUB-200, Caltech-101, Oxford-Flower-102, and Food-101. In this setting, each dataset is bisected into two incremental splits. To maintain class balance across varying dataset sizes, we enforce a maximum cap of 128 training samples per class.

\subsection{Pre-trained Backbone Architectures}
\label{app:backbones}
We leverage the \texttt{timm} library to instantiate pre-trained models. Specifically, we consistently employ the ViT-B/16 (Base) architecture across all settings to eliminate architectural variations. Our evaluation encompasses three distinct representation strategies: for supervised pre-training, we utilize the standard \texttt{vit\_base\_patch16\_224} (pre-trained on ImageNet-21K and fine-tuned on ImageNet-1K); for self-supervised learning, we evaluate DINO \texttt{(vit\_base\_patch16\_224.dino)} and MoCoV3. For MoCoV3, we manually load the official weights into the standard architecture to strictly align with protocols established in SLCA and CoMA; and finally, for vision-language pre-training, we adopt the CLIP-based \texttt{vit\_base\_patch16\_clip\_224.openai} model to leverage text-aligned image representations for superior transferability. All training and inference processes are executed on NVIDIA RTX 4090 GPUs.

\subsection{Backbone optimization Strategies}
We utilize distinct optimization strategies for full fine-tuning methods and PEFT-based adaptation, which are detailed as follows.

\paragraph{SeqFT and its variants.} For methods involving the full backbone updates (e.g., SeqFT, SeqKD, NSP-SeqFT), we utilize SGD with a momentum of $0.9$. Based on the cross-validation on a held-out validation set from ViT/B-Sup21K and ImageNet-R, we select a backbone learning rate of $5 \times 10^{-6}$ from $\{1\times 10^{-6}, 5\times 10^{-6}, 1\times 10^{-5}\}$ and apply this learning rate across all datasets. The classifier learning rate is set to $5 \times 10^{-5}$, which is $10$ times larger than the backbone's learning rate. We use a cosine annealing scheduler with a 10\% warmup phase. The learning rate is decayed to one-third of its initial value by the end of each task. We restrict the fine-tuning to the two MLP modules within all Transformer layers. The training duration is 1,000 steps for method without distillation and NSP, and 1,500 steps for distillation, NSP or LoRA-based variants. We apply SCE loss with $\lambda_1=0.5, \lambda_2=0.5$.

\paragraph{LoRA Adaptation.} For PEFT approaches (LoRA-SeqFT, LoRA-SeqKD), we employ the AdamW optimizer with an adapter learning rate of $1 \times 10^{-4}$ and a classifier learning rate of $1 \times 10^{-3}$. The LoRA adaptation is formulated as $W' = W + BA$, where $A \in \mathbb{R}^{r_{\rm lora} \times d_{\text{in}}}$ is initialized via Kaiming uniform distribution where $a=\sqrt{5}$, and $B \in \mathbb{R}^{d_{\text{out}} \times r_{\rm lora}}$ is initialized to zero values. We set $r_{\rm lora}=4$. The LoRA adapters are applied to the QKV linear projections and MLP modules across all Transformer blocks.

\paragraph{Knowledge distillation.} For distillation baselines (SeqKD, NSP-SeqKD), we implement feature-level distillation on the \texttt{cls\_token} of the final transformer block. The regularization coefficient is $0.5$ for NSP-SeqKD and $1.0$ for other variants. 

\paragraph{Null space projection.} For NSP, gradients are projected onto the null space of the feature covariance matrix. We adopt a ``Hard Projection" strategy with soft relaxation ($\epsilon=0.05$) and identity interpolation ($\lambda=0.02$).

\subsection{Engineering construction of Classifiers}
\label{app:classifier_recon}

This subsection details the engineering realization of the classifier reconstruction phase. Given the updated historical statistics $\mathcal{H}_t = \{(\boldsymbol{\mu}_c, \boldsymbol{\Sigma}_c)\}_{c=1}^C$.

\paragraph{SGD-based classifier.}
Unlike analytic methods, SGD classifiers require iterative optimization. To ensure scalability as the number of classes $C$ grows, we implement a dynamic training schedule: the total steps are defined as $T = T_{\text{base}} + N_{\text{steps/cls}} \times C$, with $T_{\text{base}}=5000$ and $N_{\text{steps/cls}}=4$. Before optimization, we sample 256 pseudo-features per class from $\mathcal{N}(\boldsymbol{\mu}_c, \boldsymbol{\Sigma}_c)$. The classifier is optimized via the AdamW optimizer (initial lr $10^{-3}$, weight decay $10^{-4}$) coupled with a cosine annealing scheduler that decays the learning rate to $10^{-4}$. To prevent the overfitting, we incorporate an early stopping mechanism with a patience of 100 steps.

\paragraph{LDA Classifier.}
For LDA, we compute the global average covariance $\boldsymbol{\Sigma}_{\text{avg}}$ in a \textbf{streaming fashion} to minimize peak memory usage, avoiding the need to stack all $C$ covariance matrices simultaneously. To ensure the numerical stability, we apply spherical regularization as $\boldsymbol{\Sigma}_{\text{reg}} = (1-\gamma)\boldsymbol{\Sigma}_{\text{avg}} + \gamma\mathbf{I}$, where $\gamma=0.1$ is selected via the cross-validation from values $\{0.0, 0.1, 0.2, 0.3 \}$. The weights $\mathbf{W}$ and biases $\mathbf{b}$ are then derived analytically via closed-form solutions.

\paragraph{RGDA classifier.}
The RGDA classifier requires inverting class-specific covariance matrices. We compute the precision matrix $\boldsymbol{\Lambda}_c = (\boldsymbol{\Sigma}_c^{\text{reg}})^{-1}$ efficiently using \texttt{torch.cholesky\_inverse}, which offers the superior speed and numerical stability compared to standard inversion. To handle potential singularities in degenerate covariance matrices, we implement an automatic fallback mechanism that switches to SVD-based inversion if Cholesky decomposition fails. What's more, the inference is fully vectorized using \texttt{torch.einsum} to parallelize the Mahalanobis distance computation.

\paragraph{LR-RGDA classifier.}
To realize the theoretical complexity benefits of LR-RGDA on GPU hardware, we implement a batched initialization pipeline (typically with a batch size of 12 classes) that pre-computes the inference parameters without overloading GPU memory.
\begin{enumerate}
	\item \textbf{Randomized SVD:} Instead of full SVD, we employ \texttt{torch.svd\_lowrank} to decompose the class-specific covariance matrices to obtain their low-rank approximations. Compared to the standard SVD, it reduces the factorization cost from $\mathcal{O}(Cd^3)$ to $\mathcal{O}(Cd^2r)$ since the smallest singular values should be ignored. 
	\item \textbf{Woodbury Inversion:} We solve the global base inverse $\mathbf{B}^{-1}$ and the core inverse $\mathbf{M}_c^{-1}$ using the Woodbury matrix identity. Crucially, we pre-compute and store the exact affine weights $\mathbf{w}_c \in \mathbb{R}^d$ and the projection matrix $\mathbf{P}_c = \widetilde{\mathbf{U}}_c^\top \mathbf{B}^{-1} \in \mathbb{R}^{r \times d}$ during the construction phase.
	\item \textbf{Projection Pre-computation:} By freezing these pre-computed tensors, the inference phase avoids all heavy matrix inversions. The projection of a test sample $\mathbf{x}$ into the subspace becomes a simple matrix-vector multiplication of size $(r \times d) \times (d \times 1)$.
\end{enumerate}

\newpage
\section{Extended Experimental Analysis}
\label{sec:extended_experiments}

This section provides a more comprehensive empirical analysis to validate the robustness and generalization capabilities of LR-RGDA and HopDC. We specifically evaluate the impact of different backbone adaptation strategies, the sensitivity to the low-rank dimension $r$, the hyperparameter landscape, and the resilience of the HopDC mechanism against varying conditions.

\subsection{Empirical Validation of the Low-Rank Assumption}
\label{app:low_rank_validation}

The core design of LR-RGDA hinges on the assumption that class-specific covariance matrices in pre-trained ViT feature spaces exhibit a low-rank structure. To empirically validate this assumption, we analyze the spectral properties of class-specific covariance matrices computed from a frozen ViT/B-CLIP backbone across the seven diverse datasets comprising our cross-domain CIL benchmark (totaling 1001 classes).

\begin{figure}[htbp]
	\centering
	\includegraphics[width=0.65\linewidth]{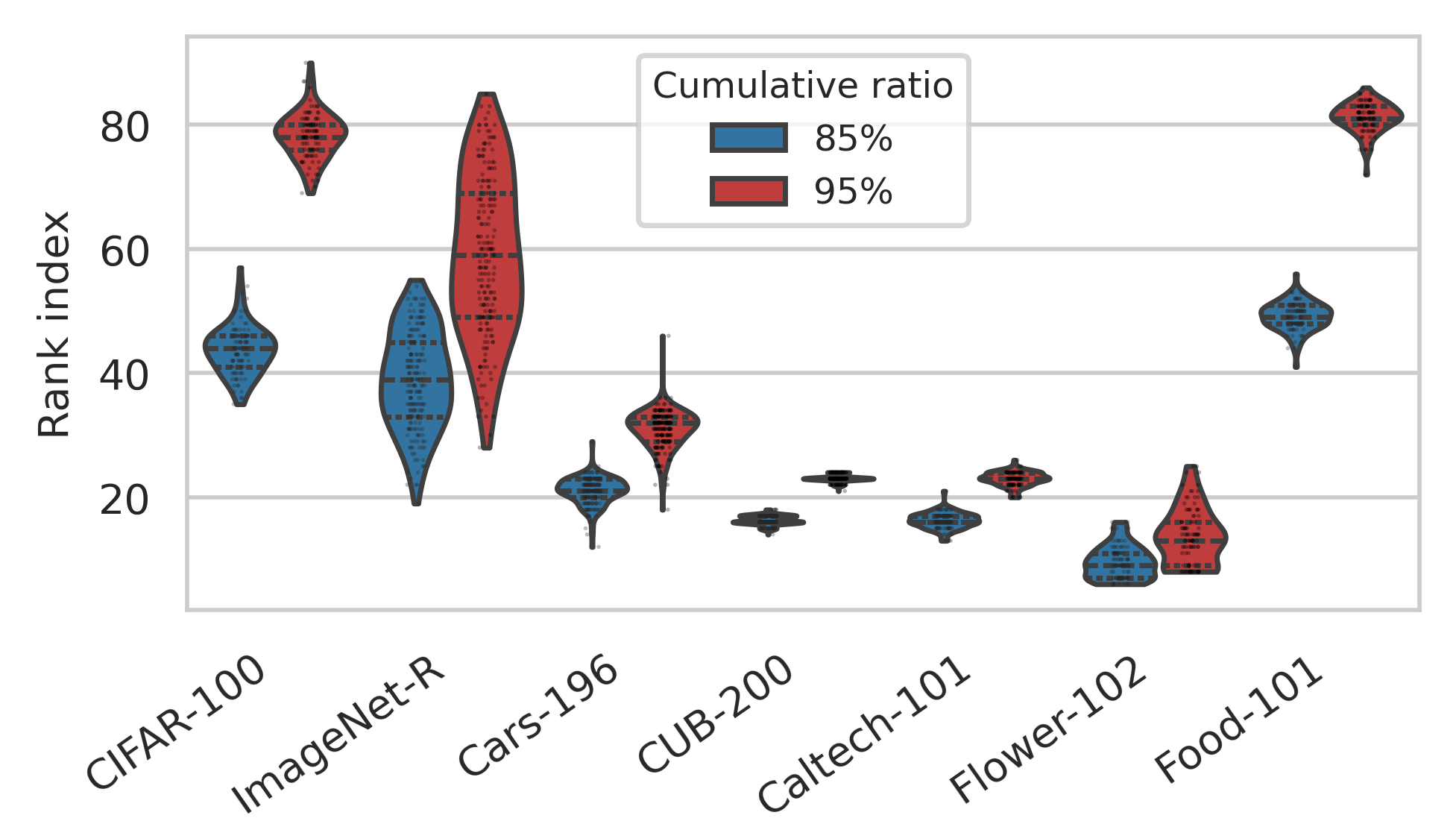}
	\caption{Statistical distributions of the minimum number of largest singular values needed to capture 85\% and 95\% of the total variance for class-specific covariance matrices on seven datasets (1001 classes). The ViT/B-CLIP is used as the feature extractor. Crucially, the top-60 singular values are consistently sufficient to capture the dominant variance ($>85\%$) across diverse datasets, validating the low-rank assumption of LR-RGDA.}
	\label{fig:vit-b-p16-clip95percentrankboxplot}
\end{figure}

As shown in Figure~\ref{fig:vit-b-p16-clip95percentrankboxplot}, we plot the distribution (via box plots) of the minimum rank required to explain 85\% and 95\% of the total variance (i.e., the cumulative explained variance ratio) for each class-specific covariance matrix. The key observation is that the median rank required to capture 85\% of the variance lies consistently below 60 across all datasets, despite the feature dimension $d=768$. Even for the more stringent 95\% threshold, the median rank rarely exceeds 120. This confirms that the discriminative information is concentrated within a low-dimensional subspace of the original feature space. This empirical evidence strongly supports the low-rank factorization strategy employed by LR-RGDA, justifying the use of a small subspace rank $r \ll d$ (e.g., $r=64$) to achieve an efficient yet expressive approximation of the full covariance structure.

\subsection{Results on ViT/B-MoCoV3}
\label{app:results_mocov3}

\begin{table*}[t]
	\centering
	\caption{CIL performance evaluations (\%) on four within-domain datasets using ViT/B-MoCoV3 (averaged over three random seeds).}
	\setlength{\tabcolsep}{5pt}
	\renewcommand{\arraystretch}{0.8}
	\small
	\scalebox{1.0}{
		\begin{tabular}{@{}lccccccccccc@{}}
			\toprule
			\multirow{2}{*}{Method}
			& \multicolumn{2}{c}{CUB-200}
			& \multicolumn{2}{c}{Cars-196}
			& \multicolumn{2}{c}{CIFAR-100}
			& \multicolumn{2}{c}{ImageNet-R}
			& \multicolumn{2}{c}{Four datasets} \\
			\cmidrule(lr){2-3} \cmidrule(lr){4-5} \cmidrule(lr){6-7} \cmidrule(lr){8-9} \cmidrule(lr){10-11}
			& Last & Inc & Last & Inc & Last & Inc & Last & Inc & Last & Inc \\
			\midrule
			\textbf{Empirical upper bounds} \\
			\quad Joint-Training & 82.36 & --- & 82.13 & --- & 88.52 & --- & 75.78 & --- & 82.20 & --- \\
			\midrule
			\textbf{Existing baselines} \\
			\quad BiC & 74.39 & 82.13 & 65.57 & 73.95 & 80.57 & 89.39 & 57.36 & 68.07 & 69.47 & 78.39 \\
			\quad RanPAC & 74.43 & 83.63 & 63.21 & 74.01 & 86.47 & 90.81 & 69.11 & 75.20 & 73.31 & 80.91 \\
			\quad SLCA & 73.01 & 82.13 & 66.04 & 72.59 & 85.27 & 89.51 & 68.07 & 73.04 & 73.10 & 79.32 \\
			\quad SLCA++ & 75.48 & 82.94 & 69.71 & 75.67 & 84.77 & 89.53 & 69.01 & 74.75 & 74.74 & 80.72 \\
			\quad CoMA & 75.12 & 82.76 & 67.48 & 74.90 & 86.59 & 91.02 & 69.33 & 75.64 & 74.63 & 81.08 \\
			\midrule
			\textbf{Proposed methods} \\
			\quad RanProj + \textit{LR-RGDA} & 70.33 & 82.94 & 63.15 & 69.10 & \textbf{87.89} & 91.27 & 66.27 & 70.77 & 71.91 & 78.52 \\
			\cmidrule(lr){2-11}
			
			\quad SeqFT + \textit{LR-RGDA} & 59.91 & 76.15 & 46.05 & 56.08 & 78.36 & 85.84 & 66.10 & 73.74 & 62.61 & 72.95 \\
			\quad \quad + \textit{HopDC} & 79.88 & 86.85 & 78.87 & 83.45 & 82.90 & 89.91 & 72.33 & 78.03 & 78.50 \textcolor{red}{\tiny +15.89} & 84.56 \textcolor{red}{\tiny +11.61} \\
			\cmidrule(lr){2-11}
			
			\quad SeqKD + \textit{LR-RGDA} & 69.92 & 83.34 & 61.76 & 74.41 & 84.39 & 90.19 & 71.15 & 77.83 & 71.81 & 81.44 \\
			\quad \quad + \textit{HopDC} & 81.50 & 87.48 & 82.07 & 85.34 & 86.55 & 91.05 & \textbf{74.77} & 79.36 & 81.22 \textcolor{red}{\tiny +9.41} & 85.81 \textcolor{red}{\tiny +4.37} \\
			\cmidrule(lr){2-11}
			
			\quad NSP-SeqFT + \textit{LR-RGDA} & 76.42 & 85.22 & 66.94 & 78.20 & 85.65 & 90.16 & 70.60 & 76.52 & 74.90 & 82.53 \\
			\quad \quad + \textit{HopDC} & 79.84 & 86.43 & 79.92 & 83.34 & 86.25 & 90.64 & 71.45 & 76.90 & 79.37 \textcolor{red}{\tiny +4.47} & 84.33 \textcolor{red}{\tiny +1.80} \\
			\cmidrule(lr){2-11}
			
			\quad NSP-SeqKD + \textit{LR-RGDA} & 77.10 & 85.62 & 68.91 & 79.59 & 86.28 & 90.70 & 71.17 & 77.20 & 75.87 & 83.28 \\
			\quad \quad + \textit{HopDC} & 80.48 & 86.81 & 80.24 & 83.86 & 86.85 & 91.05 & 72.02 & 77.61 & 79.90 \textcolor{red}{\tiny +4.03} & 84.83 \textcolor{red}{\tiny +1.55} \\
			\cmidrule(lr){2-11}
			
			\quad LoRA-SeqFT + \textit{LR-RGDA} & 65.99 & 78.11 & 50.74 & 62.36 & 75.71 & 84.09 & 59.65 & 68.34 & 63.02 & 73.23 \\
			\quad \quad + \textit{HopDC} & 81.56 & 87.72 & 80.98 & 86.98 & 79.75 & 88.71 & 69.96 & 76.55 & 78.06 \textcolor{red}{\tiny +15.04} & 84.99 \textcolor{red}{\tiny +11.76} \\
			\cmidrule(lr){2-11}
			
			\quad LoRA-SeqKD + \textit{LR-RGDA} & 73.95 & 84.98 & 70.14 & 82.32 & 84.97 & 91.08 & 70.28 & 77.34 & 74.84 & 83.93 \\
			\quad \quad + \textit{HopDC} & \textbf{81.77} & \textbf{87.90} & \textbf{83.07} & \textbf{87.88} & 86.52 & \textbf{91.62} & 73.93 & \textbf{79.88} & \textbf{81.32} \textcolor{red}{\tiny +6.48} & \textbf{86.82} \textcolor{red}{\tiny +2.89} \\
			\bottomrule
		\end{tabular}
	}
	\label{tab:results_mocov3_new}
	\normalsize
\end{table*}

Table~\ref{tab:results_mocov3_new} presents the comprehensive CIL performance on the ViT/B-MoCoV3 backbone, which is pre-trained via self-supervised contrastive learning. Compared to the supervised ViT/B-Sup21K results in the main text, MoCoV3 presents a more challenging scenario due to its less task-aligned feature space. The analysis reveals several key insights that further validate the robustness and generalizability of our proposed framework.

\textbf{1. Superiority over state-of-the-art Baselines.}
Our framework consistently outperforms all existing baselines across all four datasets. The strongest baseline, CoMA, achieves an average Last-Acc of 74.63\%. In contrast, our best-performing variant (LoRA-SeqKD + LR-RGDA + HopDC) reaches 81.32\%. 

\textbf{2. Critical role of HopDC in mitigating severe Drift.}
The results highlight an even more dramatic effect of HopDC on MoCoV3 compared to Sup21K. For highly plastic strategies like SeqFT and LoRA-SeqFT, the performance without HopDC is severely degraded (62.61\% and 63.02\% average Last-Acc, respectively). This indicates that self-supervised features undergo more substantial and nonlinear representation drift when the backbone is aggressively updated. The application of HopDC recovers these strategies to competitive levels (78.50\% and 78.06\%), with improvements exceeding +15 pp. 

\textbf{3. Synergy between LR-RGDA and stable adaptation strategies.}
While HopDC is essential for plastic strategies, the combination of LR-RGDA with stable adaptation methods like SeqKD and NSP-SeqKD already yields strong baseline performance (71.81\% and 75.87\%). This suggests that LR-RGDA effectively leverages the well-preserved feature structure maintained by these strategies. Furthermore, adding HopDC provides a consistent boost, pushing the final accuracy to 81.22\% and 79.90\%, respectively. 

\textbf{4. Effectiveness of parameter-efficient fine-Tuning (LoRA).}
The LoRA-based variants show compelling results. LoRA-SeqKD + LR-RGDA + HopDC achieves the overall best performance (81.32\% Last-Acc, 86.82\% Inc-Acc). Notably, LoRA-SeqFT, which is highly parameter-efficient, benefits tremendously from HopDC (+15.04 pp), achieving performance nearly on par with much more computationally expensive full fine-tuning methods. 

\textbf{5. Dataset-specific observations.}
On fine-grained datasets like Cars-196, the gains are particularly pronounced. For instance, SeqFT + LR-RGDA improves from 46.05\% to 78.87\% with HopDC. This indicates that fine-grained distinctions are especially vulnerable to representation drift, and our associative memory-based compensation is highly effective at preserving these nuanced class boundaries. On CIFAR-100, LR-RGDA alone (e.g., with RanProj) achieves the best single score (87.89\%), highlighting the inherent strength of the analytic classifier on this dataset even with a frozen backbone.

\subsection{Performance Comparison of Different Classifiers With or Without HopDC on Within-Domain CIL Datasets}
\label{sec:extended_experiments_within}
\begin{figure*}[htbp]
	\centering
	\begin{subfigure}[b]{0.49\linewidth}
		\centering
		\includegraphics[width=\linewidth]{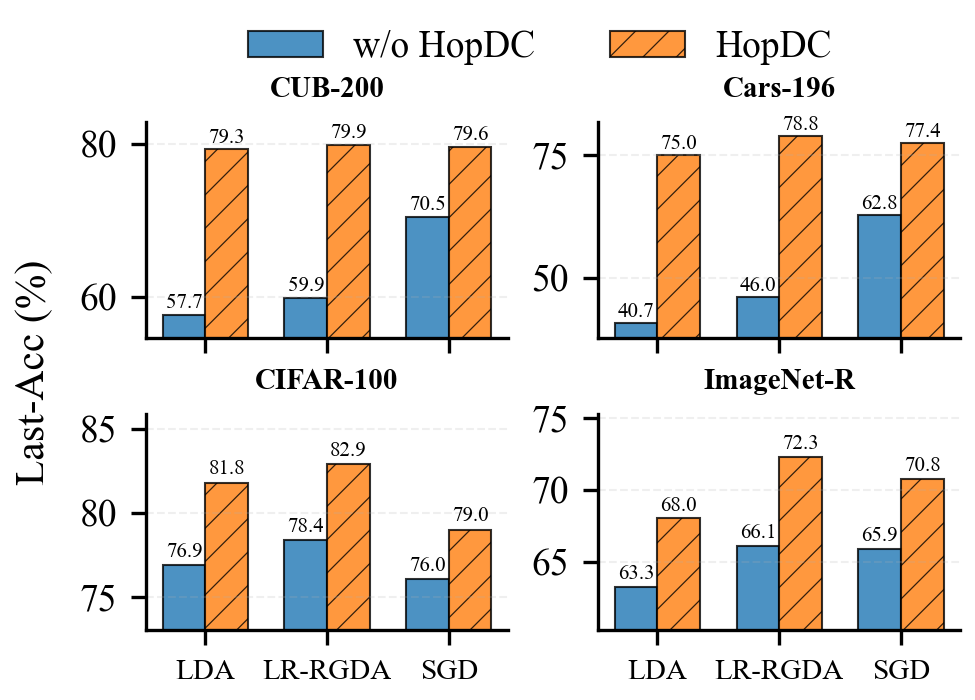}
		\caption{SeqFT}
		\label{fig:seqft}
	\end{subfigure}
	\hfill
	\begin{subfigure}[b]{0.49\linewidth}
		\centering
		\includegraphics[width=\linewidth]{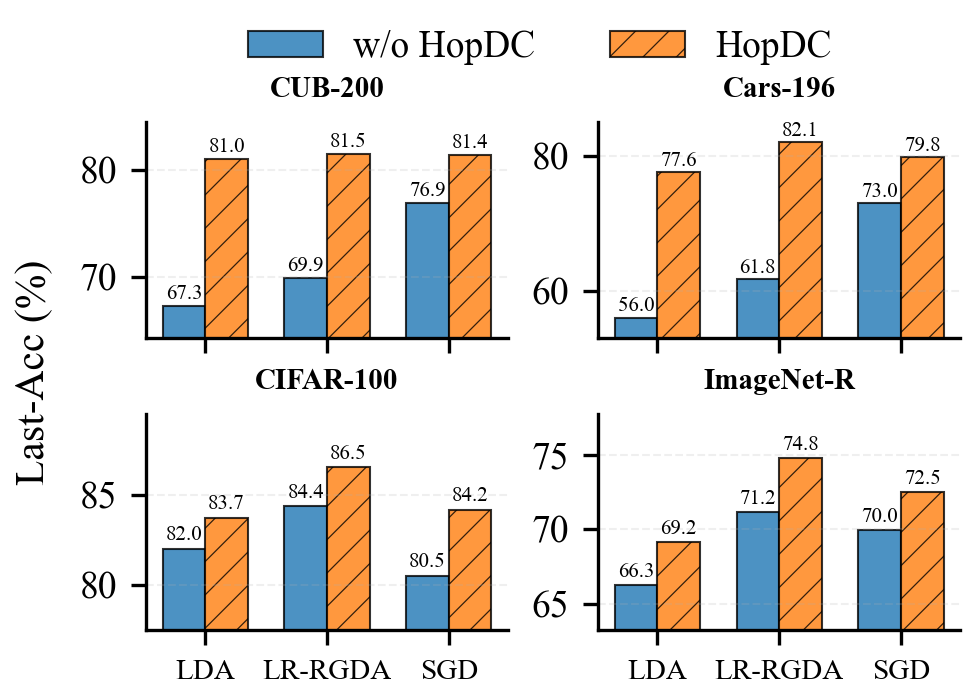}
		\caption{SeqKD}
		\label{fig:seqkd}
	\end{subfigure}
	\vskip\baselineskip
	\begin{subfigure}[b]{0.49\linewidth}
		\centering
		\includegraphics[width=\linewidth]{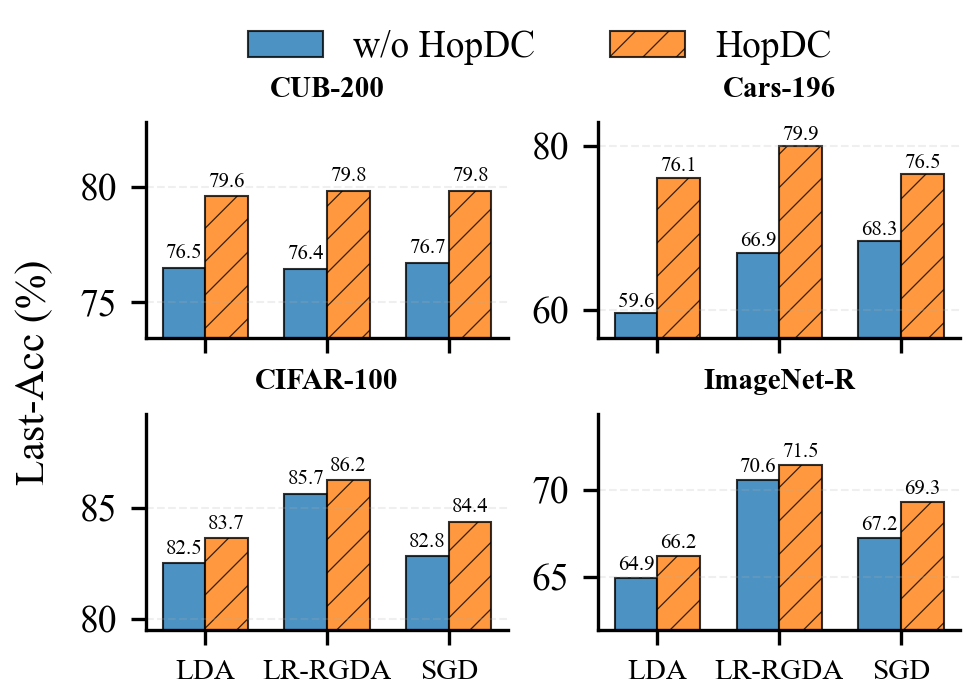}
		\caption{NSP-SeqFT}
		\label{fig:nsp-seqft}
	\end{subfigure}
	\hfill
	\begin{subfigure}[b]{0.49\linewidth}
		\centering
		\includegraphics[width=\linewidth]{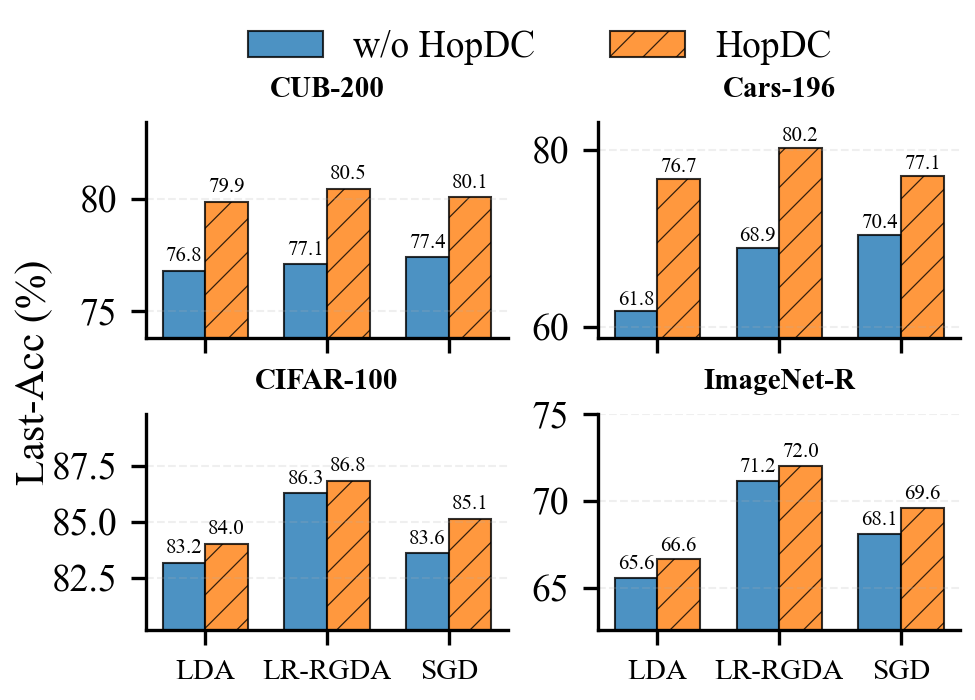}
		\caption{NSP-SeqKD}
		\label{fig:nsp-seqkd}
	\end{subfigure}
	\caption{Comparative analysis of Last-Acc (\%) on four within-domain CIL datasets using the ViT/B-MoCoV3 backbone. We evaluate LDA, LR-RGDA of rank 64, and SGD-trained classifiers under four backbone adaptation strategies: (a) SeqFT, (b) SeqKD, (c) NSP-SeqFT, and (d) NSP-SeqKD. \textbf{Key Observations are described as follows.} (1) LR-RGDA consistently outperforms both the linear baselines (LDA and SGD-based classifiers) across all settings. (2) HopDC serves as a vital drift rectification mechanism, especially for highly plastic strategies (e.g., SeqFT). Notably, in the SeqFT setting on Cars-196 (a), HopDC improves the accuracy of LDA from $\sim 40.7\%$ to $\sim 75.0\%$ and LR-RGDA from $\sim 46.0\%$ to $\sim 78.8\%$, which indicates that performance degradation stems primarily from feature distribution misalignment.}
	\label{fig:seqftcomparison}
\end{figure*}

Figure~\ref{fig:seqftcomparison} presents a detailed comparison of classifiers across four backbone update strategies. LR-RGDA consistently demonstrates superior performance compared to both the linear baseline (LDA) and the iteratively optimized SGD classifier. Notably, in the SeqFT setting on Cars-196 (Figure~\ref{fig:seqft}), LR-RGDA achieves substantial gains over LDA even in the absence of drift compensation. This confirms that the low-rank quadratic perturbations in LR-RGDA effectively capture fine-grained, class-specific manifold structures that are typically inaccessible to linear decision boundaries.

The advantages of HopDC are prominent in highly plastic strategies. In SeqFT (Figure~\ref{fig:seqft}), which is prone to severe catastrophic forgetting due to significant representation drift, HopDC serves as a critical rectification mechanism. On the Cars-196 dataset, applying HopDC boosts the accuracy of LDA from $\sim 40.7\%$ to $\sim 75.0\%$, and LR-RGDA from $\sim 46.0\%$ to $\sim 78.8\%$. These improvements suggest that the observed ``forgetting'' in SeqFT is largely attributable to distribution misalignment rather than irreversible information loss. Furthermore, even with stability-oriented strategies like SeqKD (Figure~\ref{fig:seqkd}) and NSP variants, where baselines already perform well, HopDC consistently yields further performance enhancement. 

\textit{Crucially, these consistent gains across the full spectrum of classifiers, that ranges from LDA and SGD to the proposed LR-RGDA, substantiate HopDC as a versatile and classifier-agnostic solution for mitigating distribution drift.}

\subsection{Impact of Low-Rank Dimension $r$ of LR-RGDA}
\label{sec:ablation_rank}

To provide a comprehensive understanding of the low-rank assumption inherent in LR-RGDA, we conduct extensive ablation studies on the subspace dimension $r$ of LR-RGDA across both cross-domain and within-domain CIL scenarios. 

\begin{figure}[htbp]
	\centering
	\includegraphics[width=1.0\linewidth]{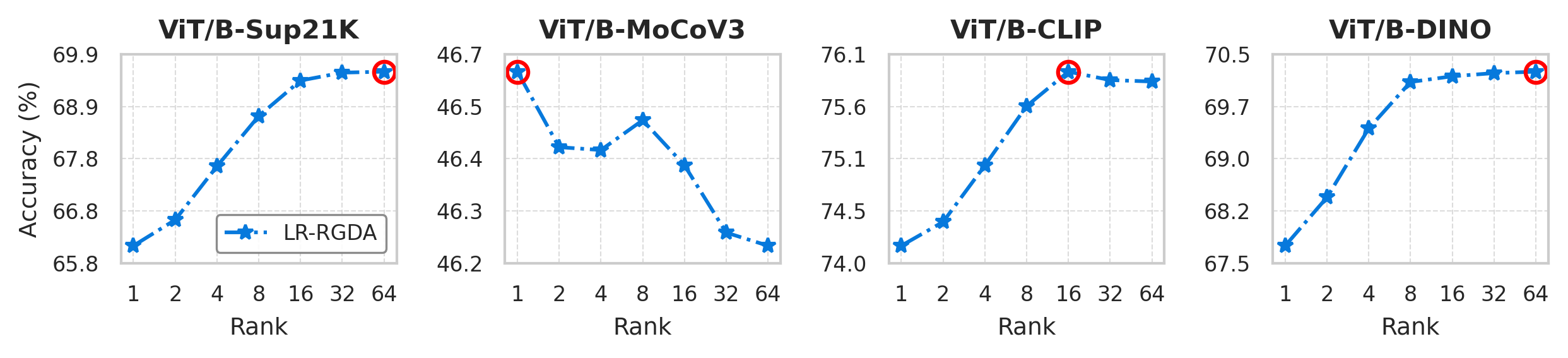}
	\caption{Impact of low-rank dimension $r$ of LR-RGDA on the cross-domain benchmarks across different ViT backbones. We utilize frozen pre-trained ViTs without further fine-tuning. The red circles denote the peak performance points. For most backbones (Sup21K, CLIP, DINO), accuracy improves significantly as $r$ increases to 16, after which it stabilizes or slightly degrades. It indicates that the primary discriminative information is concentrated in the top principal components.}
	\label{fig:rankablationcomparisonfour}
\end{figure}

\begin{figure}[htbp]
	\centering
	\includegraphics[width=1.0\linewidth]{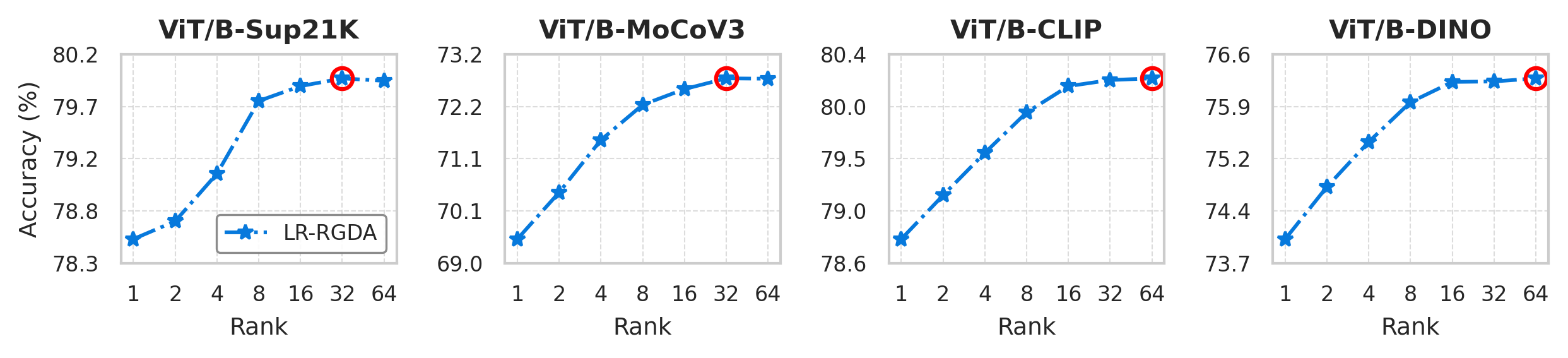}
	\caption{Impact of low-rank dimension $r$ of LR-RGDA on the cross-domain benchmarks across different ViT backbones. \textbf{We utilize ViT backbones that have been fine-tuned for 500 iterations.} The red circles denote the peak performance points. Compared to the frozen baseline, the optimal rank tends to shift towards higher values, and the performance remains robust at higher ranks (e.g., $r=32$ or $64$). This indicates that short-term fine-tuning rectifies the feature manifold and enriches the discriminative information contained in the trailing principal components.}
	\label{fig:rankablationcomparisonfour_iter500}
\end{figure}

\paragraph{Analysis of cross-domain benchmarks.} 
We investigate the sensitivity of LR-RGDA to the subspace dimension $r$ of LR-RGDA across diverse pre-trained ViTs. By contrasting the performance of frozen backbones (Figure~\ref{fig:rankablationcomparisonfour}), which rely on static pre-trained features, with their counterparts that are subjected to short-term optimization for 500 iterations (Figure~\ref{fig:rankablationcomparisonfour_iter500}), we observe how the backbone actively reshapes its representation space. This comparison reveals two distinct phases of spectral evolution:
\begin{enumerate}
	\item As illustrated in Figure~\ref{fig:rankablationcomparisonfour}, distinct trends emerge depending on the pre-training paradigms. For strong transfer learners like ViT/B-Sup21K, ViT/B-CLIP, and ViT/B-DINO, the classification accuracy exhibits a steep ascent as $r$ increases from 1 to 16. It validates our core hypothesis regarding the \textit{spectral concentration} of semantic information: the principal directions of the class-specific covariance matrices capture the robust, discriminative features, while the trailing eigenvectors are often dominated by task-irrelevant noises. Notably, the performance tends to saturate or even decline slightly beyond $r=32$. This ``peaking phenomenon'' suggests that a full-rank estimation (as in standard RGDA) incurs a higher risk of overfitting to the training samples, whereas a truncated low-rank approximation ($r \approx 16$) acts as an effective regularizer.
	
	\item When the backbones are fine-tuned for 500 iterations only, as shown in Figure~\ref{fig:rankablationcomparisonfour_iter500}, we observe a significant shift in spectral properties where the optimal rank generally moves towards higher values across all architectures. This trend indicates a dimensionality expansion of the effective feature space, which is most dramatic in \textbf{ViT/B-MoCoV3}. Although the frozen baseline of this architecture suffers from a severe performance collapse at high ranks that indicates noisy trailing eigenvectors, its fine-tuned counterpart exhibits a consistent monotonic improvement that peaks at $r=32$ or $64$. Such a transformation confirms that short-term optimization effectively ``rectifies'' the feature manifold by suppressing noise in the lower-variance directions while encoding task-relevant semantics into a higher-dimensional subspace.
\end{enumerate}

\paragraph{Analysis of within-domain benchmarks.} 
We systematically investigate the impact of backbone optimization on the spectral properties of class covariances by comparing the frozen baseline on within-domain datasets (Figure~\ref{fig:rankablationcomparison_frozen}) with the model fine-tuned for 500 iterations (Figure~\ref{fig:rankablationcomparison_finetuned}). Our empirical findings reveal four critical phases of representation evolution:
\begin{itemize}
	\item \textbf{Low-rank concentration in frozen representations:} 
	In the absence of task-specific adaptation, performance consistently peaks at a low rank (typically $r \in \{2, 4, 8\}$) across most ViT architectures. It implies that the effective fine-grained knowledge in pre-trained ViTs is intrinsically low-dimensional. Notably, on fine-grained datasets (e.g., CUB-200, Cars-196), the accuracy curves exhibit a sharp ``inverted-V'' shape, which indicates that increasing $r$ beyond the optimal point introduces high-frequency noise from the pre-training manifold, leading to the significant degradation, where such sensitivity is most pronounced in ViT/B-MoCoV3.
	
	\item \textbf{Dimensionality expansion via fine-tuning:} 
	Short-term optimization (500 steps) drives the optimal rank towards higher values (typically shifting to $r \in \{8, 16, 32\}$). This rightward shift demonstrates that fine-tuning effectively extracts task-specific fine-grained information from downstream datasets, thereby expanding the intrinsic dimensionality of the class manifolds. Consequently, LR-RGDA necessitates a higher-rank subspace to accurately reconstruct these enriched feature correlations.
	
	\item \textbf{Feature space rectification and stability:} 
	Beyond the rank shift, fine-tuning fundamentally alters the ``peak-to-tail'' dynamics. Unlike the sharp decay observed in frozen models, fine-tuned models exhibit a stable performance plateau even at full rank ($r=64$). This suggests that backbone optimization acts as a \textit{feature space rectification} process: it aligns the trailing principal components with task-relevant semantics while suppressing irrelevant noise.
	
	\item \textbf{Heterogeneous architecture sensitivity:} 
	The impact of fine-tuning varies by pre-training paradigm. For instance, ViT/B-CLIP shows marginal gains on coarse-grained tasks (e.g., CIFAR-100) where its zero-shot priors are already robust, yet achieves substantial improvement on fine-grained tasks (e.g., Cars-196). It indicates that while vision-language pre-training provides excellent separability for generic concepts, backbone adaptation is essential for capturing visual nuances that textual descriptions may miss.
\end{itemize}

\begin{figure}[htbp]
	\centering
	\includegraphics[width=1.0\linewidth]{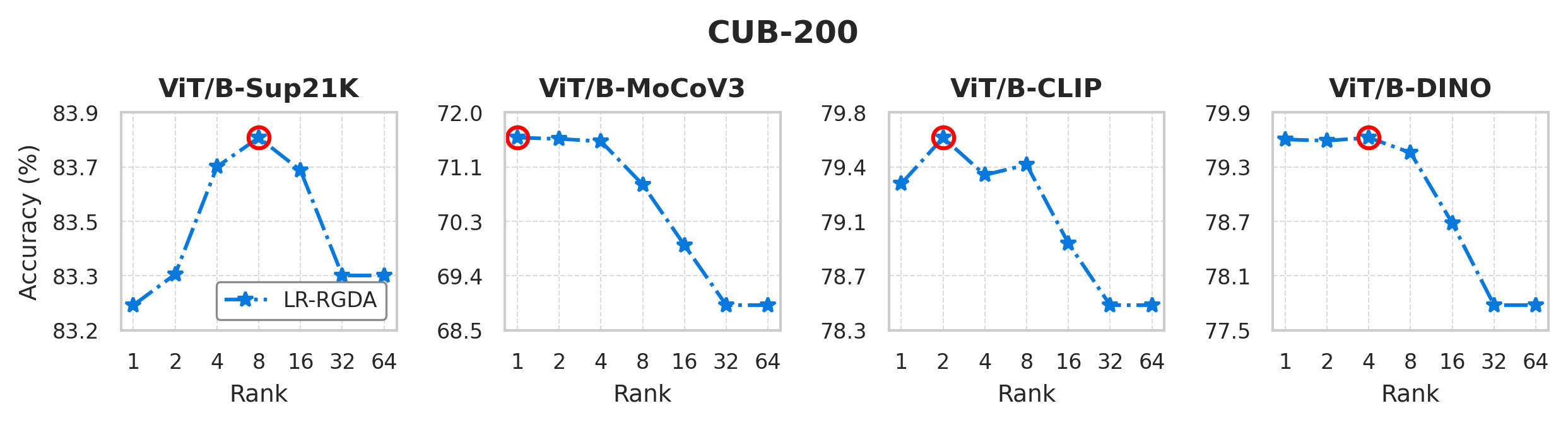}
	\vspace{-0.5em}
	\includegraphics[width=1.0\linewidth]{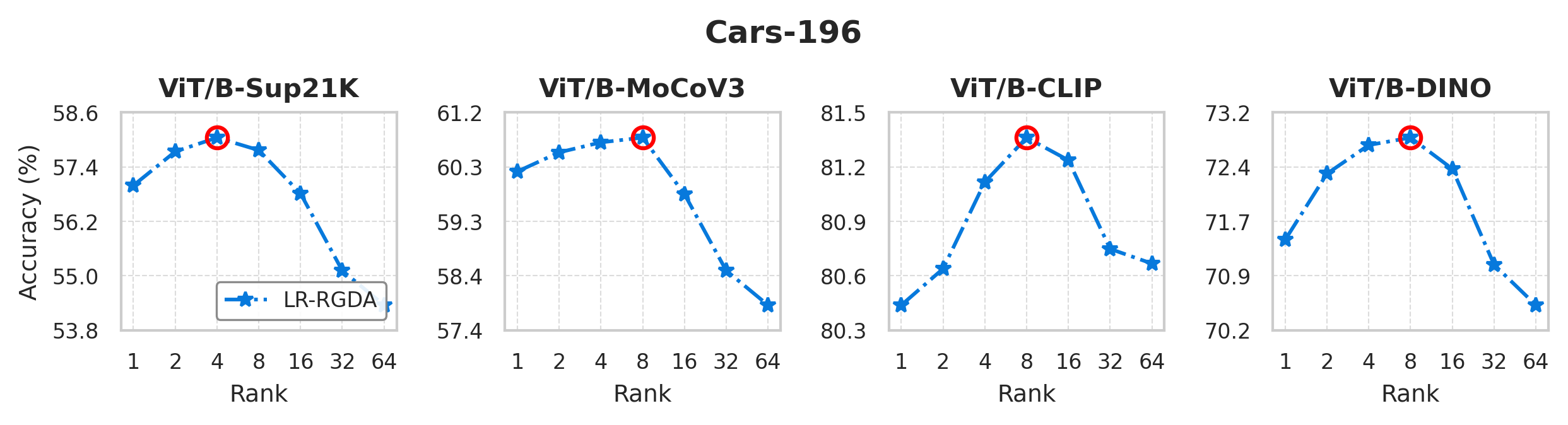}
	\vspace{-0.5em}
	\includegraphics[width=1.0\linewidth]{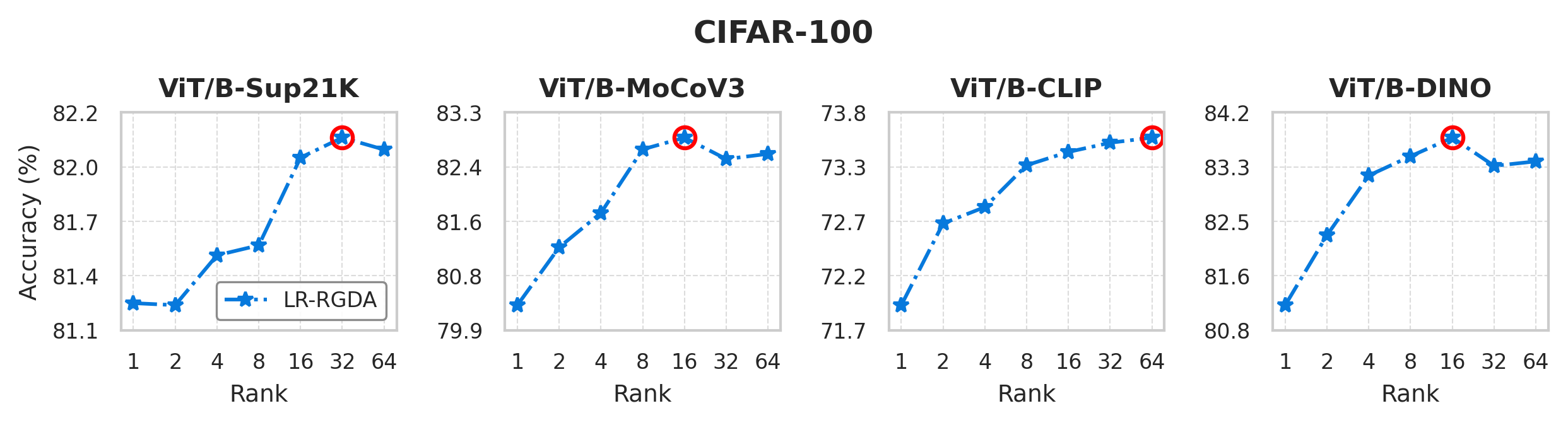}
	\vspace{-0.5em}
	\includegraphics[width=1.0\linewidth]{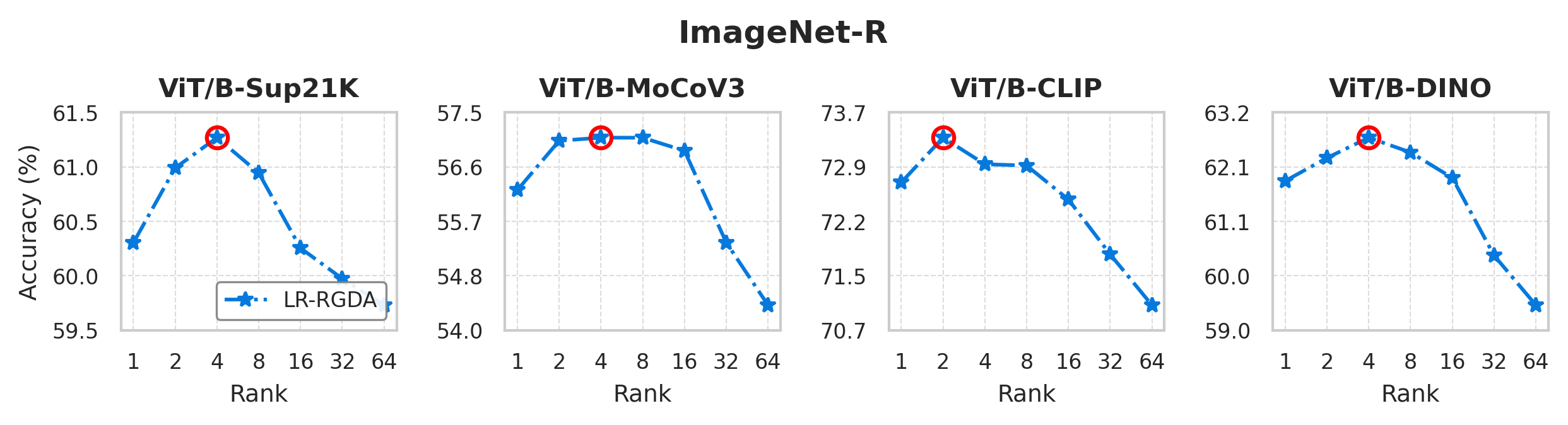}
	\caption{Rank sensitivity analysis on within-domain benchmarks using frozen pre-trained ViT backbones. Performance consistently peaks at lower ranks (e.g., $r=4$ or $8$) and degrades at higher ranks. It implies that discriminative information in frozen features is concentrated in a small subspace. It is worth noting that ViT/B-MoCoV3 has distinct behavior s: it suffers from severe performance collapse at high ranks, and the sharp ``inverted-V'' trends on fine-grained datasets (CUB-200, Cars-196), which highlights their reliance on specific covariance structures.}
	\label{fig:rankablationcomparison_frozen}
	\end{figure}

\begin{figure}[htbp]
	\centering
	\includegraphics[width=1.0\linewidth]{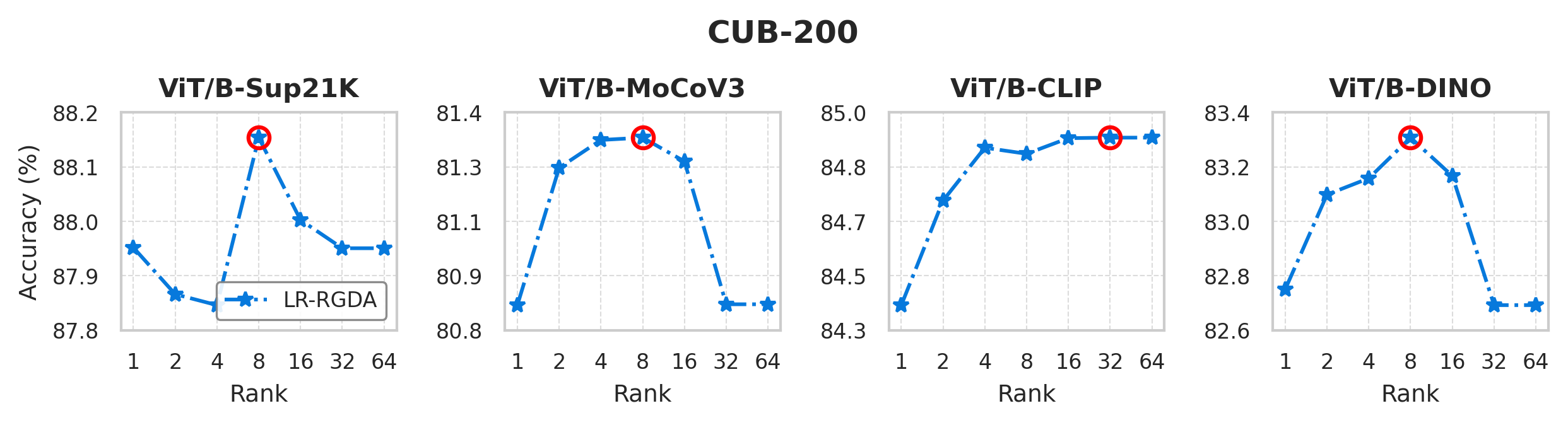}
	\vspace{-0.5em}
	\includegraphics[width=1.0\linewidth]{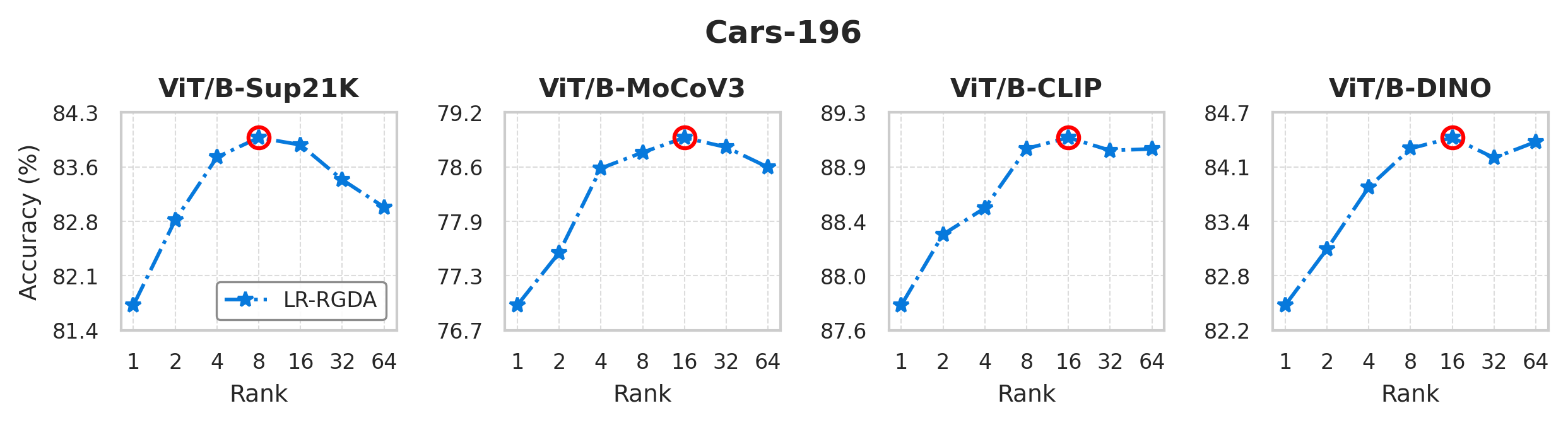}
	\vspace{-0.5em}
	\includegraphics[width=1.0\linewidth]{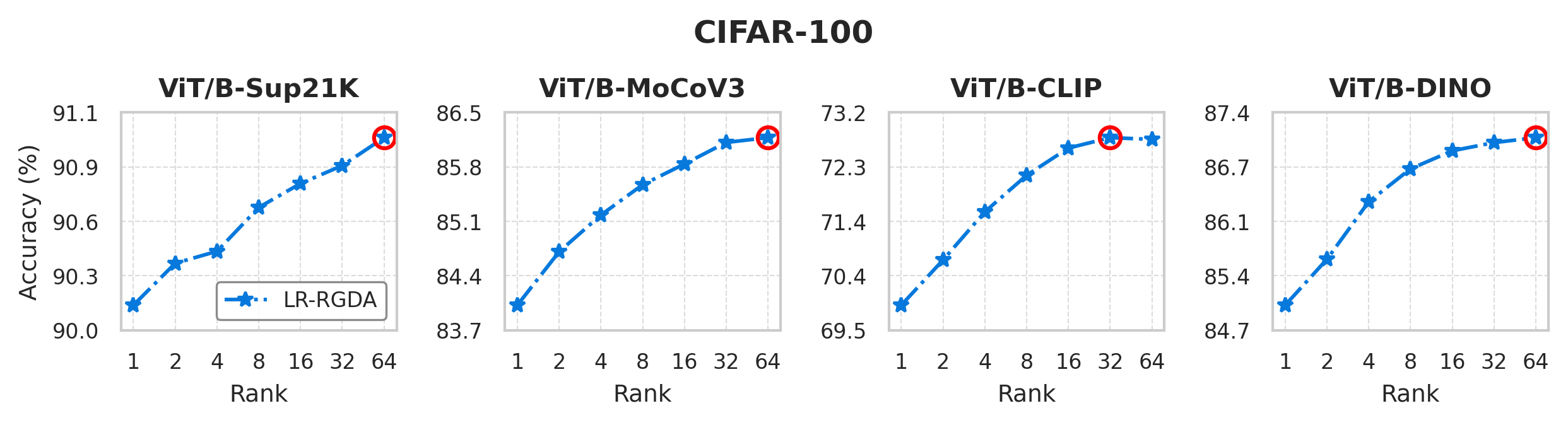}
	\vspace{-0.5em}
	\includegraphics[width=1.0\linewidth]{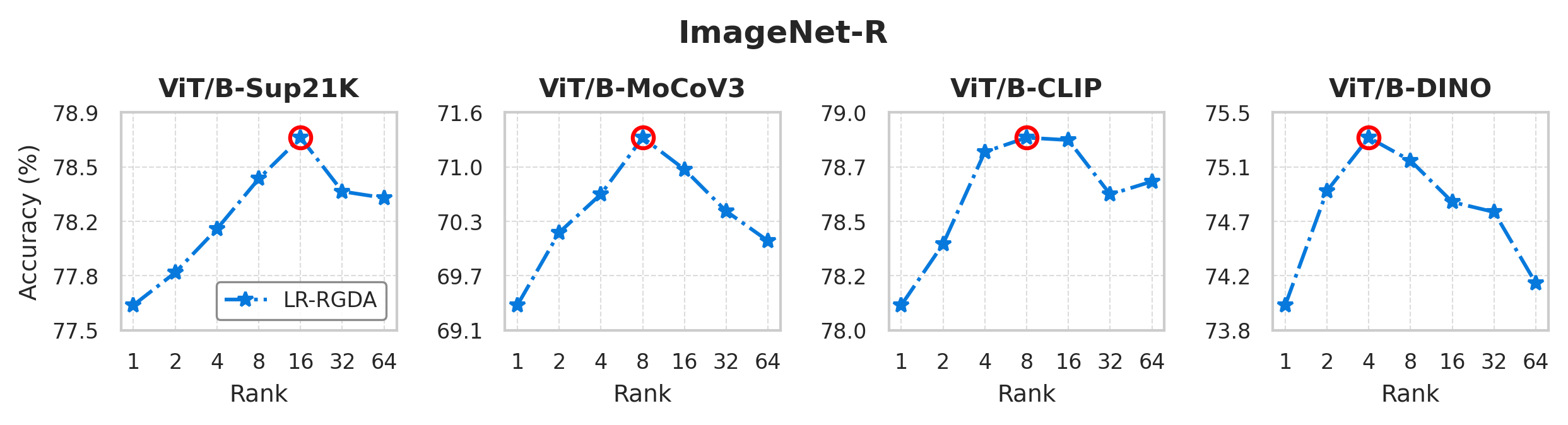}
	\caption{Rank sensitivity analysis on within-domain benchmarks after short-term backbone fine-tuning (500 iterations). 
	Compared to the frozen baseline, the optimal rank shifts to higher values (e.g., $r \in \{16, 32\}$). It indicates that the fine-tuning can enrich the feature space with task-specific details. Crucially, the performance at high ranks ($r=64$) stabilizes into a plateau rather than decaying, which suggests that backbone optimization rectifies the feature manifold and suppresses task-irrelevant noise in the trailing principal components.}
	\label{fig:rankablationcomparison_finetuned}
	\end{figure}
	
\subsection{Hyperparameter Sensitivity and Performance Contours}
\label{sec:contours}
We perform a comprehensive sensitivity analysis of the regularization hyperparameters $\alpha_1$ (class-specific weight) and $\alpha_2$ (global shared weight) to understand the operating mechanism of LR-RGDA. We fix $\alpha_3=0.5$ across all experiments.
\begin{figure}[t]
	\centering
	\includegraphics[width=0.70\linewidth]{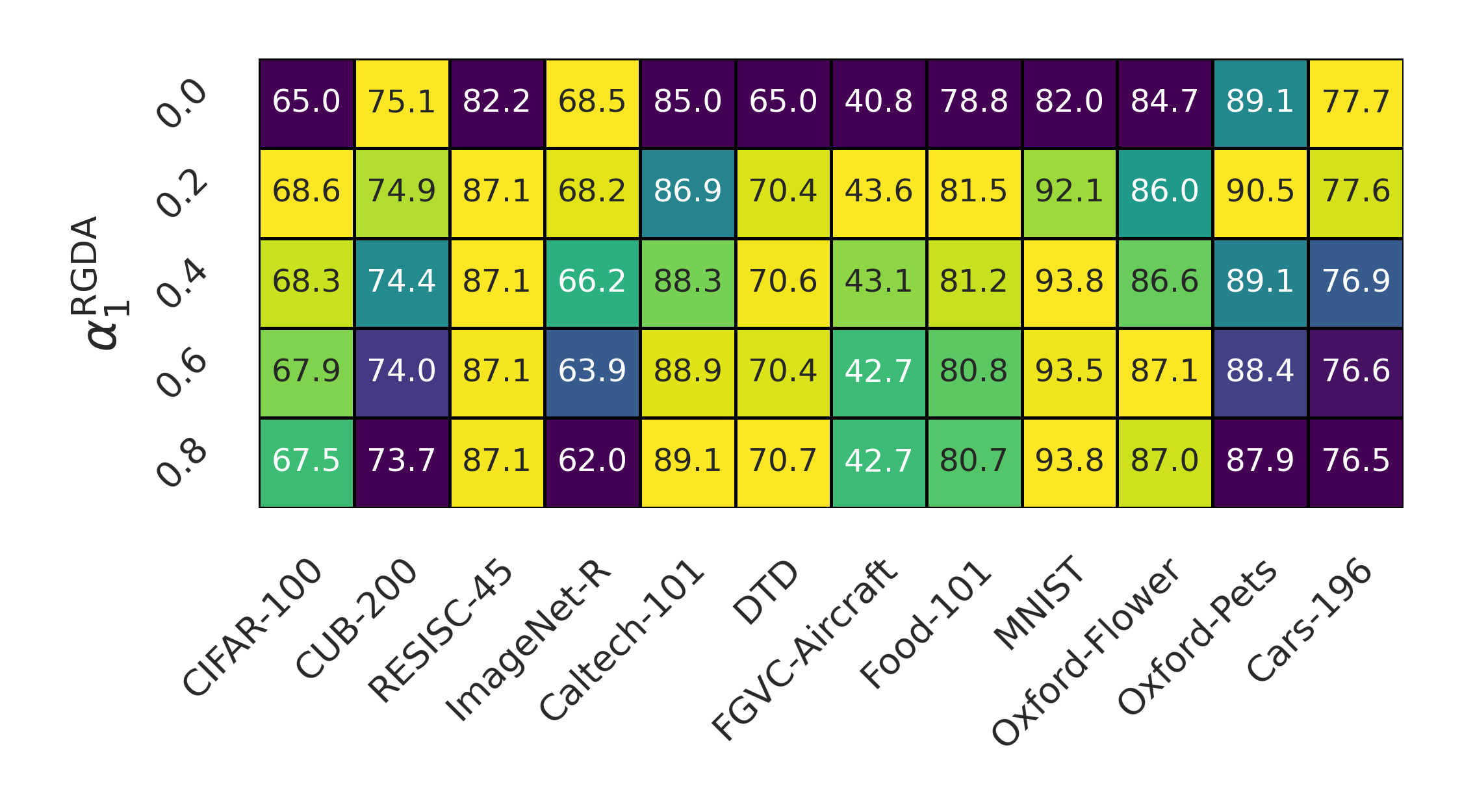}
	\includegraphics[width=0.70\linewidth]{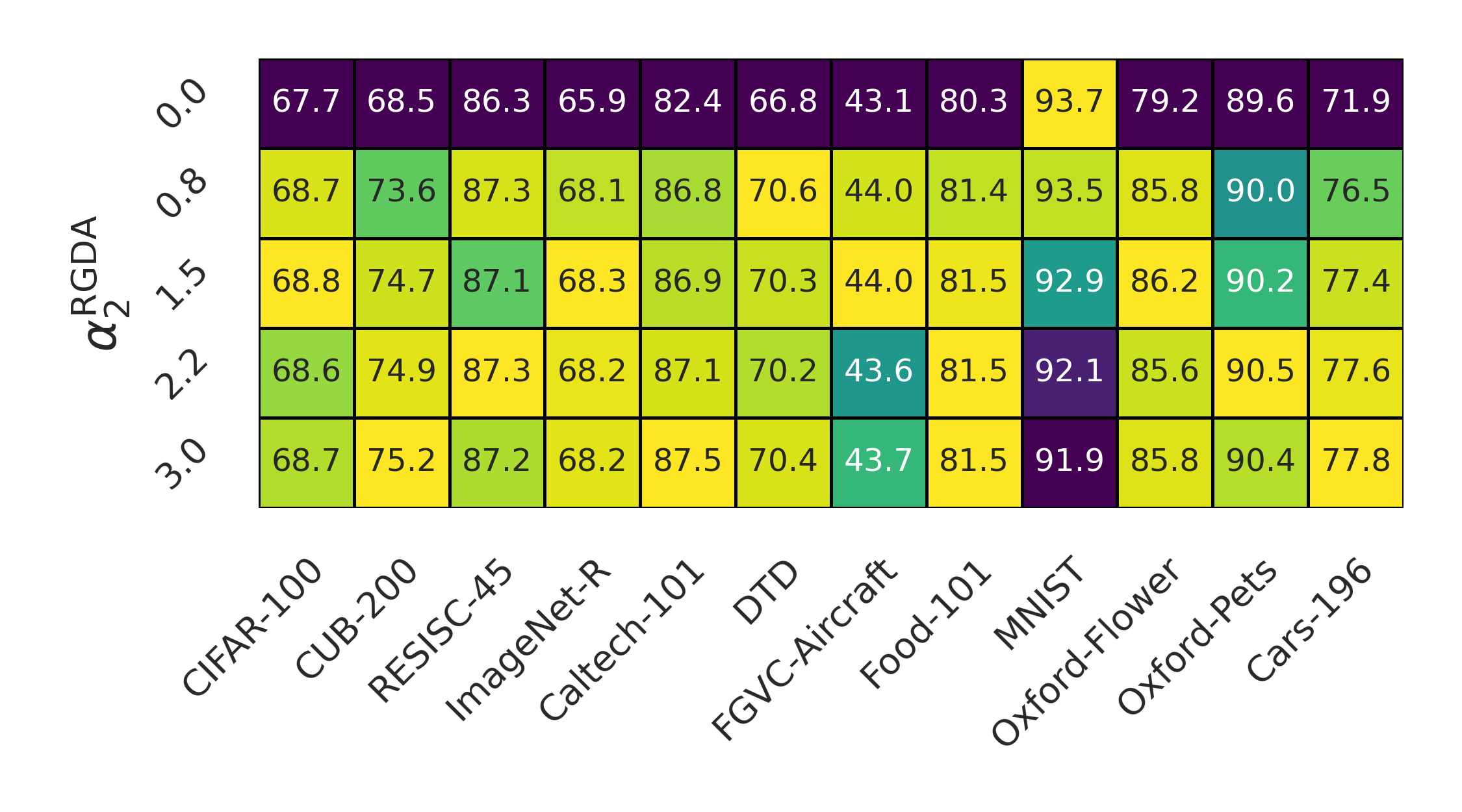}
	\caption{Per-dataset hyperparameter sensitivity analysis. We evaluate LR-RGDA on an expanded 12-dataset cross-domain CIL benchmark using a frozen ViT/B-CLIP. Note that the notations $\alpha_{\rm 1}^{\rm RGDA}$ and $\alpha_{\rm 2}^{\rm RGDA}$ in the figure correspond to $\alpha_1$ and $\alpha_2$ in the main text, respectively. The results highlight that $\alpha_2$ provides a universal stability term for all datasets, while $\alpha_1$ acts as a fine-tuning knob dependent on dataset characteristics.}
	\label{fig:exp3datasetsensitivityvit-b-p16-clipalpha1}
\end{figure}

\paragraph{Per-Dataset sensitivity on expanded benchmarks.}
To rigorously evaluate the robustness across diverse data distributions, we extend the cross-domain CIL evaluation from the main text's 7-dataset benchmark to a larger collection of 12-datasets by adding RESISC-45 \cite{cheng2017remote}, DTD \cite{cimpoi14describing}, FGVC-Aircraft \cite{maji13fine-grained}, MNIST \cite{lecun2010mnist}, and Oxford-Pets \cite{parkhi12a}. Figure~\ref{fig:exp3datasetsensitivityvit-b-p16-clipalpha1} visualizes the performance heatmaps using a frozen ViT/B-CLIP encoder.

\begin{itemize}
	\item Nuance of class-specific tuning ($\alpha_1$): The top heatmap shows that while small $\alpha_1$ values ($0.1 - 0.2$) are generally optimal, and the dataset heterogeneity exists. For examples, highly structured datasets like MNIST and Cars-196 tolerate slightly higher $\alpha_1$ values compared to texture-heavy datasets like DTD. It suggests that datasets with distinct geometric boundaries benefit more from the class-specific covariance information, whereas vague categories require stronger global regularization to prevent overfitting.
	
	\item Criticality of global regularization ($\alpha_2$): The bottom heatmap reveals a consistent pattern: setting $\alpha_2 \approx 0.0$ causes a catastrophic performance collapse across all datasets (indicated by deep purple cells). This confirms that the global average covariance acts as a foundational term that can stabilize the ill-conditioned class-specific statistics derived from few-shot datasets.
\end{itemize}

\begin{figure}[htbp]
	\centering
	\begin{subfigure}{1.0\linewidth}
		\centering
		\includegraphics[width=\linewidth]{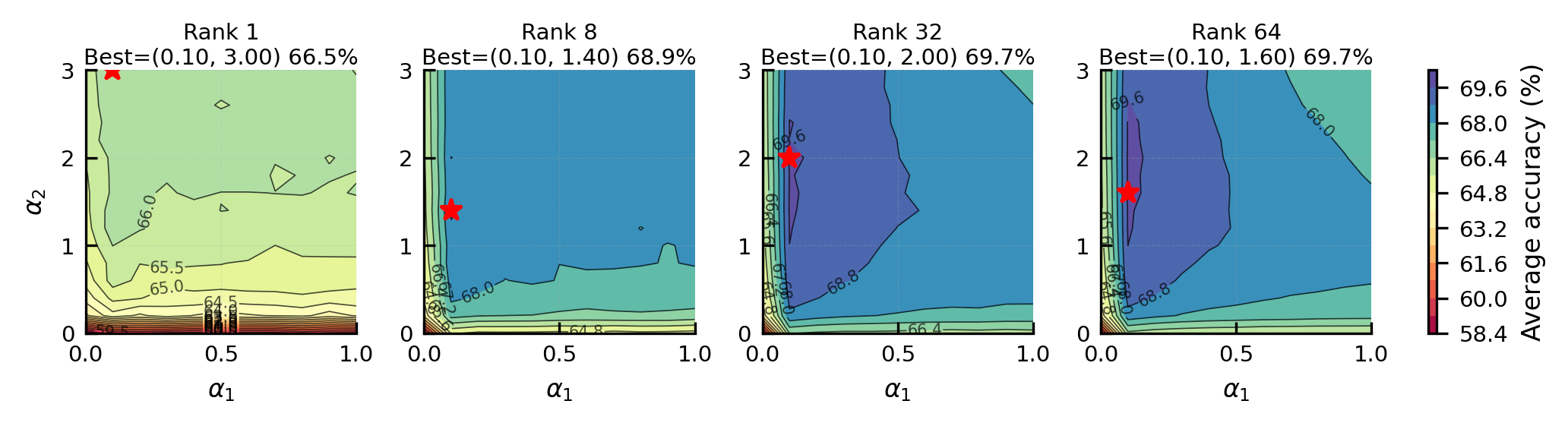}
		\caption{ViT/B-Sup21K}
	\end{subfigure}
	\begin{subfigure}{1.0\linewidth}
		\centering
		\includegraphics[width=\linewidth]{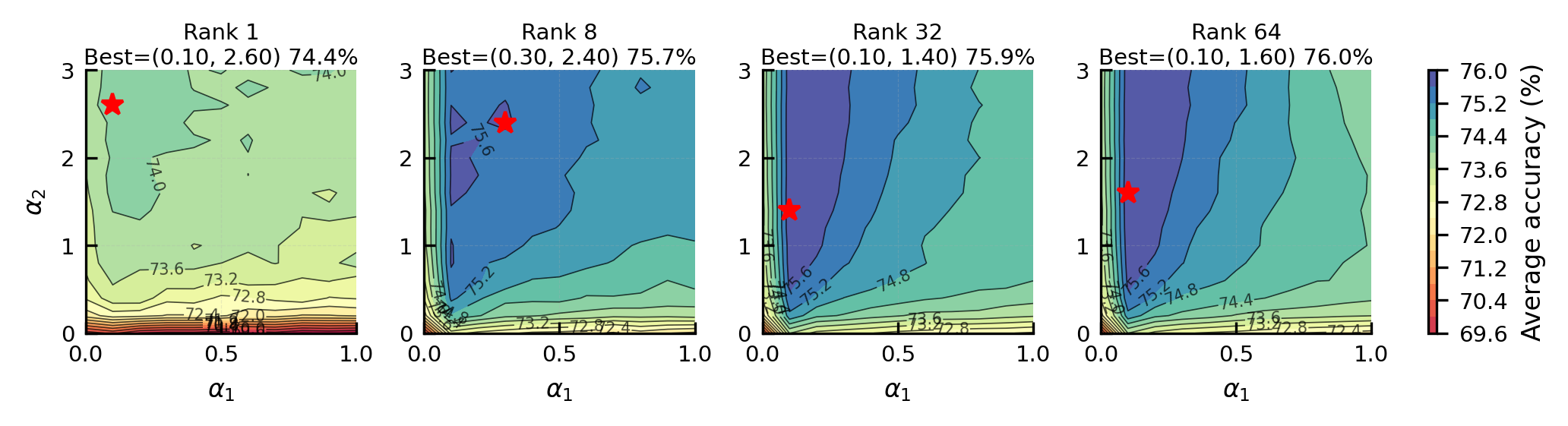}
		\caption{ViT/B-CLIP}
	\end{subfigure}
	\begin{subfigure}{1.0\linewidth}
		\centering
		\includegraphics[width=\linewidth]{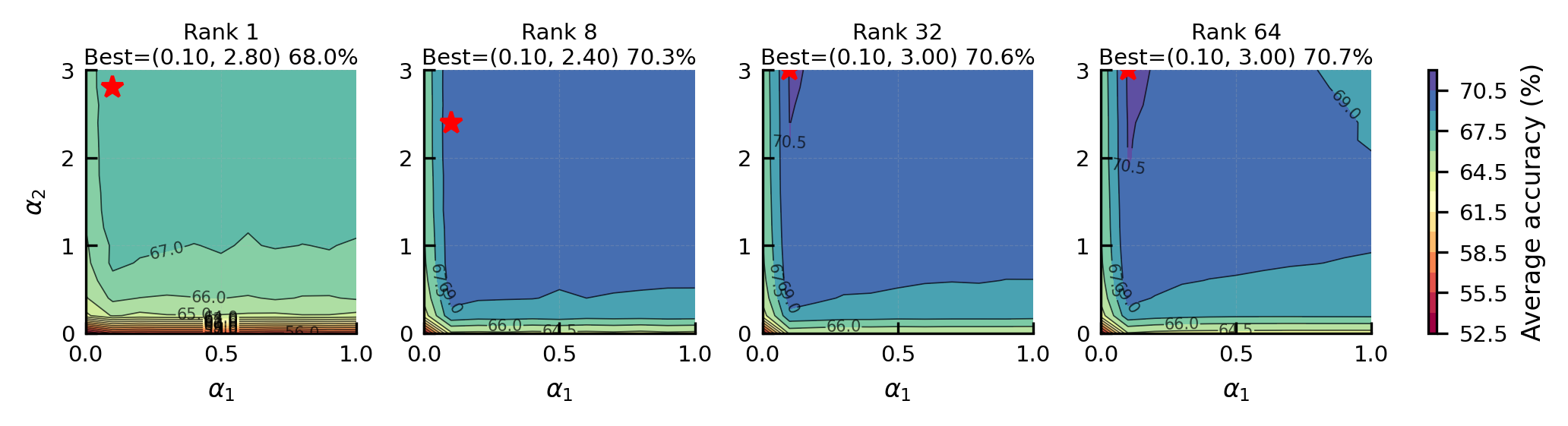}
		\caption{ViT/B-DINO}
	\end{subfigure}
	\begin{subfigure}{1.0\linewidth}
		\centering
		\includegraphics[width=\linewidth]{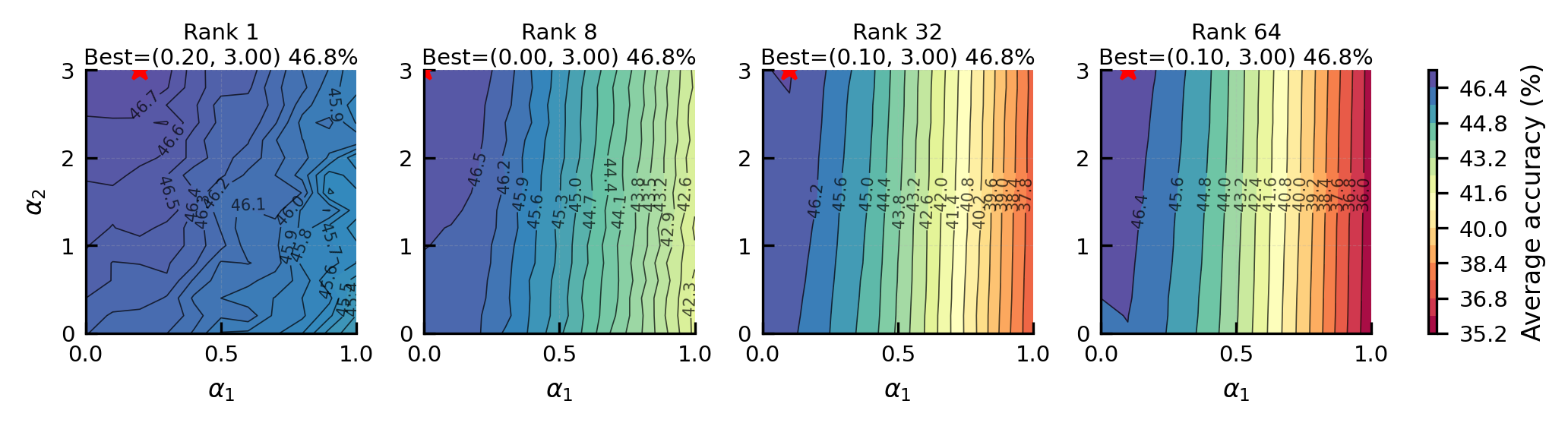}
		\caption{ViT/B-MoCoV3}
	\end{subfigure}
    \caption{Performance contours of LR-RGDA across different pre-trained backbones. The x-axis represents $\alpha_1$ (class-specific weight) and the y-axis represents $\alpha_2$ (the global weight). The consistent convergence of optimal regions to the upper-left quadrant demonstrates the structural universality of LR-RGDA across diverse representation spaces of different ViTs.}
	\label{fig:multirankcontourvit-b-p16iter0}
\end{figure}

\paragraph{Performance Landscapes Across Architectures.}
Figure~\ref{fig:multirankcontourvit-b-p16iter0} presents the performance contour plots on the 7-dataset cross-domain CIL benchmark for four distinct ViT architectures (Sup21K, MoCoV3, CLIP, DINO) with different subspace ranks. We sample $\alpha_1 \in [0, 1.0]$ and $\alpha_2 \in [0, 3.0]$ with 11 steps. The visualization yields three key discoveries regarding the geometry of the performance landscape:

\begin{itemize}
	\item The ``upper-left'' universality: Regardless of the pre-training paradigm (supervised, contrastive, or multi-modal), the peak performance regions (dark blue) consistently converge to the upper-left quadrant (low $\alpha_1 \approx 0.1$, high $\alpha_2 \in [1.5, 3.0]$). It empirically validates our theoretical factorization: the global base must dominate to ensure invertibility and generalization, while the class-specific term should act as a low-rank \textit{perturbation}.

	\item The stability threshold of $\alpha_2$: We observe a ``ridge of stability'' along the vertical axis. Once $\alpha_2$ exceeds a threshold (typically $>1.0$), the contours become sparse vertically. It indicates that performance becomes largely insensitive to the exact magnitude of $\alpha_2$. This property significantly simplifies hyperparameter tuning—practitioners can set a sufficiently large $\alpha_2$ and focus solely on tuning $\alpha_1$.
	
	\item Robustness of vision-language PTMs: When comparing the contour shapes across different pre-trained ViTs, ViT/B-CLIP exhibits a notably broader high-performance basin compared to MoCoV3 or DINO. It implies that the feature space aligned with language is more robust to regularization mismatches, whereas self-supervised features require more precise hyperparameter calibration to balance the noise and signal.
\end{itemize}

\subsection{Robustness and Efficiency Analysis of HopDC}
\label{sec:app_robustness}
In this section, we assess the practical deployability of the HopDC mechanism. A robust drift compensator should meet two criteria: (1) it should be insensitive to hyperparameter choices to ensure the ease of use, and (2) it should be data-efficient, such that it can function effectively even when the auxiliary unlabelled anchor set is small.

\paragraph{Robustness to hyperparameters $\tau$ and $k$.}
We evaluate the sensitivity of HopDC to the softmax temperature $\tau$ (which controls the sharpness of the attention mechanism) and the retrieval sparsity $k$ (the number of top anchors considered). The results are summarized in Table~\ref{tab:app_hyperparams}.
\begin{itemize}
	\item Impact of temperature ($\tau$):
	The optimal temperature correlates with task granularity. For fine-grained datasets like Cars-196, analytic classifiers (LR-RGDA and LDA) benefit significantly from lower temperatures (e.g., $\tau=0.01$). A lower $\tau$ enforces a sharper attention distribution, which ensures that the drift is estimated only from the most semantically similar anchors, thereby filtering out noise from ``neighboring'' but distinct sub-classes. In contrast, for coarse-grained tasks like ImageNet-R, the performance remains stable across a wide range of $\tau$, as the class boundaries are more distinct. Interestingly, SGD-based classifiers are largely insensitive to $\tau$, likely because the iterative optimization implicitly smooths out the noise in the compensated statistics.
	
	\item Impact of sparsity ($k$):
	HopDC exhibits remarkable stability with respect to $k$. Varying $k$ from 100 to 400 results in negligible performance fluctuations ($<0.2\%$) across all datasets and classifiers. It implies that the drift signal is concentrated in the top-100 nearest anchors, and retrieving more anchors does not introduce significant interference. This property allows us to set a conservative $k$ (e.g., 400) without tuning.
\end{itemize}

\begin{table*}[htbp]
	\centering
	\caption{Hyperparameter sensitivity analysis of HopDC. We report the Last-Acc (\%) on Cars-196 and ImageNet-R while varying the temperature $\tau$ and sparsity $k$. The results highlight that while fine-grained tasks prefer to sharper attention (lower $\tau$), the method is generally robust to variations in sparsity $k$.}
	\setlength{\tabcolsep}{4.5pt}
	\renewcommand{\arraystretch}{1.1}
	\begin{tabular}{ll|cccc|ccc}
		\toprule
		\multirow{2}{*}{\textbf{Dataset}} & \multirow{2}{*}{\textbf{Classifier}} & \multicolumn{4}{c|}{\textbf{Temperature} ($\tau$)} & \multicolumn{3}{c}{\textbf{Sparsity} ($k$)} \\
		\cmidrule(lr){3-6} \cmidrule(lr){7-9}
		& & 0.01 & 0.05 & 0.25 & 1.00 & 100 & 200 & 400 \\
		\midrule
		\multirow{3}{*}{Cars-196} 
		& LR-RGDA (rank@64) & 82.91 & 82.76 & 81.69 & 81.46 & 82.76 & 82.78 & 82.66 \\
		& SGD      & 80.04 & 80.16 & 79.63 & 79.53 & 80.24 & 80.24 & 80.16 \\
		& LDA      & 78.54 & 78.29 & 77.27 & 77.01 & 78.45 & 78.45 & 78.29 \\
		\midrule
		\multirow{3}{*}{ImageNet-R} 
		& LR-RGDA (rank@64)  & 74.80 & 74.93 & 74.80 & 74.78 & 74.93 & 74.87 & 74.90 \\
		& SGD      & 72.53 & 72.30 & 72.42 & 72.38 & 72.42 & 72.25 & 72.30 \\
		& LDA      & 69.22 & 69.17 & 69.18 & 69.18 & 69.13 & 69.17 & 69.17 \\
		\bottomrule
	\end{tabular}
	\label{tab:app_hyperparams}
\end{table*}

\paragraph{Data efficiency (impact of auxiliary set size $N$).}
A key constraint in real-world applications of HopDC is the storage budget for auxiliary data. To test the data efficiency, we compare HopDC against the linear baseline $\alpha_1$-SLDC \cite{rao2025compensating} by progressively reducing the auxiliary set size $N$ from 2048 to 256. Table~\ref{tab:app_data_efficiency} presents the comparison results.
\begin{itemize}
	\item Degradation of linear baselines:
	$\alpha_1$-SLDC relies on estimating global drift statistics (e.g., mean drift vectors) from the auxiliary set. As $N$ decreases, these global estimates become statistically unreliable, leading to significant performance drops (e.g., a $-4\%$ drop for LDA on Cars-196 as $N$ reduces from 2048 to 256).
	
	\item Stability of HopDC: 
	In contrast, HopDC maintains near-constant accuracy even when the auxiliary set is reduced by $8\times$ (down to $N=256$). For instance, LR-RGDA with HopDC on Cars-196 sustains $\sim82.68\%$ accuracy across all sizes. This validates the core advantage of our \textit{associative memory} formulation: HopDC does not learn a global parameter but performs local interpolation. Even with a sparse anchor set, the mechanism can effectively retrieve and interpolate the drift field for old prototypes, making it extremely data-efficient.
\end{itemize}

\begin{table*}[htpt]
	\centering
	\caption{Data efficiency analysis (auxiliary Set size $N$). We compare the Last-Acc (\%) of HopDC against the linear baseline $\alpha_1$-SLDC across Cars-196 and ImageNet-R. While the linear $\alpha_1$-SLDC degrades significantly with scarce data, HopDC maintains robust performance even with only 256 anchors, which demonstrates the superiority of the associative memory mechanism.}
	\setlength{\tabcolsep}{4pt}
	\renewcommand{\arraystretch}{1.1}
	\begin{tabular}{ll|cccc|cccc}
		\toprule
		\multirow{2}{*}{\textbf{Classifier}} & \multirow{2}{*}{\textbf{Method}} & \multicolumn{4}{c|}{\textbf{Cars-196}} & \multicolumn{4}{c}{\textbf{ImageNet-R}} \\
		\cmidrule(lr){3-6} \cmidrule(lr){7-10}
		& & $N$=256 & 512 & 1024 & 2048 & $N$=256 & 512 & 1024 & 2048 \\
		\midrule
		\multirow{2}{*}{LR-RGDA} 
		& $\alpha_1$-SLDC & 77.99 & 79.03 & 80.00 & 80.86 & 73.12 & 73.50 & 73.85 & 74.35 \\
		& \textbf{HopDC (Ours)}  & \textbf{82.68} & \textbf{82.68} & \textbf{82.68} & \textbf{82.66} & \textbf{74.80} & \textbf{74.85} & \textbf{74.78} & \textbf{74.90} \\
		\midrule
		\multirow{2}{*}{SGD}      
		& $\alpha_1$-SLDC  & 77.55 & 77.95 & 78.85 & 79.21 & 71.20 & 71.43 & 71.88 & 72.00 \\
		& \textbf{HopDC (Ours)}  & \textbf{80.18} & \textbf{80.19} & \textbf{80.19} & \textbf{80.16} & \textbf{72.20} & \textbf{72.22} & \textbf{72.27} & \textbf{72.30} \\
		\midrule
		\multirow{2}{*}{LDA}      
		& $\alpha_1$-SLDC  & 73.13 & 74.85 & 76.06 & 77.06 & 69.33 & 69.45 & 69.62 & 69.57 \\
		& \textbf{HopDC (Ours)}  & \textbf{78.30} & \textbf{78.31} & \textbf{78.29} & \textbf{78.29} & \textbf{69.18} & \textbf{69.27} & \textbf{69.07} & \textbf{69.17} \\
		\bottomrule
	\end{tabular}
\label{tab:app_data_efficiency}
\end{table*}

\end{document}